\documentclass[format=acmsmall, review=false]{acmart}
\usepackage{acm-ec-26-proc}
\usepackage{booktabs} % For formal tables

\usepackage{amsmath,amsfonts,bm}

\def\eqref#1{equation~\ref{#1}}
\def\1{\bm{1}}

\def\rx{{\textnormal{x}}}
\def\ry{{\textnormal{y}}}

\def\vzero{{\bm{0}}}

\def\vb{{\bm{b}}}

\def\ve{{\bm{e}}}

\def\vq{{\bm{q}}}

\def\vs{{\bm{s}}}

\def\vv{{\bm{v}}}

\def\vx{{\bm{x}}}
\def\vy{{\bm{y}}}

\def\mA{{\bm{A}}}
\def\mB{{\bm{B}}}

\def\mD{{\bm{D}}}
\def\mE{{\bm{E}}}

\def\mG{{\bm{G}}}

\def\mJ{{\bm{J}}}

\def\mP{{\bm{P}}}
\def\mQ{{\bm{Q}}}
\def\mR{{\bm{R}}}

\def\mT{{\bm{T}}}

\def\mV{{\bm{V}}}

\DeclareMathAlphabet{\mathsfit}{\encodingdefault}{\sfdefault}{m}{sl}
\SetMathAlphabet{\mathsfit}{bold}{\encodingdefault}{\sfdefault}{bx}{n}

\newcommand{\E}{\mathbb{E}}

\newcommand{\R}{\mathbb{R}}

\DeclareMathOperator{\Tr}{Tr}

\usepackage{amsthm,amsmath}
\usepackage{natbib} % For author--year citations

\usepackage[capitalise]{cleveref}

\usepackage{url}
\usepackage{wrapfig}
\usepackage{graphicx}
\usepackage{subcaption}
\usepackage{float}
\newtheorem{theorem}{Theorem}[section]
\newtheorem{lemma}[theorem]{Lemma}
\newtheorem{proposition}[theorem]{Proposition}

\newtheorem{definition}[theorem]{Definition}

\newtheorem{remark}[theorem]{Remark}

\newcommand{\defn}[1]{{\textbf{\textit{#1}}}}

\DeclareMathOperator{\exact}{{\rm EXACT}}
\DeclareMathOperator{\hamming}{{\rm Hamming}}

\DeclareMathOperator{\dist}{{dist}}

\DeclareMathOperator{\blackwell}{{\rm Blackwell}}

\newcommand{\squishlist}{
   \begin{list}{$\bullet$}
    { \setlength{\itemsep}{0pt}      \setlength{\parsep}{3pt}
      \setlength{\topsep}{3pt}       \setlength{\partopsep}{0pt}
      \setlength{\leftmargin}{1.5em} \setlength{\labelwidth}{1em}
      \setlength{\labelsep}{0.5em} } }
\newcommand{\squishend}{  \end{list}  }

\usepackage[ruled]{algorithm2e} % For algorithms

\SetAlFnt{\small}
\SetAlCapFnt{\small}
\SetAlCapNameFnt{\small}
\SetAlCapHSkip{0pt}
\IncMargin{-\parindent}

\setcitestyle{authoryear}

\title[Data Reliability Scoring]{Data Reliability Scoring}

\author{Yiling Chen}
\affiliation{
  \institution{Harvard University}
  \city{Cambridge}
  \country{United States}
}

\author{Shi Feng}
\affiliation{
  \institution{Harvard University}
  \city{Cambridge}
  \country{United States}
}

\author{Paul Kattuman}
\affiliation{
  \institution{University of Cambridge}
  \city{Cambridge}
  \country{United Kingdom}
}

\author{Fang-Yi Yu}
\affiliation{
  \institution{George Mason University}
  \city{Fairfax}
  \country{United States}
}

\renewcommand{\shortauthors}{Chen, Feng, Kattuman, and Yu}

\begin{abstract}

How can we assess the reliability of a dataset without access to ground truth? We introduce the problem of {\em reliability scoring} for datasets collected from potentially strategic sources. The true data are unobserved, but we see outcomes of an unknown statistical experiment that depends on them. To benchmark reliability, we define ground-truth–based orderings that capture how much reported data deviate from the truth. We then propose the {\em Gram determinant score}, which measures the volume spanned by vectors describing the empirical distribution of the observed data and experiment outcomes. We show that this score preserves several ground-truth-based reliability orderings and, uniquely up to scaling, yields the same reliability ranking of datasets regardless of the experiment -- a property we term experiment agnosticism. Experiments on synthetic noise models, CIFAR-10 embeddings, and real employment data demonstrate that the Gram determinant score effectively captures data quality across diverse observation processes.

\end{abstract}
\begin{document}

\settopmatter{printfolios=false}
\acmYear{2026}\copyrightyear{2026}
\setcopyright{cc}
\setcctype[4.0]{by}
\acmConference[EC '26]{The 27th ACM Conference on Economics and Computation}{July 6--10, 2026}{Rome, Italy}
\acmBooktitle{The 27th ACM Conference on Economics and Computation (EC '26), July 6--10, 2026, Rome, Italy}
\acmDOI{10.1145/3821539.3827689}
\acmISBN{979-8-4007-2813-6/26/07}

\maketitle

\section{Introduction}
Reliable data can effectively inform decision-making. For example, vehicle condition and driving behavior data help insurance companies set policies; investors' positions guide regulators in adjusting financial market rules; and
during the COVID-19 pandemic, case numbers were used by governments to allocate medical resources.
Yet, such data are typically reported by people. They can be noisy, and more importantly, strategically or maliciously distorted. Direct verification is often impossible or impractical. This raises a central question: how can we tell whether a dataset is reliable? Answering this would greatly enhance the value of data-driven methods for decision-making.

Without further knowledge, this question is unresolvable. But in practice, we often have access to data that are related to the private data in question. For instance, insurance companies may use telematic devices--albeit imperfect--to estimate vehicle condition; regulators can observe trading volumes correlated with investors' positions; and governments track COVID mortality numbers linked to true case counts through disease fatality rates.
Such auxiliary observations can provide useful information to assess how well the reported data are consistent with the unobservable ground truth.

In this paper, we initiate the study of reliability scoring for datasets collected from potentially strategic or noisy sources.  Although the underlying truth remains unknown, we assume access to outcomes of unknown statistical experiments that depend on
it. Our contributions include:
\squishlist
\item We formalize the problem of reliability scoring from observations generated by unknown experiments. (\cref{sec:model})
\item We introduce ground-truth-based dataset reliability orderings as benchmarks for evaluating reliability scores. (\cref{sec:ordering})
\item We propose a novel reliability measure, the \defn{Gram Determinant Score}, along with its kernel variant, which preserves several ground-truth-based dataset reliability orderings under certain conditions. Moreover, we show that the Gram Determinant Score is, up to scaling, the unique reliability score that produces the same dataset ranking for all experiments – a property we term {\em experiment agnosticism}. (\cref{sec:gram})
\item We analyze the limitations of reliability scoring and show that the conditions under which the Gram Determinant Score preserves reliability orderings are nearly tight. (\cref{sec:impossible})
\item We empirically validate the Gram Determinant Score using synthetic data, the CIFAR-10 image dataset, and employment data.\footnote{Code is available at \url{https://github.com/chen-lab-seas/Data-Reliability-Scoring}.} (\cref{sec:experiment})
%observations generated from random experiments, and embeddings from contrastive learning~\citep{chen2020simple} on image data. (\cref{sec:experiment})
\squishend

The Gram Determinant Score admits a geometric interpretation: it measures the volume of the parallelepiped spanned by the joint distribution of the reported data and the experiment outcomes. Under the conditions analyzed below, this volume decreases as the reported data deviate further from the truth. (\Cref{fig:illustrate_vol}) %

\subsection{Related Work}

Early frameworks categorize data reliability into intrinsic, contextual, accessibility, and representational dimensions~\citep{wang1996beyond,priestley2023survey}.
Our work focuses on intrinsic reliability---the extent to which reported data match the true data---using auxiliary observations.

Our approach is inspired by information elicitation, which designs scoring mechanisms that incentivize truthful reporting. A key distinction is our emphasis on preserving ordinal relationships: assigning higher scores to more reliable data. Traditional elicitation instead focuses solely on ensuring that truthful reporting is strictly optimal among alternatives.
Information elicitation has two main settings (1) when the scoring mechanism can access the ground truth, e.g., proper scoring rules for predictions of future observable events~\citep{gneiting2007strictly,osband1985providing,lambert2008eliciting,frongillo2015vector,liuchen};
and (2) peer prediction mechanisms, which do not have access to ground truth but rely on multiple agents' reports~\citep{MRZ05,dasgupta2013crowdsourced}.  The most relevant work to ours is \citet{10.1145/3638239}, which introduces determinant mutual information and inspires our Gram Determinant Score.  We provide a more detailed comparison with \citet{10.1145/3638239} in \Cref{sec:kernels}.  More recent work~\citep{zheng2025properdatasetvaluationpointwise} uses Shannon (pointwise) mutual information to evaluate datasets and introduce the Blackwell ordering to compare reported dataset.

Traditional statistical approaches~\citep{huber2004robust,meeker2021statistical} often assess reliability under distributional assumptions.
In contrast, our method evaluates reliability agnostic to the underlying distribution.
There are several general-purpose scores that measure the stochastic relationship between random variables, e.g., KL-divergence~\citep{kullback1951information}, $f$-divergence~\citep{csiszar1972class}, determinant~\citep{10.5555/2999325.2999469,xu2019l_dmi}, PCA~\citep{amiri2022fundamentalstaskagnosticdatavaluation}.
But they often lack clear connections to standard, interpretable criteria such as accuracy or data integrity.  On the other hand, one line of data valuation focuses on task-dependent utility---quantifying the value of a dataset or individual samples for a specific objective.  Examples include value of information in decision theory~\citep{howard2007information,7782936, FrankelKamenica2019AER}, influence-based valuation~\citep{ade490ae-4e2e-3d0a-b973-8fce242c5658,pmlr-v70-koh17a}, and data Shapley~\citep{ghorbani2019data}.  In contrast, our reliability scoring aims to evaluate datasets in a task-agnostic and experiment-agnostic manner.
% , and

% Data Shapley — Ghorbani & Zou, 2019. Canonical Shapley‑value framing of each point’s contribution to utility; requires a downstream task/training loop (unlike your experiment‑only setting).

% Efficient Shapley approximations — Jia et al., 2019–2021 (TMC‑Shapley; KNN‑Shapley). Pragmatic estimators that trade accuracy for speed; again task‑dependent, so useful to contrast with your experiment‑agnostic ordering guarantees.

% Influence‑based valuation — Koh & Liang, 2017 (influence functions); Pruthi etal., 2020 (TracIn). Attribute model predictions/loss to training points; emphasize that these need models/labels, whereas your score needs only auxiliary outcomes.

Other related areas include learning with noisy labels~\citep{natarajan2013learning}, which typically assumes that reports are corrupted by independent noise. Some works (e.g.,~\citep{liu2020peer}) relax this by allowing unknown noise, but our setting is more general: auxiliary observations may lie in an entirely different space.
Anomaly detection~\citep{chandola2009anomaly} addresses distribution shifts, but focuses on adaptive detection rather than reliability scoring. Finally, reliability theory primarily studies system robustness to failure ~\citep{gnedenko2014mathematical}, a concept distinct from data reliability.

\section{Model}\label{sec:model}
In this section, we introduce the problem of designing data reliability scores to assess how much a dataset deviates from its inaccessible ground truth. To benchmark reliability, we propose ground-truth-based reliability orderings---partial orders over datasets that compare their relative deviations from the same true dataset. The ideal goal of a reliability score is to preserve these orderings, assigning higher scores to datasets that more faithfully reflect the true data.

\subsection{Basic Setup}
There is a single data source (an agent) who has access to a set of \emph{true data} $\vx = (x_1,\dots,x_N)$ of size $N$.\footnote{$\vx$ is non-time-series data. Hence, the order of the data within the set is not important.}
%Each $x_i$ is a label for a corresponding feature vector (e.g., an image or a patient).
The agent provides \emph{reported data}
$\hat{\vx}=(\hat{x}_1,\dots,\hat{x}_N)$, which can potentially be different from $\vx$. Let $\mathcal{X} = [d]$ be the set of $d$ possible data values. Thus, $x_n\in \mathcal{X}$ and $\hat{x}_n \in \mathcal{X}$ for all $n$.%

Our goal is to evaluate how reliably the reported data $\hat{\vx}$ reflects the true data $\vx$. Although $\vx$ is unobserved, we have access to additional observable data $\vy=(y_1,\dots,y_N)$, called \emph{observations}, which are indirectly related to $\vx$. The observation space $\mathcal{Y}$ may differ from $\mathcal{X}$. We model the relationship between $\vy$ and $\vx$ as an unknown, statistical \emph{experiment}, represented by a column-stochastic matrix $\mP = (P_x)_{x\in \mathcal{X}}$, where each column $P_x$ is a distribution over $\mathcal{Y}$.  Given true data $\vx = (x_1,\dots,x_N)$, observations are generated according to $\mP$ with $y_n\sim P_{x_n}$ independently for all $n\in [N]$. We denote this generation as $\vy \sim \mP(\vx)$.
%Departing from conventional statistical inference, we make minimal assumptions about the conditional distributions $\mP$, and only assume that each distribution $P_x$ is linear independent where $\mP$ has full column rank.  We denote $\mathcal{P}_{\text{indep}}$ as the set of linearly independent experiments on $\mathcal{Y}$.

%
%For example, $\vx$ can represent true Covid-19 county-level case numbers, with $\hat{\vx}$ being the reported case numbers. The observations $\vy$ in this case can be Covid-19 mortality data, which are usually reliably recorded. Or,
For instance, $\vx$ may represent patients' true disease state (having or not having the disease), $\hat{\vx}$ the diagnoses reported by a hospital to an insurance database for reimbursement, and $\vy$ the results of inexpensive blood tests or imaging biomarkers correlated with the disease. As another example, in an image-labeling dataset, $\vx$ denotes the true image labels, while $\hat{\vx}$ are the reported labels. The observations $\vy$ may come from encoder representations, such as those produced from contrastive learning methods~\citep{zbontar2021barlowtwinsselfsupervisedlearning}.

Having access to $\vy$ and knowing that $\vy$ are generated by unknown experiment $\mP$, we want to design a \emph{reliability score} $S:\mathcal{X}^N\times \mathcal{Y}^N\to \mathbb{R}$ such that, if a dataset $\hat{\vx}$ aligns with $\vx$ more than a dataset $\hat{\vx}'$ does, dataset $\hat{\vx}$ receives a higher reliability score in expectation than dataset $\hat{\vx}'$: $\E_{\vy\sim \mP(\vx)}[S(\hat{\vx}, \vy)]> \E_{\vy\sim \mP(\vx)}[S(\hat{\vx}', \vy)]$. However, to formalize this goal, we will first need metrics to quantify how much reported data align with the true data. In \cref{sec:datarelation}, we describe how to use a misreport matrix to represent the relationship between reported data and true data. Then, we introduce four notions of ground-truth-based reliability ordering of reported datasets in \cref{sec:ordering} before returning to define the ideal goal of reliability scoring in \cref{sec:scoring}.

\subsection{Representation of Dataset Relationships} \label{sec:datarelation}

The relationship between the true dataset $\vx$ and a reported dataset $\hat{\vx}$ can be summarized by the size of the datasets $N$ and a $d \times d$-dimension \emph{misreport matrix} $\mQ$ where each entry $\mQ (i, j)$ represents the frequency of misreporting true value $i$ in $\vx$ for value $j$ in $\hat{\vx}$:
$$\mQ(i,j) = \frac{1}{N}\sum_{n = 1}^N \mathbf{1}[{x}_n = i, \hat{x}_n = j].$$
$\mQ$ is the joint frequency of true data and reported data. It can be further decomposed into marginal frequency and conditional frequency.
Let $\vq_\vx(i) = \frac{1}{N}\sum_{n=1}^N \mathbf{1}[x_n = i]$ and $\vq_{\hat{\vx}}(i)  = \frac{1}{N}\sum_{n=1}^N \mathbf{1}[\hat{x}_n = i]$ $\forall i\in \mathcal{X}$, the marginal frequency matrices are defined as $d\times d$ diagonal matrices $\mQ_\vx,\mQ_{\hat{\vx}}$ with $\vq_\vx$ and $\vq_{\hat{\vx}}$ respectively as diagonal and zeros everywhere else.
When the relevant label has positive frequency, define the conditional-frequency matrices $\mQ_{\hat{\vx}\mid \vx}$ and $\mQ_{\vx\mid \hat{\vx}}$ by, for all $i, j\in \mathcal{X}$,
$\mQ_{\hat{\vx}\mid \vx}(i,j) = \frac{\sum_n \mathbf{1}[x_n = j, \hat{x}_n = i]}{\sum_n \mathbf{1}[x_n = j]}$ and $\mQ_{\vx\mid \hat{\vx}}(i,j) = \frac{\sum_n \mathbf{1}[x_n = i,\hat{x}_n = j]}{\sum_n \mathbf{1}[\hat{x}_n = j]}$.
% For brevity, we will also write them as $\mQ_\to = \mQ_{\hat{\vx}\mid \vx}$ and $\mQ_\gets = \mQ_{\vx\mid \hat{\vx}}$.
Hence,
\begin{equation}\label{eq:m2j}
    \mQ = (\mQ_{\hat{\vx}\mid \vx}\mQ_{\vx})^\intercal\text{ and }\mQ = \mQ_{\vx\mid \hat{\vx}}\mQ_{\hat{\vx}}.
\end{equation}
The joint and marginal frequency matrices exist for every pair of $\vx$ and $\hat{\vx}$; the conditional-frequency matrices are defined when every conditioning label occurs at least once. All of these matrices are unobserved because $\vx$ is unknown. We introduce them to quantify the deviation of $\hat{\vx}$ from $\vx$.
In this paper, we use $\mathcal{Q}$ to denote a set of misreporting matrices and, by abuse of notation, the pairs $(\vx,\hat{\vx})$ whose associated misreport matrix belongs to $\mathcal{Q}$.

Given a statistical experiment $\mP$, the matrix product $\mP\mQ$ is a $|\mathcal{Y}|\times |\mathcal{X}|$ matrix representing the joint distribution\footnote{In our setting, $\vx$ and $\hat{\vx}$ are fixed, while $y$ is random. For convenience, we use the term joint distribution, rather than expected frequency, for $\mP\mQ$ and $\mP\mQ_{\vx}$.} of observations and reported data; its $(k,i)$ entry is $\Pr(y=k, \hat{x}=i)$. The matrix product $\mP\mQ_{\vx}$ is a $|\mathcal{Y}|\times |\mathcal{X}|$ matrix representing the joint distribution of observations and true data; its $(k,i)$ entry is $\Pr(y=k, x=i)$. Although both $\mP\mQ$ and $\mP\mQ_{\vx}$ are unknown, the paired data $(\hat{x}_n,y_n)_{n=1}^N$ provide the observations available for reliability scoring.

\subsection{Reliability Orderings of Datasets} \label{sec:ordering}

To compare the reliability of reported datasets relative to the true data $\vx$, some preference on relative dataset reliability is needed. While the preference may depend on applications, we suggest three natural strict partial orderings of reported datasets, each defined with respect to true data $\vx$.
\begin{enumerate}
% \squishlist
    \item[1.] {\bf Exact Match Ordering}: $\hat{\vx}\succ_{\exact}^{\vx}\hat{\vx}'$ if $\hat{\vx} = \vx$ but $\hat{\vx}'\neq \vx$. When every true label occurs at least once, this is equivalent to $\mQ_{\hat{\vx}\mid \vx}'\neq \mathbb{I}$ and $\mQ_{\hat{\vx}\mid \vx} = \mathbb{I}$. This ordering picks up only complete agreement with the true data, and does not differentiate any pair of reported datasets if neither agrees with the true data. This order captures the notion of data integrity~\citep{10.1145/191177.191183}.
    \item[2.] {\bf Blackwell dominant ordering}: $\hat{\vx}\succ_{\blackwell}^{\vx}\hat{\vx}'$ if $\mQ$ and $\mQ'$ are both invertible and (row) diagonally maximized (i.e. $\mQ(i, j) \le \mQ(i,i)$ and $\mQ'(i, j) \le \mQ'(i,i)$ for all $i$ and $j$) and there exists a (column) stochastic matrix $\mT\neq \mathbb{I}$ so that $\mT \mQ_{\hat{\vx}\mid \vx} =  \mQ_{\hat{\vx}\mid \vx}'$ (equivalently, $\mQ' = \mQ\mT^\intercal$ by \cref{eq:m2j}).
    This ordering captures the idea that post-processing that transforms $\hat{\vx}$ into $\hat{\vx}'$ can only reduce the reliability or informativeness of the data~\citep{blackwell1953equivalent}. In particular, this ordering ensures that the true data ranks highest and that uninformative random reports rank lowest.
    %\item[3.] Hamming ordering: $\hat{\vx}\succ_{\hamming}^{\vx}\hat{\vx}'$ if $\sum_{n = 1}^N \mathbf{1}[\hat{x}_n \neq x_n]< \sum_{n = 1}^N \mathbf{1}[\hat{x}'_n \neq x_n]$ or, equivalently, $\Tr(\mQ)> \Tr(\mQ')$.  This ordering counts the number of disagreements between the true data and the reported data.~\citep{hamming1950error}
    \item[3.] {\bf dist ordering}: Given a distance function $\dist: \mathcal{X}\times \mathcal{X}\to \mathbb{R}$ so that $\dist(x,x') = \dist(x',x)$, $\dist(x,x) = 0$ and $\dist(x,x')>0$ if $x\neq x'$,\footnote{Any metric, e.g., $\ell_2$-norm, satisfies the above three conditions. Additionally, a function with these properties is often referred to as a semimetric.} we say $\hat{\vx}\succ_{\dist}^{\vx}\hat{\vx}'$ if $\sum_{n = 1}^N \dist(\hat{x}_n, x_n)<\sum_{n = 1}^N \dist(\hat{x}_n', x_n)$. This ordering captures the coordinate-wise difference between true and reported data.  We may also consider a weaker notion, \emph{$\alpha$-$\dist$ ordering} with some $\alpha\in (0,1]$.  We say $\hat{\vx}\succ_{\dist, \alpha}^{\vx}\hat{\vx}'$ if $\sum_{n = 1}^N \dist(\hat{x}_n, x_n)<\alpha\sum_{n = 1}^N \dist(\hat{x}_n', x_n)$.
 In other words, the distance between $\hat{\vx}$ and $\vx$ is at least a factor of $\alpha$ smaller than that of $\hat{\vx}'$ and $\vx$, in order to rank $\hat{\vx}$ and $\hat{\vx}'$.

   A special case of $\dist$ ordering is \defn{Hamming ordering}, when $\dist$ is the discrete metric $\dist(i,j) = \mathbf{1}[i\neq j]$ for all $i,j\in \mathcal{X}$. We say $\hat{\vx}\succ_{\hamming}^{\vx}\hat{\vx}'$ if $\sum_{n = 1}^N \mathbf{1}[\hat{x}_n \neq x_n]< \sum_{n = 1}^N \mathbf{1}[\hat{x}'_n \neq x_n]$ or, equivalently, $\Tr(\mQ)> \Tr(\mQ')$.  Hamming ordering~\citep{hamming1950error} counts the number of disagreements between the true data and the reported data.

% \squishend
\end{enumerate}

%Hamming ordering is a special case of $\dist$-ordering when $\dist$ is the discrete metric $\dist(i,j) = \mathbf{1}[i\neq j]$ for all $i,j\in \mathcal{X}$.
%While exact match ordering, hamming ordering and $\dist$ ordering are defined for all misreport matrices $\mQ$ and $\mQ'$
Blackwell dominant ordering is intentionally defined for a subset of misreport matrices: $\mQ, \mQ' \in \mathcal{Q}_{\text{reg}}$, which is the collection of invertible and (row) diagonally maximal matrices such that $\mQ(i,j)\le \mQ(i,i)$ for all $i$ and $j$. Intuitively, diagonal maximality requires the true data values to dominate any misreport in a reported dataset. The restriction to $\mathcal{Q}_{\text{reg}}$ is necessary for Blackwell dominant ordering to be a strict partial order. In \cref{app:model}, we formally prove that all of the above orderings are strict partial orders. In particular, the Blackwell ordering fails to be strict if either invertibility or diagonal maximality of $\mQ$ and $\mQ'$ is not enforced.\footnote{Instead of $\mathcal{Q}_{\text{reg}}$, we can alternatively require (a) $\mQ$ and $\mQ'$ are invertible and (b) $\mT$ is not a permutation matrix (i.e. $\mQ\mT^\intercal$ is not a permutation of columns of $\mQ$) to ensure that Blackwell dominant ordering is strict. However, this set of conditions does not support the result in \cref{prop:comparison} that Hamming ordering refines the Blackwell dominant ordering.%
}

These orderings reflect different ways of measuring the extent of misreporting, with some providing finer distinctions between datasets than others.
%For example, the exact match ordering focuses only on whether the reported data is perfectly aligned with the true data or not, while the Hamming distance ordering provides a finer distinction by considering the number of position where the reports disagree with the truth.
Formally, given a set of misreport matrices $\mathcal{Q}$, partial ordering $\succ_1^\cdot$ \emph{refines} partial ordering $\succ_2^\cdot$ on $\mathcal{Q}$ if $\forall \vx, \hat{\vx}, \hat{\vx}'$ with associated misreport matrices $\mQ, \mQ'\in \mathcal{Q}$,
$\hat{\vx}\succ_2^{\vx}\hat{\vx}'\Rightarrow \hat{\vx}\succ_1^{\vx}\hat{\vx}'.$
The following proposition shows that Blackwell dominant ordering refines exact-match ordering, and Hamming ordering refines Blackwell dominant ordering.  The proofs are in~\cref{app:model}.  %
\begin{proposition}[Refinement]\label{prop:comparison}
The reliability orderings have the following relationships:
\begin{enumerate}
% \squishlist
    \item[1.] Blackwell dominant ordering refines the exact match ordering on $\mathcal{Q}_{\text{reg}}$.
    \item[2.] Hamming ordering refines the Blackwell dominant ordering on $\mathcal{Q}_{\text{reg}}$.
    \item[3.] For all $\alpha\ge \alpha'$ and distance function $\dist$, $\alpha$-$\dist$ ordering refines $\alpha'$-$\dist$ ordering.
%    \item If $\dist$ is discrete metric so that $\dist(i,j) = \mathbf{1}[i\neq j]$ for all $i,j\in \mathcal{X}$, $\dist$ ordering reduces to Hamming ordering, for all misreport matrices.  Moreover, $\frac{1}{N}\sum_{i = n}^N \mathbf{1}[\hat{x}_n\neq x_n] = 1-\Tr(\mQ)$
% \squishend
\end{enumerate}
\end{proposition}

\subsection{Reliability Scoring} \label{sec:scoring}
We now return to formally define the ideal goals of reliability scoring.
%should reflect a partial ordering based on how closely the reported data aligns with the true data.
%
\begin{definition}\label{def:proper}
    Given a reliability ordering $\succ^\cdot$ over $\mathcal{X}^N$, a reliability score $S:\mathcal{X}^N\times \mathcal{Y}^N\to \mathbb{R}$ \emph{preserves} partial ordering $\succ^\cdot$ under experiment $\mP$, if for all $\vx, \hat{\vx}, \hat{\vx}' \in \mathcal{X}^N$ with $\hat{\vx}\succ^\vx \hat{\vx}'$ we have
\begin{equation}\label{eq:proper}
    \E_{\vy\sim P(\vx)}[S(\hat{\vx}, \vy)]> \E_{\vy\sim P(\vx)}[S(\hat{\vx}', \vy)].
\end{equation}
\end{definition}
Given a set of experiments $\mathcal{P}$, a set of misreport matrices $\mathcal{Q}$, and a minimum size of reported datasets $N_0\in \mathbb{N}$, we say that a reliability score preserves $\succ^\cdot$ under $\mathcal{P}, \mathcal{Q}$ and $N_0$ if \cref{eq:proper} holds for all $\mP\in \mathcal{P}$ and tuples $\vx, \hat{\vx}, \hat{\vx}'$ of size at least $N_0$ with $\hat{\vx}\succ^\vx\hat{\vx}'$ and $\mQ, \mQ'\in \mathcal{Q}$.  We further say $S$ \emph{asymptotically} preserves $\succ^\cdot$ under $\mathcal{P}, \mathcal{Q}$, if for all $\mP\in \mathcal{P}$ and $\mQ, \mQ'\in \mathcal{Q}$ there exists $N_0$ so that $S$ preserve $\succ^\cdot$ under $\mP$ for all  $\vx, \hat{\vx}, \hat{\vx}'$ of size at least $N_0$ with $\hat{\vx}\succ^\vx\hat{\vx}'$ and misreport matrices $\mQ, \mQ'$.

In the remainder of the paper, we study the problem of designing reliability scores $S(\hat{\vx}, \vy)$ that preserve partial orderings of interest. We refer to this as the detail-free setting, since scoring does not rely on knowledge of $\mQ$ or $\mP$. For the analysis, however, we also consider a partial-knowledge setting, where the score can take the joint distribution \(\mP\mQ\) as input, \(S(\mP\mQ)\). This setting serves as a technical tool: it allows us to establish impossibility results (\cref{sec:impossible}) and to illustrate the core ideas underlying our approach to detail-free scoring (\cref{sec:gram}).

%However, as a technical tool, our analysis uses a partial knowledge setting, where reliability scoring can take the joint distribution $\mP\mQ$ as input, $S(\mP\mQ)$, to establish impossibility results (\cref{sec:impossible}) and to illustrate our core ideas for detail-free scoring (\cref{sec:gram}).

\section{Impossibility Results for Reliability Scoring}\label{sec:impossible}
We explore innate limitations of reliability scoring. These impossibility results form a foundation for charting the feasible combinations of $\mathcal{P}$ and $\mathcal{Q}$ for reliability scoring and motivate~\cref{sec:gram}.

This section focuses on the partial knowledge setting, where the joint distribution of observations and reported data, $\mP\mQ$, is assumed to be known, and provided as input to the score. Impossibility results in this setting extend to the detail-free setting for reliability scores that rely on estimates of $\mP\mQ$. In particular, the impossibility results apply to the Gram determinant score that we will introduce in \cref{sec:gram}.
%The impossibility results in the partial knowledge setting hold for a broad class of reliability scores in the detail-free setting that applies some risk/loss function over pairs $(\hat{x}_n,y_n)$, e.g., log-likelihood functions, and also apply to our reliability scores, Gram determinant score, defined in \cref{sec:gram}.
We provide a more detailed discussion in~\cref{app:impossible}. %

We first introduce the class of independent experiments and a few classes of misreport matrices that'll be used in this paper. %
\begin{itemize}
\item $\mathcal{P}_{\text{indep}}$: the set of linearly independent experiments, where $\mP\in \mathcal{P}_{\text{indep}}$ if and only if all columns of $\mP$ are linearly independent.
\item $\mathcal{Q}_{\text{nonperm}}$: the set of misreport matrices $\mQ$ so that the associated $\mQ_{\hat{\vx}|\vx}$ is neither a permutation matrix nor an identity matrix.
\item $\mathcal{Q}_{\text{reg}}$: the set of invertible and diagonally maximal misreport matrices where $\mQ(i,j)\le \mQ(i,i)$ for all $i$ and $j$. This was also defined earlier in \cref{sec:ordering}. %
\item $\mathcal{Q}_{\text{dom}}$: the set of (row) diagonally dominant misreport matrices where $\sum_{j: j\neq i}|\mQ(i,j)|\le |\mQ(i,i)|$ for all $i$.\footnote{Note that diagonally dominant matrices are invertible by Gershgorin circle theorem.} %%
\item $\mathcal{Q}_{L,\delta}$: the set of (row) diagonally dominant misreport matrices where the true data are $L$ balanced and the Hamming distance is bounded above by $N\delta$.
True data $\vx$ is $L$-balanced if $\vq_\vx(x)\le L\vq_\vx(x')$ for all $x, x'\in \mathcal{X}$. We use $\mathcal{Q}_L :=\mathcal{Q}_{L,1}$ to denote the set of (row) diagonally dominant misreport matrices where the true data are $L$ balanced, with no restriction on Hamming distance.
%\item $\mathcal{Q}_L$: is defined as $\mathcal{Q}_{L,1}$, the set of diagonally dominant misreport matrices where the true data are $L$ balanced without restriction on the Hamming distance.
\end{itemize}

We note that $\mathcal{Q}_{L,\delta}\subseteq \mathcal{Q}_{L}\subset \mathcal{Q}_{\text{dom}}\subset \mathcal{Q}_{\text{reg}}\subset \mathcal{Q}_{\text{nonperm}}$ for all $L$ and $\delta$.

%Let $\mathcal{P}_{\text{indep}}$ be the set of linearly independent experiments, where $\mP\in \mathcal{P}_{\text{indep}}$ if and only if all columns of $\mP$ are linearly independent.  Let $\mathcal{Q}_{\text{nonperm}}$ consist of all misreport matrices so that $\mQ_{\hat{\vx}|\vx}$ is neither a permutation matrix nor an identity matrix, and $\mathcal{Q}_{\text{dom}}$ consists of diagonally dominant misreport matrices.  True data $\vx$ is $L$-balanced if $\vq_\vx(x)\le L\vq_\vx(x')$ for all $x, x'\in \mathcal{X}$.  We define $\mathcal{Q}_{L,\delta}$ as the class of diagonally dominant misreport matrices where the true data are $L$ balanced and the Hamming distance is bounded above by $\delta$.  Additionally, we set $\mathcal{Q}_L=\mathcal{Q}_{L,1}$ which ensures the true data are $L$ balanced without restriction on the Hamming distance.  Note that $\mathcal{Q}_{L,\delta}\subseteq \mathcal{Q}_{L}\subset \mathcal{Q}_{\text{dom}}\subset \mathcal{Q}_{\text{reg}}\subset \mathcal{Q}_{\text{nonperm}}$.

%

\begin{proposition}\label{prop:impossible}
In the partial-knowledge setting, it is sometimes impossible for any reliability score to preserve reliability orderings. In particular,
%No reliability score exists in the partial-knowledge setting in each of the following:
\begin{enumerate}
\item[1.] \textbf{Exact match ordering:} There exists a $\mathcal{P}$ so that no score preserves the exact match ordering under $\mathcal{P}$ and $\mathcal{Q}_{\text{nonperm}}$.  Additionally, for all $\mathcal{Q}\supsetneq \mathcal{Q}_{\text{nonperm}}$, no score preserves the exact match ordering on $\mathcal{P}_{\text{indep}}$ and $\mathcal{Q}$.
\item[2.] \textbf{Blackwell dominant ordering:} For any $\mathcal{P}$, if there exists $\mP\in \mathcal{P}$ and a rational vector $\vv\neq \vzero$ so that $\mP \vv = \vzero$, no score preserves the Blackwell ordering on $\mathcal{P}$ and $\mathcal{Q}_{\text{reg}}$.
   % \item[2.] \textbf{Blackwell dominant ordering:} For any $\mathcal{P}$, if there exists $\mP\in \mathcal{P}$ and $\vv\neq \vzero \in \mathbb{Q}^{d}$ so that $\mP \vv = \vzero$, no score preserves the Blackwell ordering on $\mathcal{P}$ and $\mathcal{Q}_{\text{reg}}$. $\mathbb{Q}^{d}$ is the set of $d$ dimensional vectors of rational numbers.
    \item[3.] \textbf{Hamming and dist orderings}: No score preserves the Hamming ordering under $\mathcal{P}_{\text{indep}}$ and $\mathcal{Q}_{\text{dom}}$.  Additionally, no score preserves the $\dist$ ordering under $\mathcal{P}_{\text{indep}}$ and $\mathcal{Q}_{\text{dom}}$ for any $\dist$.%%
% \squishend
\end{enumerate}

\end{proposition}
The first part of \cref{prop:impossible} establishes that no score can respect the exact-match reliability ordering across all experiment sets.   The non-permutation condition is needed here to exclude degenerate cases such as label permutations. By \cref{prop:comparison}, these impossibility results also extend to the other orderings.  The second part further shows that even a single linearly dependent experiment is enough to make preservation of the Blackwell ordering impossible.  We therefore focus on the class of linearly independent experiments, $\mathcal{P}_{\text{indep}}$.  Finally, the third part shows that no reliability score can preserve the Hamming or any other $\dist$ ordering, even under diagonally dominant misreport matrices $\mathcal{Q}_{\text{dom}}$. In ~\cref{sec:gram}, we thus further restrict our attention to $\mathcal{Q}_{L,\delta}$.

\section{Gram Determinant Reliability Score}\label{sec:gram}

% Our idea for measuring data reliability is to leverage the diversity of observations. If two reports share the same true labels, their corresponding observations come from the same distribution. Using this idea, we introduce the Gram determinant score.
Our idea for measuring data reliability is to leverage the diversity of observations. We formalize this idea with the Gram determinant score---the determinant of a Gram matrix of the observation distributions conditional on reported labels.

\begin{definition}\label{def:gramdet}
    Given finite sets $\mathcal{X} = [d]$ and $\mathcal{Y}$, and an experiment $\mP$, we define Gram matrix of labels as $\mG=\mP^\intercal\mP \in \R^{|\mathcal{X}|\times |\mathcal{X}|}$ where $\mG(x,x') = \langle P_x, P_{x'}\rangle = \Pr_{y\sim P_x, y'\sim P_{x'}}[y = y']$.  Moreover, given $\vx$ and $\hat{\vx}$, we define the Gram matrix of reports $\hat{\vx}$ as $\hat{\mG} =(\mP\mQ)^\intercal(\mP\mQ) \in \R^{|\mathcal{X}|\times |\mathcal{X}|}$ where
$\hat{\mG}(x,x') :  = \frac{1}{N^2}\sum_{n, n': \hat{x}_n = x, \hat{x}_{n'} = x'}\langle P_{x_n}, P_{x_{n'}}\rangle.$
The \defn{Gram determinant score} is
    \begin{equation}\label{eq:gramdet}
    \Gamma:= \det\left(\hat{\mG}\right) = \sum_{\sigma\in symm(d)} sgn(\sigma)\prod_{i = 1}^d \hat{\mG}(i, \sigma(i)).
\end{equation}
where $symm(d)$ is the set of all permutations on $[d]$ and $sgn$ the sign function of permutations.  We further denote $\Gamma(\mP\mQ):= \Gamma$ to highlight that the Gram determinant score takes $\mP\mQ$ as input.
\end{definition}

% \begin{wrapfigure}{l}{0.4\linewidth}
%\vspace{-24pt}
\begin{figure}
    \centering
    \includegraphics[width=0.5\linewidth]{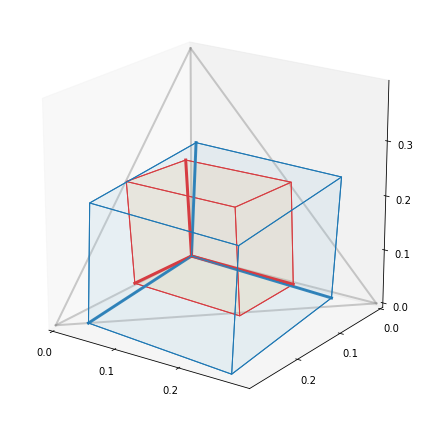}
    \caption{The Gram determinant score of the true data $\Gamma(\mP\mQ_\vx)$ is the squared volume of the blue parallelepiped spanned by column vectors in $\mP\mQ_\vx$, $vol(\mP\mQ_\vx)^2$.  Additionally, as ${\Gamma(\mP\mQ)}=  {\Gamma(\mP\mQ_\vx\mQ_{\hat{\vx}|\vx}^\intercal)} = {\Gamma(\mP\mQ_\vx\mQ_{\hat{\vx}|\vx})}$, the Gram determinant score of reported data is the squared volume of the red parallelepiped $vol(\mP\mQ_\vx \mQ_{\hat{\vx}|\vx})^2$, which is smaller than the volume of the truth data because $\mQ_{\hat{\vx}|\vx}$ is column stochastic and each column of $\mP\mQ_\vx \mQ_{\hat{\vx}|\vx}$ is a convex combination of columns of $\mP\mQ_\vx$.\protect\footnotemark}
    %\vspace{-6pt}
    % \caption{Gram determinant scores and parallelepipeds.
    % \protect\footnotemark}
    \label{fig:illustrate_vol}
    % \vspace{-20pt}
% \end{wrapfigure}
\end{figure}

\footnotetext{
%For true data, $\Gamma(\mP\mQ_\vx) = vol(\mP\mQ_\vx)^2$, the squared volume of the blue parallelepiped spanned by column vectors in $\mP\mQ_\vx$. For reported data, because $\Gamma(\mP\mQ)=\Gamma(\mP\mQ_\vx\mQ_{\hat{\vx}|\vx}^\intercal)=\Gamma(\mP\mQ_\vx\mQ_{\hat{\vx}|\vx})$, the squared volume $vol(\mP\mQ_\vx \mQ_{\hat{\vx}|\vx})^2$ corresponds to the red parallelepiped. This volume is smaller because $\mQ_{\hat{\vx}|\vx}$ is column-stochastic and each column of $\mP\mQ_\vx \mQ_{\hat{\vx}|\vx}$ is a convex combination of $\mP\mQ_\vx$.
\Cref{fig:illustrate_vol} uses $\mP = \begin{pmatrix}
        0.1&0.1&0.7\\
        0.9&0.1&0.2\\
        0&0.8&0.1
    \end{pmatrix}$,
    and $\mQ=\begin{pmatrix}
        0.03&0.27&0\\
        0.03&0.03&0.24\\
        0.28&0.08&0.04
    \end{pmatrix}$ with
    $\mQ_\vx = \begin{pmatrix}
        0.3&0&0\\
        0&0.3&0\\
        0&0&0.4
    \end{pmatrix}$, and $\mQ_{\hat{\vx}|\vx} = \begin{pmatrix}
        0.1&0.1&0.7\\
        0.9&0.1&0.2\\
        0&0.8&0.1
    \end{pmatrix}$.}

Before proving properties of the Gram determinant score, we first present some intuitions. $\mG(x,x')$ corresponds to the probability that true data $x$ and $x'$ have the same observation. The Gram matrix of reports is
\begin{equation}\label{eq:gram2gram}
    \hat{\mG} = \mQ^\intercal \mP^\intercal\mP\mQ = \mQ^\intercal\mG\mQ.
\end{equation}
%
%Recall that $\mP$ is the column stochastic matrix, and the Gram matrix of labels can simply be written as matrix product $\mG = \mP^\intercal\mP$ where $\mG(x,x')$ corresponds to the probability of observation agreement between label $x$ and $x'$,
%
%Before proving properties of the Gram determinant score, we first present some geometric intuitions.  Recall that $\mP$ is the column stochastic matrix, and the Gram matrix of labels can simply be written as matrix product $\mG = \mP^\intercal\mP$ where $\mG(x,x')$ corresponds to the probability of observation agreement between label $x$ and $x'$, and the Gram matrix of reports
%\begin{equation}\label{eq:gram2gram}
%    \hat{\mG} = \mQ^\intercal\mG\mQ = \mQ^\intercal \mP^\intercal\mP\mQ.
%\end{equation}
Moreover, $\det(\hat{\mG}) = \det((\mP\mQ)^\intercal\mP\mQ)$ is the Gram determinant of $\mP\mQ\in \R^{|\mathcal{Y}|\times |\mathcal{X}|}$ which is the square of the volume of the parallelepiped spanned by the column vectors of $\mP\mQ$~\citep{horn2012matrix}, as illustrated in~\cref{fig:illustrate_vol}.  The Gram determinant score of true data, $\Gamma(\mP\mQ_\vx)$, is the squared volume of the blue parallelepiped spanned by column vectors in $\mP\mQ_\vx$, $vol(\mP\mQ_\vx)^2$.  As ${\Gamma(\mP\mQ)}=  {\Gamma(\mP\mQ_\vx\mQ_{\hat{\vx}|\vx}^\intercal)} = {\Gamma(\mP\mQ_\vx\mQ_{\hat{\vx}|\vx})}$, the Gram determinant score of reported data is the squared volume of the red parallelepiped, $vol(\mP\mQ_\vx \mQ_{\hat{\vx}|\vx})^2$, which is smaller than that of the true data because $\mQ_{\hat{\vx}|\vx}$ is column stochastic and each column of $\mP\mQ_\vx \mQ_{\hat{\vx}|\vx}$ is a convex combination of columns of $\mP\mQ_\vx$.  A symbolic example is presented in~\cref{sec:delta_kernel}.

In the remainder of this section, we first show that the Gram determinant score preserves several reliability orderings and is invariant under experiments (\cref{sec:preserve}). We then introduce two estimators of the Gram determinant score for the detail-free setting (\cref{sec:estimator}).
%Then, although the Gram determinant score only applies in the partial knowledge setting of~\cref{eq:gram2gram}, we introduce two estimators for the detail-free setting in \cref{sec:estimator}.
Finally, we introduce kernels to generalize Gram determinant score to handle non-finite observation spaces $\mathcal{Y}$ (\cref{sec:examples}).

\subsection{Preserving Reliability Orderings and Invariance}\label{sec:preserve}
% We will provide an estimator $S: \mathcal{X}^N\times \mathcal{Y}^N\to \R$ for \eqref{eq:gramdet} in the detail-free setting that satisfies
% \begin{equation}
%     0\le \Gamma -\E[S(\hat{\vx}, \vy)]\le \epsilon.
% \end{equation}
% for some $\epsilon\ge 0$.

We show that Gram determinant reliability score preserves the exact, the Blackwell dominant, and the approximated Hamming (or $\dist$) ordering.

\begin{theorem}\label{thm:gram_preserve}
    Given $\mathcal{X} = [d]$, a finite set $\mathcal{Y}$, and $L\ge 1$, the Gram determinant score in \cref{def:gramdet} preserves
   \begin{enumerate}
  % \squishlist
        \item[1.] exact match ordering under $\mathcal{P}_{\text{indep}}$ and $\mathcal{Q}_{\text{nonperm}}$,
        \item[2.] Blackwell ordering under $\mathcal{P}_{\text{indep}}$ and $\mathcal{Q}_{\text{reg}}$, and
        \item[3.] $\frac{1}{4L\Delta}$-$\dist$ ordering under $\mathcal{P}_{\text{indep}}$ and $\mathcal{Q}_{L,1/64L^2d^2}$ for all $\dist$ with $\Delta = \frac{\max_{x,x'\in \mathcal{X}} \dist(x,x')}{\min_{x\neq x'\in \mathcal{X}} \dist(x,x')}$.
        % In particular, $\frac{1}{4L}$-Hamming ordering.
% \squishend
   \end{enumerate}
\end{theorem}

% This result accommodates any linearly independent experiment, which is necessary due to the impossibility results in \cref{sec:impossible}. Moreover, it requires little assumptions on the misreport matrices, and is nearly tight relative to the impossibility results in \cref{sec:impossible}.
% Specifically, as shown in \cref{prop:impossible_permutation}, no score can preserve exact ordering over any $\mathcal{Q}\supsetneq \mathcal{Q}_{\text{nonperm}}$; Blackwell ordering is only a strict partial ordering on $\mathcal{Q}_{\text{reg}}$ as shown in \cref{prop:blackwell_ordering}; and no score can preserve Hamming distance over $\mathcal{Q}_{\text{reg}}$.

\Cref{thm:gram_preserve} covers every linearly independent experiment and places minimal assumptions on misreports, nearly matching our impossibility results. In particular, \cref{prop:impossible,prop:comparison} show: 1) no score preserves exact ordering for any superset of $\mathcal{Q}_{\text{nonperm}}$; 2) the Blackwell relation is a strict partial order only on $\mathcal{Q}_{\text{reg}}$; and 3) no score preserves Hamming ordering or any other $\dist$ ordering on $\mathcal{Q}_{\text{dom}}$. The third part of \cref{thm:gram_preserve} implies that the score preserves $\frac{1}{4L}$-Hamming ordering, because the aspect ratio for Hamming distance is $\Delta = 1$.

The key idea of the proof is that the determinant has the multiplicative property and \cref{eq:gram2gram}, $$\Gamma(\mP\mQ) = \det(\mQ^\intercal \mP^\intercal \mP \mQ) = \det(\mQ^\intercal)\det( \mP^\intercal \mP)\det(\mQ) = \det( \mP^\intercal \mP)\det(\mQ)^2$$ because $\mQ$ and $\mP^\intercal\mP$ are square matrices.  Hence, we can decouple the misreport matrix $\mQ$ from the quality of the experiment $\mP$. In particular, it is sufficient to focus on misreport matrices as the Gram matrix of labels is positive definite $\mP^\intercal \mP$, $\det(\mP^\intercal\mP)>0$, for all $\mP\in \mathcal{P}_{\text{indep}}$.  This observation may provide a recipe for considering other reliability orderings in the Gram determinant score.  The formal proof is deferred to~\cref{app:gram_preserve}.

% Then, we show that given that the Gram matrix of labels $\mG$ is positive definite, the Gram determinant score preserves the Blackwell ordering and thus the exact ordering in \cref{lem:preserve_exact}.  Similarly, given positive definite $\mG$, \cref{lem:preserve_hamming} shows that the Gram determinant score approximately preserves Hamming ordering. This completes the proof of \cref{thm:gram_preserve}.
% \end{proof}

% \subsection{Invariance}\label{sec:invariance}
We now establish an invariance principle: the induced ranking of datasets should be invariant to the unknown experiment, to relabelings, and to priors.
The latter two are straightforward by the multiplicative property of Gram determinant.  For the first one, we show that the Gram determinant is \emph{experiment-agnostic} so that the reliability ranking of a dataset $\hat{\vx}$ should depend only on $\hat{\vx}$ and the true data $\vx$ (defined in \cref{eq:exp_agn}).  Thus the choice of experiment does not affect which reported dataset is deemed more reliable.  Moreover, we show that the Gram determinant score is the unique experiment agnostic score up to scaling under mild coherence assumption.

\begin{proposition}\label{prop:exp_agn}
Given $\mathcal{X} = [d]$ and a finite set $\mathcal{Y}$, the Gram determinant score in \cref{def:gramdet} is experiment agnostic so that for all $\mQ, \mQ'\in GL_d$ general linear group and $\mP\in \mathcal{P}_{\text{indep}}$,
\begin{equation}\label{eq:exp_agn}
    \Gamma(\mQ)\ge \Gamma(\mQ')\Leftrightarrow \Gamma(\mP\mQ)\ge \Gamma(\mP\mQ').
\end{equation}

Moreover, if there exists a continuous function $S:GL_d\to \R_{>0}$ with a continuous $c:\R_{>0}\to \R_{>0}$ so that for all $\mQ, \mQ', \mP\in GL_d$, and $t>0$, \cref{eq:exp_agn} holds and $S(t\mQ) = c(t) S(\mQ)$, there exists $\alpha>0, \beta\neq 0$ so that
$S(\mQ) = \alpha \det(\mQ^\intercal\mQ)^\beta.$
\end{proposition}

% Note that $GL_d\subset \mathcal{P}_{\text{indep}}$.
As discussed above, the first part follows directly from multiplicative property of determinant, $\Gamma(\mP\mQ) = \det( \mP^\intercal \mP)\det(\mQ)^2 = \det( \mP^\intercal \mP)\Gamma(\mQ)$.  We defer the proof for the second part to~\cref{app:exp_agn}.  Finally, since $GL_d\subset \mathcal{P}_{\text{indep}}$, the second part of \cref{prop:exp_agn} implies that even when we restrict to settings where the observation space has the same dimension as the data space $|\mathcal{Y}| = |\mathcal{X}|$, the Gram determinant score remains unique up to scaling.

\subsection{Estimators for Gram Determinant Scores}\label{sec:estimator}
We introduce two estimators for the Gram determinant score in the detail-free setting: plug-in and stratified matching estimator.
\begin{definition}[plug-in Gram determinant reliability score]\label{def:plug-in}  Given $\hat{\vx}$ and $\vy$ of size $N$, define $\bar{\mG}\in \R^{d\times d}$ so that for all $x,x'\in \mathcal{X}$
$\bar{\mG}(x,x') = \frac{1}{N^2}\sum_{n, n'\in[N]: \hat{x}_n = x, \hat{x}_{n'} = x'}\mathbf{1}[y_n = y_{n'}].$
The plug-in Gram determinant reliability score is then defined as
$\bar{S}(\hat{\vx}, \vy) = \det(\bar{\mG})$.
\end{definition}

The plug-in estimator first estimates $\hat{\mG}$ using empirical distribution between reports $\hat{\vx}$ and observations $\vy$ and computes the determinants of $\hat{\mG}$.
Note that the probability of $y_n = y_{n'}$ is simply the inner product of $P_{x_n}$ and $P_{x_{n'}}$ if $n\neq n'$.  \Cref{thm:plug-in} shows
 that the plug-in estimator asymptotically preserves all reliability orderings in \cref{thm:gram_preserve}.

\begin{proposition}\label{thm:plug-in}
    Given $\mathcal{X} = [d]$, finite set $\mathcal{Y}$ and $L\ge 1$, the plug-in Gram determinant score in \cref{def:plug-in} asymptotically preserves reliability orderings in \cref{thm:gram_preserve}.
\end{proposition}

% \paragraph{Stratified matching estimator}

While the above plug-in estimator can asymptotically preserve all reliability orderings in \cref{thm:gram_preserve}, it lacks provable guarantees for finite samples. In practice, only a limited number of observations are available, and a strategic data source may aim to maximize its reliability score. \Cref{def:stratified} provides a finite-sample estimator that preserves the exact-match ordering, rewarding truthful reporting more than any alternative report.

\begin{definition}[stratified matching Gram determinant reliability score]\label{def:stratified}
    Given $\mathcal{X} = [d]$ and $\hat{\vx}, \vy$ of size $N$, a \emph{stratified-matching estimator} for the Gram determinant score is defined as follows:
\begin{enumerate}
% \squishlist
    \item[1.] Return $0$ if the minimum occurrence $\min_{x\in \mathcal{X}} |\{n\in [N]: \hat{x}_n = x\}|$ is less than $2$. Otherwise, randomly select two disjoint index sets $Col, Row\subseteq [N]$ of size $d$ in which each label $i\in \mathcal{X}$ occurs exactly once. Then re-index them as two sequences of pairs $(\hat{x}_{i,Col}, y_{i,Col})_{i\in [d]}$ and $(\hat{x}_{i,Row}, y_{i,Row})_{i\in [d]}$ so that $\hat{x}_{i,Col} = \hat{x}_{i,Row} = i$ for all $i\in \mathcal{X}$. We call them the column and row sequences, respectively.
    \item [2.] Randomly sample one permutation $\sigma\in sym(d)$, and return
    \begin{equation}\label{eq:stratified}
    score(\hat{\vx}, \vy) := d! sgn(\sigma)\prod_{i,j\in [d], j = \sigma(i)} \mathbf{1}\left[y_{i,Row} = y_{j,Col}\right]\vq_{\hat{\vx}}(i)\vq_{\hat{\vx}}(j).
    \end{equation}
% \squishend
\end{enumerate}
\end{definition}
\Cref{eq:stratified} approximates the second form of  the Gram determinant in~\cref{eq:gramdet} by summing over all permutations. The first step is a stratified sampling to collect one report of each label in $Col$ and $Row$ respectively.  The term $\mathbf{1}[y_{i,Row} = y_{j,Col}]$ approximates the inner product between the observation distributions of reports $i$ and $j$, and the extra $\vq_{\hat{\vx}}(i)\vq_{\hat{\vx}}(j)$ offset the stratified sampling.

The stratified-matching estimator only requires each reported label to occur at least twice. If any reported label occurs fewer than two times, the estimator returns zero and yields a worse score than truthful data. The following result shows that, under mild balance conditions, the stratified-matching estimator preserves exact-match ordering over linearly independent experiments.

\begin{proposition}\label{thm:stratified}
    Given $\mathcal{X} = [d]$ and $L\ge 1$, the stratified-matching estimator in \cref{def:stratified} preserves exact ordering on $\mathcal{P}_{\text{indep}}$, $\mathcal{Q}_{L}$, and $N\ge 2Ld$.
\end{proposition}

\subsection{Gram Determinant Score with Kernels}\label{sec:examples}
\label{sec:kernels}

% \label{sec:examples}
The Gram determinant score in \cref{def:gramdet} has two limitations. First, it cannot handle continuous or general observation spaces $\mathcal{Y}$. Second, it ignores intrinsic structure in the observation space, such as predictions or feature embeddings. We extend the Gram determinant score with kernels.
% We provide examples of different kernels that can be used in practice, together with a reliability-ordering result analogous to \cref{thm:gram_preserve} in \cref{app:examples}.
\begin{definition}[kernelized Gram determinant score]\label{def:gramdet_kernel}
    Given a finite set $\mathcal{X}$, an experiment $\mP$, and a kernel $K: \mathcal{Y}\times \mathcal{Y}\to \R$, define the Gram matrix of labels $\mG_K\in \R^{d\times d}$ by, for all $x, x'\in \mathcal{X}$, $\mG_K(x,x') = \langle P_x, P_{x'}\rangle_K := \E_{y\sim P_x, y'\sim P_{x'}}[K(y,y')]$. Given $\vx$ and $\hat{\vx}$, define the Gram matrix of reports $\hat{\mG}_K\in \R^{d\times d}$ by
$\hat{\mG}_K(x,x') = \frac{1}{N^2}\sum_{n, n': \hat{x}_n = x, \hat{x}_{n'} = x'}\langle P_{x_n}, P_{x_{n'}}\rangle_K.$
Finally, we define the \emph{Gram determinant score} with kernel $K$ as  $\Gamma_K:= \det\left(\hat{\mG}_K\right)$.
\end{definition}

We next provide examples that motivate the kernelized Gram determinant score in~\cref{def:gramdet_kernel}.
\begin{enumerate}
% \squishlist
    \item[1.] Given any feature map $\phi:\mathcal{Y}\to \R^k$ that maps observations to Euclidean space, we define $K(y,y') = \langle \phi(y), \phi(y')\rangle$ as the standard inner product between the features.  A feature map is \emph{injective} if the vectors $\{\phi(y)\}_{y\in \mathcal{Y}}$ are linearly independent.  For instance, using the one-hot encoder $\phi: y\mapsto \delta_y$ results in delta-kernel $K(y,y') = \mathbf{1}[y = y']$ and reproduces \cref{def:gramdet}.
    \item[2.] More generally, we can consider implicit feature maps, such as the Gaussian radial basis function
$K(y,y') = \exp\left(\frac{-\|y-y'\|_2^2}{\sigma^2}\right)$ for $\mathcal{Y}\subseteq \R^k$, or a general Hilbert space.~\citep{ziegel2022characteristickernelshilbertspaces}
\item[3.] We can use feature maps to incorporate special structure in $\mathcal{Y}$, such as predictions of true labels. Formally, given $\mP$, we say an observation $y$ is a pseudo-posterior with prior $\tilde{q}\in \Delta(\mathcal{X})$ if $y = \{\tilde{\Pr}[\rx = x|y]\}_{x\in \mathcal{X}} = \{\frac{P(y,x)\tilde{q}(x)}{\sum_{x'} P(y,x')\tilde{q}(x')}\}_{x\in \mathcal{X}}$ is the posterior of the true label under prior $\tilde{q}$.~\citep{kass1996selection} Rather than a one-hot encoder, we may use $\phi(y) = y\in \R^{d}$, which yields a smaller, meaningful feature space. We call the associated kernel $K(y,y') = y^\intercal y'$ with a pseudo-posterior experiment the \emph{pseudo-posterior kernel}.
\end{enumerate}
% \squishend

We show that kernelized Gram determinant reliability scores also preserve all reliability orderings in \cref{thm:gram_preserve} for general observation space $\mathcal{Y}$ under three kernel families. First, the result holds for any integrally strictly positive-definite kernel, so admitting arbitrary (possibly infinite or continuous) observation spaces. When $\mathcal{Y}$ is finite, one may use any kernel with injective feature maps. The guarantee also holds when the observations are pseudo-posteriors with arbitrary prior $\tilde{q}$ with full support.

\begin{theorem}\label{thm:gram_preserve_kernel}
    Given $\mathcal{X} = [d]$, $\mathcal{Y}$, and $L\ge 1$, the Gram determinant score with any of the following kernels in \cref{def:gramdet_kernel} preserves the reliability orderings in \cref{thm:gram_preserve}:
\begin{enumerate}
% \squishlist
            \item[1.] Integrally strictly positive definite kernels---in particular the Gaussian (RBF) kernel on any separable Hilbert space $\mathcal{Y}$.
        \item[2.] Kernels with an injective feature map $\phi: \mathcal{Y}\to \R^k$ and finite set $\mathcal{Y}$.
        \item[3.] Pseudo-posterior kernel $K$ with full-support $\tilde{q}$.
% \squishend
   \end{enumerate}
\end{theorem}
The proof is mostly identical to that of \cref{thm:gram_preserve}.  As the kernel only changes the Gram matrix of labels $\mG$, it is sufficient to show $\mG$ is positive definite to reuse \cref{lem:preserve_exact,lem:preserve_hamming}.  Similarly, those two estimators in \cref{sec:estimator} can also adopt kernels. We provide details in~\cref{app:examples}.

\begin{definition}[plug-in Kernelized Gram determinant reliability score]\label{def:plug-in_kernel}
Given a kernel $K:\mathcal{Y}^2\to \R$, $\hat{\vx}, \vy$ of length $N$, let $\bar{\mG}_K: \mathcal{X}\times\mathcal{X}\to \R$ where
$$\bar{\mG}_K(x,x') = \frac{1}{N^2}\sum_{n, n'\in[N]: \hat{x}_n = x, \hat{x}_{n'} = x'}K(y_n, y_{n'}).$$
The plug-in kernelized Gram determinant reliability score is
$\bar{S}_K(\hat{\vx}, \vy) = \det(\bar{\mG}_K)$
\end{definition}

\begin{theorem}
Given $\mathcal{X} = [d]$, a finite set $\mathcal{Y}$, and $L\ge 1$, the plug-in Gram determinant score with any bounded kernel in \cref{thm:gram_preserve_kernel} asymptotically preserves the reliability orderings in \cref{thm:gram_preserve}.
\end{theorem}

\begin{lemma}\label{lem:concent_emp_kernel}
Given $|K|\le 1$, $\delta>0$ and report length $N$,
    $$\Pr\left[\|\bar{\mG}_K-\hat{\mG}_K\|_2\le 4\sqrt{\frac{\log 2d/\delta}{N}}+2\frac{\log 2d/\delta}{N}\right]\ge 1-\delta$$
\end{lemma}

The above lemma shows that the empirical estimator $\bar{\mG}_K$ is close to its expectation $\hat{\mG}_K$ in spectral norm. The proof is mostly identical to that of \cref{lem:concent_emp}; it only requires a concentration result for sums of independent random elements in Hilbert spaces~\cite[Theorem 3.5]{pinelis1994optimum}.

Finally, we can also design an estimator that preserves exact match ordering even with finite length data.
\begin{definition}\label{def:stratified_kernel}
    Given a kernel $K: \mathcal{Y}\times\mathcal{Y}\to \R$ and $\hat{\vx}, \vy$ of length $N$, a \emph{stratified-matching estimator} estimates the kernelized Gram determinant as follows:
\begin{enumerate}
    \item Return $0$ if the minimum occurrence $\min_{x\in \mathcal{X}} |\{n\in [N]: \hat{x}_n = x\}|$ is less than $2$. Otherwise, randomly select two disjoint index sets $Col, Row\subseteq [N]$ of size $d$ in which each label $i\in \mathcal{X}$ occurs exactly once. Then re-index them as two sequences of pairs $(\hat{x}_{i,Col}, y_{i,Col})_{i\in [d]}$ and $(\hat{x}_{i,Row}, y_{i,Row})_{i\in [d]}$ so that $\hat{x}_{i,Col} = \hat{x}_{i,Row} = i$ for all $i\in \mathcal{X}$. We call them the column and row sequences, respectively.
    \item Randomly sample one permutation $\sigma\in sym(d)$, and return
    \begin{equation}\label{eq:stratified_kernel}
    score(\hat{\vx}, \vy) := d! sgn(\sigma)\prod_{i,j\in [d], j = \sigma(i)} K(y_{i,Row}, y_{j,Col})\vq_{\hat{\vx}}(i)\vq_{\hat{\vx}}(j).
    \end{equation}
\end{enumerate}

\end{definition}

\begin{theorem}\label{thm:stratified_kernel}
    Given $\mathcal{X} = [d]$ and $L\ge 1$, the stratified-matching estimator in \cref{def:stratified_kernel} with any kernel in \cref{thm:gram_preserve_kernel} preserves exact ordering on $\mathcal{P}_{\text{indep}}$, $\mathcal{Q}_{L}$, and $N\ge 2Ld$.
\end{theorem}
The proof is mostly identical to that of \cref{thm:stratified}.

\begin{remark}Our Gram determinant score can be viewed as an application of the determinant mutual information mechanism introduced in \citep{10.1145/3638239}, where one agent’s report is replaced with the observation $\vy$. In addition to offering a more fine-grained characterization of the Gram determinant score, as discussed in related work, we also introduce several technical improvements over the original determinant mutual information method. First, the prior approach requires $\mathcal{Y} = \mathcal{X}$ and overlooks potential structure in the observations. As shown in \cref{sec:examples}, we address this by introducing kernel methods, allowing us to generalize the score to arbitrary observation spaces $\mathcal{Y}$—a crucial extension for handling continuous observations such as Gaussian variables or image embeddings, as demonstrated in \cref{sec:experiment}. Second, our stratified-matching estimators in \cref{def:stratified,def:stratified_kernel} are unbiased in the multi-task peer prediction setting of \citep{10.1145/3638239}, and they reduce the estimator's range from order $(d!)^2$ to $d!$.
\end{remark}

\section{Experiments}\label{sec:experiment}

We evaluate the Gram determinant score in four parts.  (Exp. 1) synthetic categorical data with six label-manipulation policies in~\cref{sec:exp1}; (Exp. 2) real image data (CIFAR-10 embeddings) with the same six manipulations using the kernelized score in~\cref{sec:exp2}; (Exp. 3) real employment data, treating CES vintage revisions as naturally occurring manipulations in~\cref{sec:exp3}; and (Exp. 4) comparison to other scores in~\cref{sec:exp4}.

\subsection{Experiment 1: Gram Determinant Score on Synthetic Data}\label{sec:exp1}

In this experiment, we evaluate how well the Gram determinant score captures label reliability under categorical observations, as summarized in \cref{fig:all_results2,fig:delta2}.
%In this experiment we assess the ability of a Gram determinant score to reflect label reliability under categorical observation summarized in~\cref{fig:all_results2}.
Specifically, we first generate a ground-truth dataset $(\vx,\vy)$ of size $N=4000$ with $d=5$. Each label $x_k$ is drawn uniformly from ${1,\dots,d}$ for $k\in[N]$, and each outcome $y_k$ is sampled from the distribution $\mP(\cdot\mid x_k)$, where the experiment distribution matrix $\mP\in[0,1]^{d\times d}$ is constructed by sampling $\mP(i,j)\sim\mathrm{Uniform}(0,1)$ independently and normalizing columns to be stochastic. The ground-truth dataset $(\vx,\vy)$ is fixed across all trials.
To model varying reliability, for each $p\in\{0.00,0.05,\dots,0.50\}$ we corrupt the labels according to
\begin{equation}\label{eq:uniform_misreport}
\hat{x}_k =
\begin{cases}
x_k, & \text{with probability }1-p,\\
Z_k, & \text{with probability }p,
\end{cases}
\end{equation}
where $Z_k\sim \pi(\cdot\mid x_k)$ is independently drawn from a corruption policy $\pi$; in our experiments, $\pi$ is instantiated by one of the six manipulations below.

\begin{itemize}
     \item Uniformly random: $Z_k\sim\mathrm{Uniform}\{1,\dots,d\}$.
     \item Asym neighbor: with probability $0.85$ set $Z_k=\min\{x_k+1,d\}$, otherwise sample $Z_k$ uniformly from $\{1,\dots,d\}\setminus\{x_k\}$.
     \item Row-sim 2nd: $Z_k=\arg\max_{j\neq x_k}\frac{\langle \mP_{x_k,\cdot},\,\mP_{j,\cdot}\rangle}{\|\mP_{x_k,\cdot}\|\,\|\mP_{j,\cdot}\|}$, the label with closest observation distribution.
     \item Merge $0/1\to 0$: if $x_k\in\{1,2\}$ then set $Z_k=1$; otherwise $Z_k=x_k$.
     \item Group up/down: $Z_k=\min\{x_k+1,d\}$ with probability $1/2$, or $Z_k=\max\{x_k-1,1\}$ otherwise.
     \item Mixed: sample $Z_k\sim \pi_{\text{mixed}}(\cdot|x_k)$ where each row $\pi_{\text{mixed}}(i,\cdot)$ is drawn from $$\mathrm{Dirichlet}\big(\alpha_i(1),\dots,\alpha_i(d)\big)$$ with
\begin{align*}
\alpha_i(j)&=\alpha_{\text{off}}+\alpha_{\text{diag}}\mathbf{1}\{j=i\}
+\lambda_{\text{loc}}\exp\big(-\mathrm{dist}_{\text{ring}}(i,j)\big)
\\&+\lambda_{\text{up}}\exp\big(\gamma(j-i)\big)
+\lambda_{\text{def}}\mathbf{1}\{j=j_0\},
\end{align*}
$\mathrm{dist}_{\text{ring}}(i,j)=\min(|i-j|,\,d-|i-j|)$, and $j_0$ a salient default label; rows are normalized to be stochastic, where $\alpha_{\text{off}}=0.2, \alpha_{\text{diag}}=6, \lambda_{\text{loc}}=1.0, \lambda_{\text{up}}=0.4, \gamma=0.5, \lambda_{\text{def}}=0.6, j_0=1$.
This policy mimics human labeling: diagonal dominance (keep $i$), locality on the ring (near-class confusions), mild upcoding (asymmetric mistakes), and a default-label bias—yielding structured, non-uniform noise beyond uniform corruption.
\end{itemize}

Fix a ground-truth dataset $(\vx, \vy)$. For each manipulation and corruption level $p\in \{0.00,0.05,\dots,0.50\}$, in~\cref{fig:det_vs_p2,fig:hamming_vs_det2,fig:l2_vs_det2}, we run $M=100$ independent trials, producing corrupted reports $\hat{\vx}^m$. In every trial, we compute 1) the plug-in Gram determinant reliability score in \cref{def:plug-in}, 2) the Hamming error $\sum_{n=1}^N \mathbf{1}[x_n \neq \hat{x}_n^m]$, and 3) the $\ell_2$ error $\lVert \vx - \hat{\vx}^m \rVert_2$. We then report the mean and standard deviation of each metric across the $M$ trials. In \cref{fig:det_vs_p2}, the plug-in Gram-determinant score falls steadily as the corruption probability $p$ increases. \Cref{fig:hamming_vs_det2,fig:l2_vs_det2} show that higher scores correspond to lower Hamming error and smaller $\ell_2$ deviation, respectively, demonstrating a clear negative correlation between our score and these conventional error measures regardless of the manipulation policy (i.e., across all corruption schemes considered).

% We repeat the procedure for each corruption level $p$, running $M=100$ independent trials for every manipulation and $p$. The ground‐truth dataset $(\vx,\vy)\in [d]^N$ remains fixed across trials, while each trial $m$ produces a corrupted version $\hat{\vx}^m$. For every trial we compute the Hamming error $\sum_{n=1}^N \mathbf{1}[x_n \neq \hat{x}_n^m]$, the $\ell_2$ error $\lVert \vx - \hat{\vx}^m \rVert_2$, and the plugged‐in Gram determinant reliability score. We then report the mean and standard deviation of each metric over the $M$ trials.

\begin{figure}[ht]
  \centering
  \begin{subfigure}[b]{0.4\textwidth}
    \includegraphics[width=\textwidth]{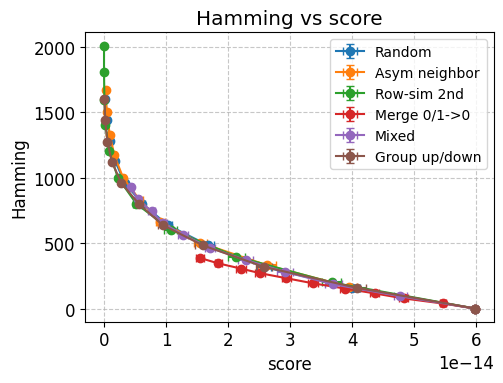}
    \caption{Hamming vs.\ $score$}
    \label{fig:hamming_vs_det2}
  \end{subfigure}
  % \hfill
  \begin{subfigure}[b]{0.4\textwidth}
    \includegraphics[width=\textwidth]{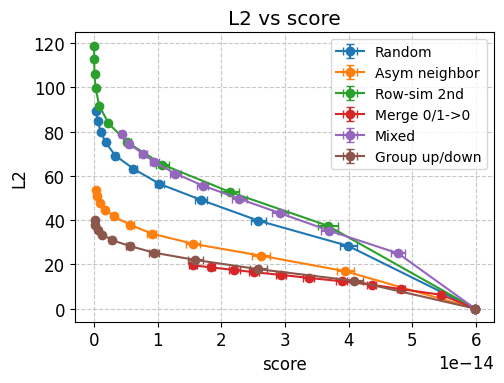}
    \caption{$\ell_2$ norm vs.\ $score$}
    \label{fig:l2_vs_det2}
  \end{subfigure}
  \caption{Comparisons between Gram determinant reliability
score and Hamming distance/$\ell_2$ norm on categorical synthetic data.}
  \label{fig:all_results2}
\end{figure}

\begin{figure}[ht]
  \centering
  \begin{subfigure}[b]{0.4\textwidth}
    \includegraphics[width=\textwidth]{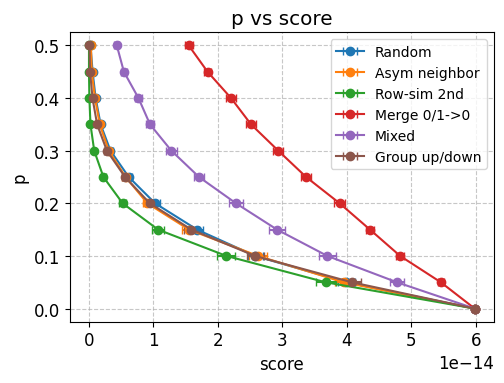}
    \caption{$score$ vs.\ $p$}
    \label{fig:det_vs_p2}
  \end{subfigure}
  % \hfill
  \begin{subfigure}[b]{0.4\textwidth}
    \includegraphics[width=\textwidth]{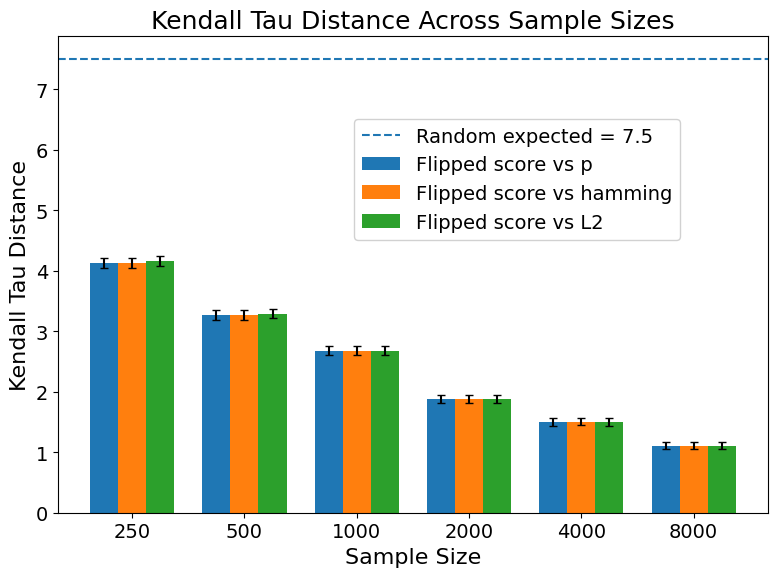}
    \caption{Kendall--tau distance}
    \label{fig:delta2}
  \end{subfigure}
  \caption{Comparison of Gram determinant reliability
score and $p$ and Kendall-tau distance comparison on categorical synthetic data.}
  \label{fig:all_results2_others}
\end{figure}

In~\cref{fig:delta2}, we vary data sizes \(N\in\{250,500,\dots,8000\}\) and generate 1000 datasets for each $N$. In each dataset and corruption level \(p\in\{0.0,0.1,\dots,0.5\}\), we use the uniform random manipulation strategy, and compute the plug-in Gram determinant, Hamming-distance error, and \(\ell_2\) error, then rank the six corrupted reports.
We report the average Kendall--tau distance between the reversed Gram--determinant ranking and the orderings induced by $p$, the Hamming distance, and the $\ell_2$ error. As shown in \Cref{fig:delta2}, the fraction of correctly recovered rankings increases with the sample size, confirming that the Gram--determinant score is a consistent indicator of true-label reliability.
% We report the proportion of datasets in which the reversed Gram determinant ranking matches the orderings induced by \(p\), Hamming distance, and the \(\ell_2\) error.\footnote{Under random guessing, any ranking has probability \(1/6!\approx 0.00139\) of agreement. See appendix for details.}  \Cref{fig:delta2} shows that the fraction of rankings rises as the sample size grows, confirming the Gram‑determinant score being a consistent indicator of true label reliability.
%

Moreover, the curves corresponding to different manipulation schemes in
\cref{fig:hamming_vs_det2} and~\cref{fig:l2_vs_det2} are themselves comparable.
Since both the Hamming distance and the $\ell_2$ norm are defined purely
at the dataset level and do not depend on the specific manipulation type,
consistency across manipulation schemes is desirable.
In particular, tighter clustering of these curves indicates that the
error metric behaves uniformly regardless of how the corruption is
generated.
In contrast, curves indexed by the corruption probability $p$ are not
directly comparable across manipulation schemes, as $p$ measures the
strength of manipulation whose effect is inherently scheme dependent.
From this perspective, our reliability score exhibits strong and
consistent monotonic behavior with respect to the Hamming distance, where
theoretical guarantees apply, and also demonstrates meaningful correlation
with the $\ell_2$ norm, despite the absence of direct theoretical support
for this metric.

% \begin{figure}
%     \centering
%     \includegraphics[width=0.4\linewidth]{sections/images/exp2-delta.png}
%     \caption{Matched rankings on categorical synthetic data}
%     \label{fig:delta2}
% \end{figure}

\subsection{Experiment 2: Gram Determinant Score with Kernels on Image Data}\label{sec:exp2}
We evaluate the Gram determinant score with continuous observations by using image embeddings.  We train a SimCLR model \citep{chen2020simple} with a ResNet-18 backbone and an 8-dimensional projection head on CIFAR-10 \citep{krizhevsky2009learning}. The model is optimized for 60 epochs using the InfoNCE loss with batch size $B=256$, temperature $\tau=0.5$, and the Adam optimizer at learning rate $5\times10^{-3}$. After training, we extract normalized projections $y_n\in\mathbb{R}^8$ for each of the $N=10000$ test images, denote the true labels by $\vx\in\{0,\dots,9\}^N$, and the embeddings by $\vy\in\mathbb{R}^{N\times 8}$.

To simulate corrupted reports, we use the same six corruption policies with $p\in\{0.00,0.04,\dots,0.40\}$. As $\mathcal{Y}=\R^{8}$ is continuous, we use plug-in Gram determinant with kernel $K(y,y')=\langle y,y'\rangle$ as the score. For each $p$ and policy we repeat the procedure over $M=100$ random trials to obtain the mean and standard error. As shown in \cref{fig:contrastive_det_vs_p3,fig:contrastive_hamming_vs_det3,fig:contrastive_l2_vs_det3}, the score decreases monotonically with $p$ across all six manipulations, and a higher score is associated with lower Hamming error and smaller $\ell_2$ deviation, mirroring the trends observed in the categorical setting.
% For each $p$ and each policy we repeat the procedure over 100 random seeds to obtain the sample mean and standard error. As shown in \cref{fig:contrastive_det_vs_p3}, \cref{fig:contrastive_hamming_vs_det3} and \cref{fig:contrastive_l2_vs_det3}, the $score$ increases monotonically with $p$ \emph{across all six manipulations}, and higher $score$ is associated with lower Hamming error and smaller $\ell_2$ deviation, mirroring the trends observed in the categorical setting.

% Let
% $$
% I_i=\{\,k:\hat{X}_k=i\},\quad i=0,\dots,9,
% $$
% and form the scaled embeddings $\widetilde Y=10\,Y$.  We compute the class‐mean embedding
% $$
% \bar y_i=\frac{1}{|I_i|}\sum_{k\in I_i}\widetilde Y_k,
% $$
% and define the convex‐hull volume score
% $$
% score
% =\mathrm{Vol}\bigl(\mathrm{ConvHull}(\{\bar y_i\}_{i=0}^9)\bigr).
% $$

\begin{figure}[ht]
  \centering
  % \begin{subfigure}[b]{0.32\textwidth}
  %   \includegraphics[width=\textwidth]{sections/images/contrastive_system_p.png}
  %   \caption{$score$ vs.\ $p$}
  %   \label{fig:contrastive_det_vs_p3}
  % \end{subfigure}
  % \hfill
  \begin{subfigure}[b]{0.4\textwidth}
    \includegraphics[width=\textwidth]{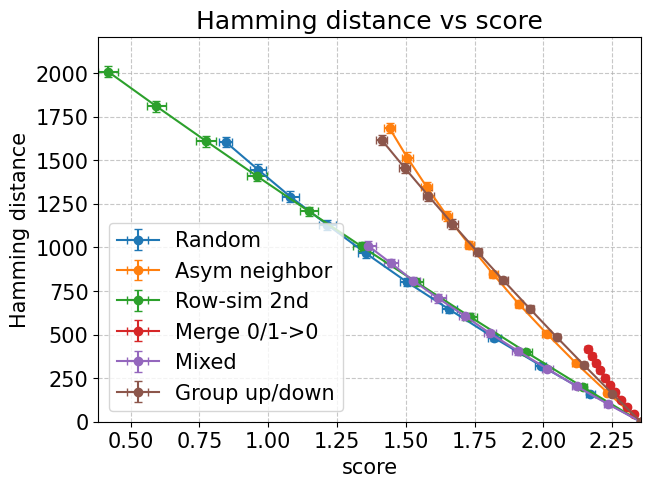}
    \caption{Hamming vs.\ $score$}
    \label{fig:contrastive_hamming_vs_det3}
  \end{subfigure}
  % \hfill
  \begin{subfigure}[b]{0.4\textwidth}
    \includegraphics[width=\textwidth]{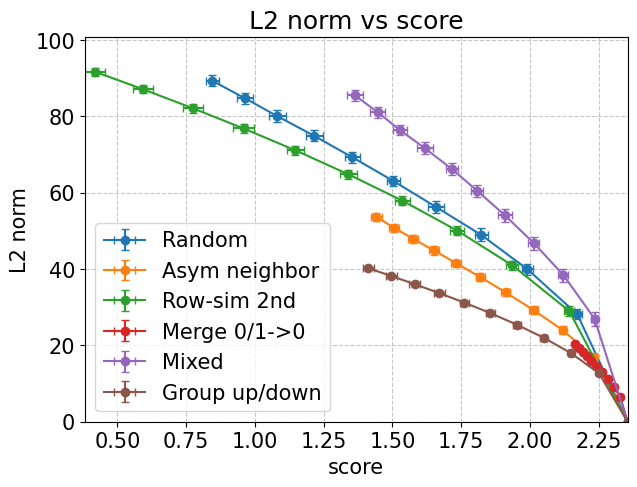}
    \caption{$\ell_2$ norm vs.\ $score$}
    \label{fig:contrastive_l2_vs_det3}
  \end{subfigure}
  % \hfill
  % \begin{subfigure}[b]{0.21\textwidth}
%   \raisebox{-0.16cm}[0pt][0pt]{%
%     \includegraphics[width=\textwidth]{sections/images/employment.png}%
%   }
%     % \includegraphics[width=\textwidth]{sections/images/employment.png}
%     \caption{Employment data}
%     \label{fig:gram-scores-employment}
% \end{subfigure}
  \caption{Comparison of Gram determinant reliability
score and Hamming distance/$\ell_2$ norm  for image–label experiments.}
  \label{fig:contrastive_plus_employment}
\end{figure}

\begin{figure}
    \centering
    \includegraphics[width=0.4\linewidth]{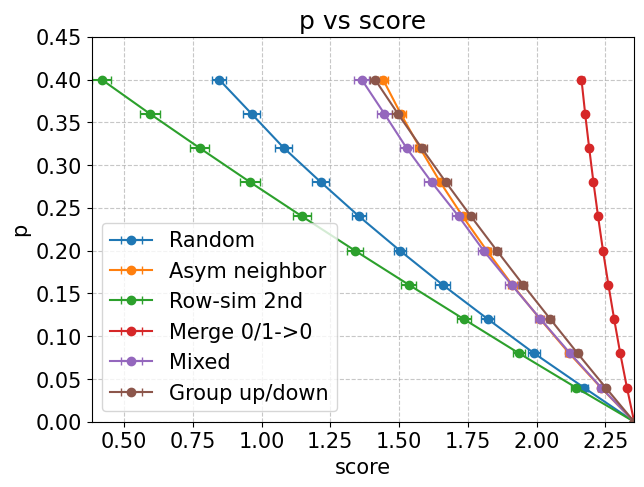}
    \caption{Comparison of Gram determinant reliability
score and $p$ for image–label experiments.}
    \label{fig:contrastive_det_vs_p3}
\end{figure}

\subsection{Experiment 3: Gram Determinant Score on Real-World Employment Data}\label{sec:exp3}

We evaluate three vintages of the CES total nonfarm employment series (not seasonally adjusted) from Oct 2005–Feb 2023, using the CES vintage dataset~\citep{bls_ces_vintage_2025}, and take as external $\vy$ the monthly changes in Withheld Income \& Employment Taxes from Treasury deposits~\citep{treasury_dts_federal_tax_deposits_2025}. For each month we use: 1) first release: initial estimate, published the next month; 2) one-month revision: first revision, one month later; and 3) final value: last available vintage including benchmark revisions.
We discretize month-to-month differences into four quantile buckets as $\vx$ and $\vy$ with $N=209$ and compute Gram determinant scores with the plug-in estimator.
\Cref{tab:employment_reliability} shows that revisions substantially improve reliability according to our score, with the final series most aligned with fiscal benchmarks.

\begin{table}[ht]
  \centering
  \begin{tabular}{l r}
    \toprule
    Version & Gram Det Score \\
    \midrule
    First Release       & $3.504\times 10^6$ \\
    One-Month Revision  & $24.920\times 10^6$ \\
    Final Value       & $33.919\times 10^6$ \\
    \bottomrule
  \end{tabular}
  \caption{Employment data reliability.}
  \label{tab:employment_reliability}
\end{table}

\subsection{Experiment 4: Comparison between Gram Determinant Score and Other Scores}\label{sec:exp4}

We compare our method with several classical measures of stochastic relationships between random variables (KL-divergence, $\chi^2$-mutual information, maximal correlation, and Top-$k$ volume). We adapt these metrics as reliability scores for reported data $\hat{\vx}$ and observations $\vy$. While we do not attempt a formal theoretical comparison, we use these measures to benchmark their performance against the Gram Determinant Score empirically.

We first define these four measures. Let $\mJ:=\mP\mQ$ denote the joint distribution, and let $P_{\hat{\vx}}$ and $P_{\vy}$ denote the marginal distributions of the report $\hat{\vx}$ and observation $\vy$, respectively. We define the baselines below and discuss additional candidates in~\cref{app:alternative_score}.
\begin{enumerate}
    \item KL-divergence~\cite{kullback1951information}: $S_{KL}(\mJ) = \sum_{x\in \mathcal{X}, y\in \mathcal{Y}}\mJ(y,x)\log \left(\frac{\mJ(y,x)}{P_{\hat{\vx}}(x)P_{\vy}(y)}\right)$ which is also known as (Shannon) mutual information.
    \item $\chi^2$-MI~\cite{ali1966general}: $I_\chi(\mJ) = \|\bar{J}\|_F^2 = \sum_{i=1}^d \bar{s}_i^2$ where $\bar{\mJ}$ is the whitened matrix\footnote{Let $\mD_{\hat{x}}$ and $\mD_{y}$ be the diagonal matrices of the marginals. The whitened matrix $\bar{\mJ}= \mD_{y}^{-1/2}(\mJ - P_{\vy}P_{\hat{\vx}}^{\intercal})\mD_{\hat{\vx}}^{-1/2}$} of $\mJ$, $\|\cdot \|_F$ is the Frobenius norm, and $\bar{s}_i$ is the $i$-th largest singular value of $\bar{\mJ}$.
    \item Max correlation~\cite{gebelein1941statistische}: $\psi_{\max}(\mJ) = \bar{s}_1$ is the largest singular value of the whitened matrix $\bar{\mJ}$.
    % It can equivalently be written as $\max_{(f,g) \in \mathcal{F}} \mathbb{E}[f(\hat{x})g(y)]$, where $\mathcal{F}$ is the collection of real-valued random variables with zero mean and unit variance.
    \item Top-$k$ volume: $\psi_{\wedge k}(\mJ) = \prod_{i=1}^{k} \bar{s}_i$ is the product of the largest $k$ singular values of $\bar{\mJ}$.
    % The Gram determinant score corresponds to the special case where $k = d$.
\end{enumerate}
We follow the same data generation process and manipulation policies as in
Experiment~1 (\cref{fig:all_results2}).  We evaluate the above four candidate reliability scores by computing them on the empirical joint distribution, which serves as a proxy for the joint distribution $\mJ = \mP\mQ$.

\begin{figure}[htbp]
    \centering

    \begin{subfigure}[b]{0.4\textwidth}
        \includegraphics[width=\textwidth]{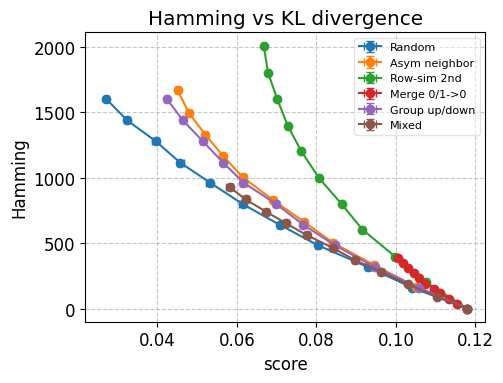}
        \caption{Hamming vs.\ KL divergence}
        \label{fig:kl-hamming}
    \end{subfigure}
    \begin{subfigure}[b]{0.4\textwidth}
        \includegraphics[width=\textwidth]{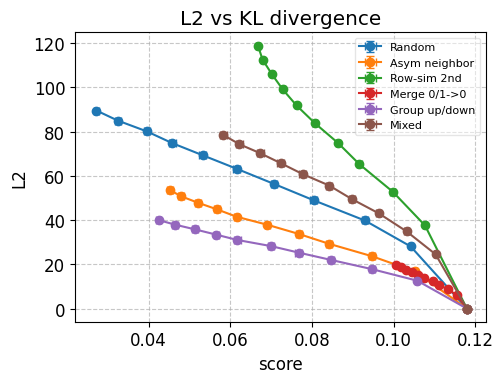}
        \caption{$\ell_2$ norm vs.\ KL divergence}
        \label{fig:kl-l2}
    \end{subfigure}

    \begin{subfigure}[b]{0.4\textwidth}
        \includegraphics[width=\textwidth]{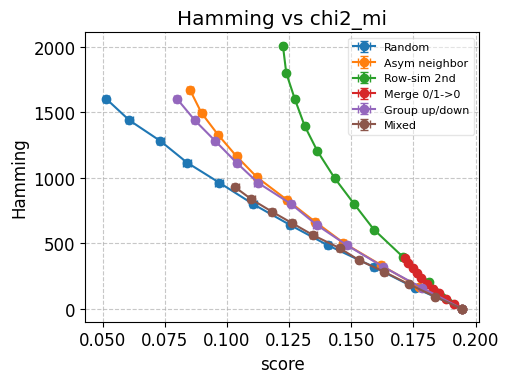}
        \caption{Hamming vs.\ $\chi^2$-MI}
        \label{fig:chi2-hamming}
    \end{subfigure}
    \begin{subfigure}[b]{0.4\textwidth}
        \includegraphics[width=\textwidth]{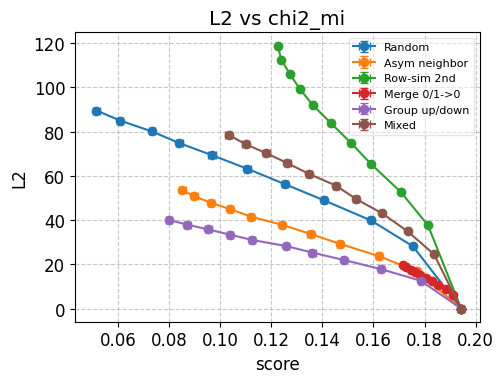}
        \caption{$\ell_2$ norm vs.\ $\chi^2$-MI}
        \label{fig:chi2-l2}
    \end{subfigure}

    \begin{subfigure}[b]{0.4\textwidth}
        \includegraphics[width=\textwidth]{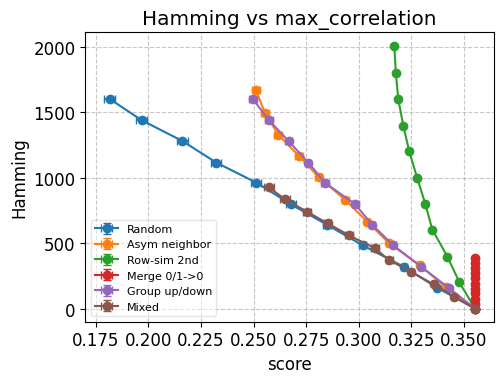}
        \caption{Hamming vs.\ Max correlation}
        \label{fig:maxc-hamming}
    \end{subfigure}
    \begin{subfigure}[b]{0.4\textwidth}
        \includegraphics[width=\textwidth]{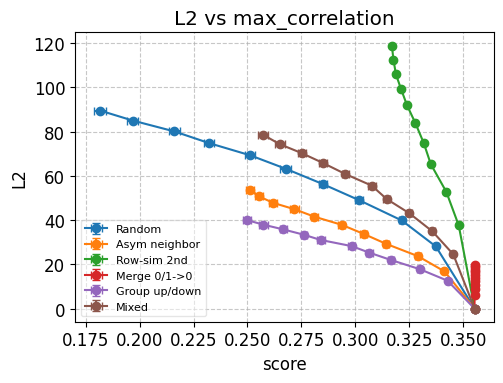}
        \caption{$\ell_2$ norm vs.\ Max correlation}
        \label{fig:maxc-l2}
    \end{subfigure}

    \begin{subfigure}[b]{0.4\textwidth}
        \includegraphics[width=\textwidth]{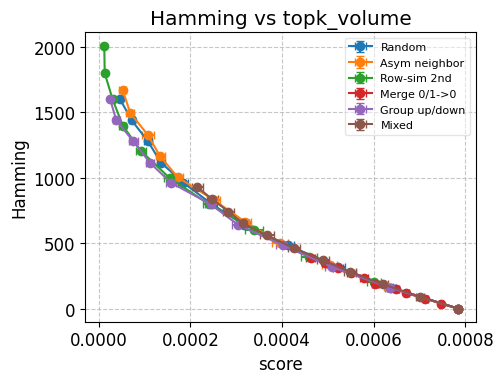}
        \caption{Hamming vs.\ Top-$k$ volume}
        \label{fig:topd-hamming}
    \end{subfigure}
    \begin{subfigure}[b]{0.4\textwidth}
        \includegraphics[width=\textwidth]{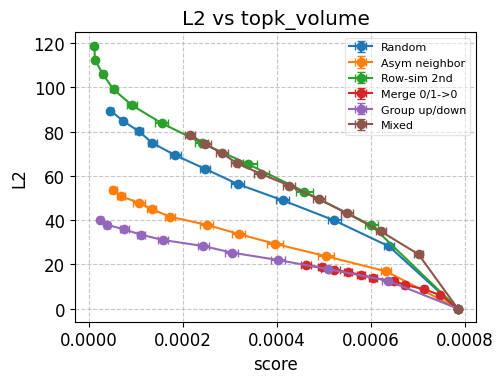}
        \caption{$\ell_2$ norm vs.\ Top-$k$ volume}
        \label{fig:topd-l2}
    \end{subfigure}

    \caption{Comparison of the KL-divergence, $\chi^2$-mutual information, maximal correlation and Top-$k$ volume scores under different corruption levels and metrics.}
    \label{fig:otherscore}
\end{figure}

Across manipulations, the scores are inversely related to the corruption level, as measured by the Hamming distance and the $\ell_2$ norm between $\vx$ and $\hat{\vx}$
(see \cref{fig:all_results2,fig:otherscore}).
This alignment across multiple metrics demonstrates that the proposed scores all provide plausible measures of data quality.  However, under the Merge 0/1$\to$ 0 manipulation, maximal correlation performs poorly, assigning almost identical scores to datasets with varying levels of manipulation. This is likely due to the fact that maximal correlation depends only on the most significant singular value and misses more fine-grained information. By contrast, the Gram determinant (product of all $d = 5$ singular values) and the top‑$k$ volume (product of the largest $k=4$) perform better, as shown in \cref{fig:hamming_vs_det2,fig:l2_vs_det2} and \cref{fig:topd-hamming,fig:topd-l2} respectively. Additionally, we observe cross‑manipulation inconsistencies in maximal correlation, $\chi^2$‑MI and KL‑divergence in~\cref{fig:kl-hamming,fig:kl-l2,fig:chi2-hamming,fig:chi2-l2}: when two datasets are manipulated by different methods, they may assign a higher score to the dataset
that is further away from the truth according to Hamming distance, violating the Hamming ordering. In contrast, our GDS (\cref{fig:hamming_vs_det2}) and the top-$k$ volume score (\cref{fig:topd-hamming}) consistently preserve the Hamming ordering across all six manipulations.

\section{Discussion}

One condition that arises from our characterization is the linear independence of the columns of $\mP$. If $\mP$ does not have full rank, no score can preserve the Blackwell ordering (\cref{prop:impossible}), whereas our Gram determinant score preserves reliability orderings under linearly independent $\mP$. This condition is inherent to the structure of the data reliability scoring problem. Is this condition restrictive?

We think this linear independence of $\mP$ is a mild condition. When the observation space $\mathcal{Y}$ is high-dimensional (e.g., composed of multiple measurements or signals with real values), collinearity among the columns of $\mP$ is highly unlikely—the set of linearly dependent matrices has Lebesgue measure zero. In practice, one can design or augment the experiment variable $\vy$ to increase informativeness. Since $\vy$ can be vector-valued, practitioners can combine multiple measurements into a composite $\vy$ to ensure the linear independence of $\mP$. %For instance, in the COVID-19 example, combining death counts, ambulance-call records, and hospital occupancy rates into a composite $\vy$ effectively restores rank.

Our proposed Gram determinant score has shown strong empirical performance in both synthetic and real-world settings. However, several caveats deserve attention:

\paragraph{Discretization versus kernelization.}
In our synthetic experiments with Gaussian label distributions, we found that both the kernelized Gram determinant score (using a Gaussian kernel) and the regular Gram determinant score (based on bucketization on ${\bf y}$) performed similarly well, with no significant difference in effectiveness. This suggests that discretization, despite being a relatively crude approach, can sometimes work better than or similar to more elaborate kernel methods. In practice, however, not all datasets admit a natural discretization strategy. For example, in image datasets such as CIFAR-10, the lack of an intuitive discretization makes kernelized versions of the Gram determinant score particularly valuable.

\paragraph{Assumptions about conditional independence in Experiment 3.}
A key limitation of Experiment~3 is the reliance on the conditional independence assumption, which is difficult to validate in real-world applications. In practice, employment data may be indirectly adjusted from tax withholding records. In the employment dataset, we lack ground-truth employment data and only have access to three fiscal time series from which scores are computed. This prevents us from directly checking whether conditional independence holds. Consequently, the reported scores for these employment series should be interpreted as indicative rather than definitive measures of reliability.

\paragraph{Comparison with alternative scores.}
We compared the Gram determinant score to four existing scoring methods in \cref{sec:exp4}. All of them showed broadly consistent behavior: their rankings aligned well with Hamming distance and $\ell_2$-norm error. We also attempted to demonstrate the advantage of the Gram determinant score as an ``experiment-agnostic'' method. However, because we only had access to samples $\hat{\vx}$ with corresponding $\vy$, the underlying joint distribution matrix $\mP\mQ$ was unknown, and any estimator we used introduced additional variance. Consequently, the Gram determinant score could not exhibit a clear advantage in this regard. This limitation makes it more difficult to establish the clear superiority of our approach over the alternatives discussed in \cref{app:alternative_score}, particularly in finite-sample regimes.

\paragraph{Application of the Gram Determinant Score in Practice}
Although verifying the formal conditions to preserve reliability orderings may be challenging in practice, several heuristics can offer guidance.  Strongly imbalanced reported labels---for example, when one class is reported far more frequently than others---may fail to provide information for rare labels to reliably distinguish their observations.  The conditional independence assumption is more credible when the observation is revealed only after reports (or kept blinded), so reporters cannot tailor reports to the observations. Persistently small determinants of the empirical Gram matrix may reflect poor reliability or weak stochastic dependence between the reported data and observations. These diagnostics are not formal tests, but they offer practitioners useful signals about whether the theoretical requirements are plausibly satisfied in applied settings.

\section{Conclusion}
We introduce the Gram determinant score --- a metric that intuitively measures the volume of class-conditional observation distributions. Under mild independence assumptions, it exactly preserves exact-match and Blackwell orderings and closely approximates Hamming orderings. We develop plug-in and stratified-matching estimators with finite-sample guarantees and extend the method to continuous or structured spaces via kernel embeddings. Experiments on synthetic data, CIFAR-10 embeddings, and employment data demonstrate its effectiveness.

Looking ahead, it is interesting to design scalable estimators for high-dimensional or continuous label domains using dimensionality-reduction (e.g., PCA, DPP sampling) and learned encoders.
Moreover, we conjecture that other singular-value–based criteria can also serve as reliability scores. \Cref{app:alternative_score} briefly discusses additional candidates beyond the Gram determinant score and reports synthetic-data experiments evaluating them.  However, formal guarantees remain to be established; each candidate will require tailored analysis to show it preserves the relevant reliability orderings.
In real-world settings, the Gram determinant score is applicable wherever labels are noisy or manipulated -- for example, by detecting incoherent star ratings in product reviews -- and could help platforms like Amazon and Yelp enhance consumer protection.
% \clearpage
% \section*{Ethics Statement}
% We have reviewed and adhere to the ICLR Code of Ethics. This paper is primarily theoretical and uses only public datasets (CIFAR-10 and publicly available employment series). No human subjects were involved and no personally identifiable or sensitive data were used. Our methods and findings do not present foreseeable harm or misuse. All data handling follows the licenses and terms of the original sources.

% \section*{Reproducibility Statement}
% We provide anonymized source code and scripts in the supplementary materials. This includes synthetic data generators, the six manipulation policies, training and evaluation pipelines, and the exact hyperparameters and random seed settings used for every figure and table. For convenience we also include the processed subsets of the public datasets required to run our code. All theoretical results have complete and rigorous proofs in the appendix. These materials enable end-to-end reproduction of all reported results.

\paragraph{Acknowledgments}

This work was partially supported by the National Science Foundation under Grant No. IIS-2147187. We thank anonymous reviewers for their feedback.

We acknowledge the use of generative-AI tools to edit the manuscript and, on occasion, to assist in drafting text and code. All such material, including introductory examples and code used to produce figures, was reviewed and revised by the authors.

\bibliography{reference}
\bibliographystyle{ACM-Reference-Format}
\newpage
\appendix
\section*{Appendix}
\section{Preliminary: Matrices and Kernels}\label{app:mat}
This section provides basic definitions and theorems for matrices and kernels.
Given a $d\times d$ matrix $\mA$, the determinant of $\mA$ is
$$\det(\mA) = \sum_{\sigma\in symm(d)} sgn(\sigma)\prod_{i = 1}^d A(i, \sigma(i)),$$
where $symm(m)$ is the set of all permutations of $[d]$ and $sgn(\sigma)$ is the sign function of a permutation.

Given two $d\times d$ matrices $\mA$ and $\mB$, the Frobenius inner product
between them is $\langle \mA, \mB\rangle_F:=\sum_{i,j\in [d]}\mA(i,j)\mB(i,j)$.

% \begin{proposition}
%     If $\mB^\intercal \mB$ is positive definite,
%     $$det((\mB\mA)^\intercal (\mB\mA)) = det(\mA^\intercal \mA)\det(\mB^\intercal \mB).$$
% \end{proposition}

% \begin{proof}
% Since $\mB^\intercal \mB$ is positive definite, it has a unique positive definite square root. Write
% $\mB=(\mB^\intercal\mB)^{1/2}\mC$ with $\mC^\intercal\mC=I$.
% Then, $(\mB\mA)^\intercal(\mB\mA)=(\mC\mA)^\intercal (\mB^\intercal\mB)(\mC\mA)$.

% Using the fact that for any matrix $X$ and any positive definite matrix $Y$, $\det(X^\intercal Y X)=\det(Y)\det(X^\intercal X),$
% with $X=\mC\mA$ and $Y=\mB^\intercal\mB$, we obtain
% \begin{align*}
% \det((\mB\mA)^\intercal(\mB\mA))
% =\det(\mB^\intercal\mB)\,\det((\mC\mA)^\intercal(\mC\mA)).
% \end{align*}
% Since $\mC^\intercal\mC=I$, we have $(\mC\mA)^\intercal(\mC\mA)=\mA^\intercal\mA$, proving that
% \begin{align*}
% \det((\mB\mA)^\intercal(\mB\mA))=\det(\mB^\intercal\mB)\,\det(\mA^\intercal\mA),
% \end{align*}
% which is what we want.
% \end{proof}
We introduce two approximation results for determinants.  The first one shows that $\det(\mA)$ can be approximated by the determinant of its diagonal matrix, and the second shows that the determinant is smooth under small perturbation.
\begin{theorem}[\citep{ipsen2011determinantapproximations}]\label{thm:appro_det}
    Let $\mA$ be a $d$-dimensional squared matrix, $\mA_D$ be the associated diagonal matrix, and $\mA_E = \mA-\mA_D$.  If $\mA_D$ is non-singular and spectral norm $\rho:=\|\mA_D^{-1} \mA_E\|_2<1$ then
    $$\frac{|\det(\mA)-\det(\mA_D)|}{|\det(\mA_D)|}\le c\rho e^{c\rho}, \text{ where } c = -d\ln(1-\rho)$$
    Moreover, if $c\rho<1$, $\frac{|\det(\mA)-\det(\mA_D)|}{|\det(\mA_D)|}\le \frac{7}{4}c\rho$.
\end{theorem}

\begin{theorem}[\citep{doi:10.1137/070704770}]\label{thm:appro_det2}
    Let $\mA$ and $\mE$ be $d\times d$ matrices. If A is nonsingular, then
    $$\frac{|\det(\mA+\mE)-\det(\mA)|}{|\det(\mA)|}\le\left(1+\kappa\frac{\|\mE\|_2}{\|\mA\|_2}\right)^d-1$$
    where $\kappa = \|\mA\|_2\|\mA^{-1}\|_2$ and $\|\cdot \|_2$ is the spectral norm.
\end{theorem}

\begin{lemma}\label{lem:non_perm}
    Given $\mA, \mB\in \R^{d\times d}$, if $\mB\neq \mathbb{I}$ is column stochastic and $\mA, \mB\mA$ are column diagonally maximal, $\mB$ is not a permutation matrix.
\end{lemma}
\begin{proof}[Proof of \cref{lem:non_perm}]
    Suppose not and there exists a permutation $\sigma:[d]\to [d]$ and $\iota\in [d]$ so that $B(i,j) = \mathbf{1}[j = \sigma(i)]$ and $\sigma(\iota)\neq \iota$.  Because $\mA$ is column diagonally maximal
    $$(\mB\mA)(\iota, \iota) = \sum_j \mB(\iota, j)\mA(j,\iota) = \mA(\sigma(\iota), \iota)< \mA(\iota,\iota).$$
    Additionally,
    $$(\mB\mA)(\sigma^{-1}(\iota), \iota) = \mB(\sigma^{-1}(\iota), \iota)\mA(\iota, \iota) = \mA(\iota, \iota)>(\mB\mA)(\iota, \iota).$$
    Therefore, $\mB\mA$ is not column diagonally maximal which is a contradiction.
\end{proof}

Now we introduce kernel.
\begin{definition}
A function $K:\mathcal{Y}\times\mathcal{Y}\to \R$ is \emph{positive definite kernel} if for all $\{y_1,\dots,y_m\}\subseteq \mathcal{Y}$, the matrix $[K(y_i,y_j)]_{ij}\in \R^{m\times m}$ is symmetric positive semi definite.  Additionally, it is \emph{strictly positive definite} if the matrix is positive definite.
\end{definition}

By Moore-Aronszajn theorem~\citep{aronszajn1950theory}, given a positive definite kernel $K$, there exists a Hilbert space $\mathcal{H}$ known as a reproducing kernel Hilbert space so that for any $y\in \mathcal{Y}$, $K(\cdot, y)\in \mathcal{H}$ and for all $h\in \mathcal{H}$, $h(y) = \langle h, K(y,\cdot)\rangle$.  This allows us to think of a kernel defines a feature map $\phi: y\mapsto K(\cdot, x)\in \mathcal{H}$ where the inner product in the embedded space reduces to kernel evaluation, because $\langle K(\cdot, y), K(\cdot, y')\rangle = K(y,y')$

% To any kernel $K: \mathcal{Y}\times\mathcal{Y}\to \R$, we can associate a \emph{normalized
% kernel} $K': \mathcal{Y}\times\mathcal{Y}\to \R$ defined for all $y, y'\in \mathcal{Y}$ as
% $$K'(y,y') = \begin{cases}
%     \frac{K(y,y')}{\sqrt{K(y,y)K(y',y')}}&\text{ if $K(y,y)K(y',y')\neq 0$}\\
%     0&\text{ otherwise}
% \end{cases}$$
% Note that $K'$ is also a kernel and $K'(y,y')\le 1$ for all $y,y'$.

Moreover, given a measurable kernel, we can define the \emph{kernel mean embedding}~\citep{berlinet2011reproducing} of probability measures on $\mathcal{Y}$, $P\in \Delta(\mathcal{Y})$, into $\mathcal{H}$ where
$$\phi(P):=\int K(\cdot, y)dP(y) = \E_{y\sim P}[\phi(y)].$$ Here we slightly abuse the notations, and note that $\phi$ is linear in $P$ by linearity of integration.  We can further extend this to signed measures $\phi(\mu) := \int K(\cdot,y)d\mu(y)$.  Finally, a kernel $K$ is \emph{integrally strictly positive definite} if the $\iint_\mathcal{Y}K(y,y')d\mu(y)d\mu(y')>0$ for all finite non-zero signed measures $\mu$.

\section{Proofs and Details in \Cref{sec:model}}\label{app:model}
We show that the reliability orderings are well-defined ordering.  Formally, a binary relationship $\succ$ on $\Omega$ is a \emph{strict partial order} if it satisfies the following conditions for all $a, b, c\in \Omega$
\begin{enumerate}
    \item  anti-reflexive: no element is larger than itself
    \item asymmetry: if $a\succ b$ then not $b\succ a$
    \item Transitivity: if $a\succ b$ and $b \succ c$, then $a \succ c$.
\end{enumerate}

% Exact match ordering, Blackwell dominant ordering, Hamming ordering, and $\dist$ ordering are all strict partial orders for all $\vx$ which satisfies the following conditions for all $\vx$ and $\hat{\vx}_1, \hat{\vx}_2, \hat{\vx}_3$
% \begin{enumerate}
%     \item  anti-reflexive: no element is larger than itself
%     \item asymmetry: if $\hat{\vx}_1\succ^\vx \hat{\vx}_2$ then not $\hat{\vx}_2\succ^\vx \hat{\vx}_1$
%     \item Transitivity: if $\hat{\vx}_1\succ^\vx \hat{\vx}_2$ and $\hat{\vx}_2\succ^\vx \hat{\vx}_3$, then $\hat{\vx}_1\succ^\vx \hat{\vx}_3$.
% \end{enumerate}

Next, we show that the reliability orderings defined in \cref{sec:model} form a strict partial order over reports, given a fixed true data.
\begin{proposition}\label{prop:exact_ordering}
For any $\vx\in \mathcal{X}^N$, the exact match ordering $\succ^\vx_{\exact}$ is a strict partial order on all $\hat{\vx}$ and $\hat{\vx}'\in \mathcal{X}^N$.
\end{proposition}
\begin{proof}
The first two are trivial.  For transitivity, if $\hat{\vx}_1\succ_{\exact}^\vx \hat{\vx}_2$, then $\hat{\vx}_2\neq \vx$ so there is no $\hat{\vx}_3$ with $\hat{\vx}_2\succ_{\exact}^\vx \hat{\vx}_3$.
\end{proof}

The following shows that Blackwell dominant ordering is a strict partial order over subsets of reports under the invertible and diagonally maximal conditions.  Those conditions are essential. If the misreport matrices are not invertible, the Blackwell ordering may fail to be asymmetric: it is possible for two distinct reports to Blackwell-dominate each other, violating the strictness of the relation. Similarly, if the misreport matrices are not diagonally maximal, the ordering also fails asymmetry via non-trivial permutation.

\begin{proposition}\label{prop:blackwell_ordering}
For any $\vx\in \mathcal{X}^N$, Blackwell dominant ordering $\succ^\vx_{\blackwell}$ is a strict partial order on all $\hat{\vx}$ and $\hat{\vx}'\in \mathcal{X}^N$ so that the associated misreport matrices $\mQ, \mQ'\in \mathcal{Q}_{\text{reg}}$ are invertible and diagonally maximal.
\end{proposition}
\begin{proof}
Suppose $\succ_{\blackwell}^\vx$ is not anti-reflective.  There exists $\hat{\vx}\succ^\vx_{\blackwell}\hat{\vx}$ with misreport matrix $\mQ$ and a column stochastic matrix $\mT\neq \mathbb{I}$ so that $$\mT \mQ_{\hat{\vx}\mid \vx} =  \mQ_{\hat{\vx}\mid \vx}.$$  Because $\mQ = (\mQ_{\hat{\vx}|\vx}\mQ_\vx)^\intercal$ is invertible, $\mQ_{\hat{\vx}\mid \vx}$ is also invertible and $\mT = \mathbb{I}$ which is a contradiction.

For asymmetry, if $\hat{\vx}\succ^\vx_{\blackwell}\hat{\vx}'$ and  $\hat{\vx}'\succ^\vx_{\blackwell}\hat{\vx}$, there exist column stochastic matrices $\mT$ and $\mT'$ so that
$$\mT \mQ_{\hat{\vx}\mid \vx} =  \mQ_{\hat{\vx}\mid \vx}'\text{ and }\mT' \mQ_{\hat{\vx}\mid \vx}' =  \mQ_{\hat{\vx}\mid \vx}.$$
Because $\mQ, \mQ'$ are invertible, $\mT\mT' = \mathbb{I}$, and both $\mT$ and $\mT'$ are permutation matrices.~\citet{2392982}  However, because $\mQ$ and $\mQ'$ are (row) diagonally maximal, $\mQ_{\hat{\vx}\mid \vx}$ and $\mQ_{\hat{\vx}\mid \vx}'$ are column diagonally maximal.  Therefore by \cref{lem:non_perm}, $\mT = \mT' = \mathbb{I}$  which is a contradiction.

Transitivity is trivial, because the product of column stochastic matrices is still stochastic.
\end{proof}
\begin{proposition}\label{prop:hamming_ordering}
For any $\vx\in \mathcal{X}^N$, $\dist$ ordering $\succ^\vx_{\dist}$ is a strict partial order on all $\hat{\vx}$ and $\hat{\vx}'\in \mathcal{X}^N$
\end{proposition}
\begin{proof}
The first two are trivial. For transitivity, given $\vx, \vx'$ let $\dist(\vx, \vx') := \sum_n \dist(x_n, x_n')$.  If $\hat{\vx}_1\succ^\vx_{\dist} \hat{\vx}_2$ and $\hat{\vx}_2\succ^\vx_{\dist} \hat{\vx}_3$ then $\dist(\vx, \hat{\vx}_1)<\dist(\vx, \hat{\vx}_2)$ and $\dist(\vx, \hat{\vx}_2)<\dist(\vx, \hat{\vx}_3)$.  Therefore, $\hat{\vx}_1\succ^\vx_{\dist} \hat{\vx}_3$.
\end{proof}

% \begin{proof}
%     For all $\vx$ and $\hat{\vx}$ if $\dist$ is discrete metric, the average distance between them is
% \begin{align*}
%     \frac{1}{N}\sum_{i = n}^N \dist(\hat{x}_n, x_n) =& \frac{1}{N}\sum_{i = n}^N \mathbf{1}[\hat{x}_n\neq x_n]\\
%     =& \frac{1}{N}\sum_{x,x'\in \mathcal{X}}\sum_{n:x_n = x, \hat{x}_n = x'}  \mathbf{1}[x\neq x']\\
%     =& \sum_{x,x'\in \mathcal{X}}Q(x,x')\mathbf{1}[x\neq x']\\
%     =& \sum_{x,x'\in \mathcal{X}}Q(x,x')-\sum_{x\in \mathcal{X}}Q(x,x)\\
%     =& 1-\Tr(\mQ)\tag{$\mQ$ is a joint distribution on $\mathcal{X}^2$}
% \end{align*}
% With above derivation, $\hat{\vx}\succ_{d}^\vx \hat{\vx}'$ under discrete metric, if and only if $1-\Tr(\mQ)<1-\Tr(\mQ')$ and thus $\hat{\vx}\succ_{\hamming}^\vx \hat{\vx}'$.
% \end{proof}

\subsection{Proof of \cref{prop:comparison}}

\begin{proof}[Proof of \cref{prop:comparison}]
Given $\vx, \hat{\vx}$, and $\hat{\vx}'$, if $\hat{\vx}\succ_{\exact}^\vx \hat{\vx}'$, $ \mQ_{\hat{\vx}\mid \vx} = \mathbb{I}$ and $ \mQ_{\hat{\vx}\mid \vx}'\neq \mathbb{I}$.  If we set a column stochastic $\mT =  \mQ_{\hat{\vx}\mid \vx}'$, $ \mQ_{\hat{\vx}\mid \vx}' = \mT  \mQ_{\hat{\vx}\mid \vx}$.  Therefore, $\hat{\vx}\succ_{\blackwell}^\vx \hat{\vx}'$.

If $\hat{\vx}\succ_{\blackwell}^\vx \hat{\vx}'$, there is $\mT\neq \mathbb{I}$ so that $\mQ_{\hat{\vx}\mid \vx}' = \mT \mQ_{\hat{\vx}\mid \vx}$.  With~\cref{eq:m2j} we have $\mQ' = (\mQ_{\hat{\vx}|\vx}'\mQ_\vx)^\intercal = (\mT\mQ_{\hat{\vx}|\vx}\mQ_\vx)^\intercal = \mQ\mT^\intercal$, and
\begin{align*}
   \Tr(\mQ') =& \Tr(\mQ\mT^\intercal)=\sum_{i,j} \mQ(i,j)\mT(i,j)\\
    =& \sum_i \mQ(i,i)\mT(i,i)+ \sum_{i, j: i\neq j} \mQ(i,j)\mT(i,j)\\
    \le& \sum_i \mQ(i,i)\mT(i,i)+ \sum_{i, j: i\neq j} \mQ(i,i)\mT(i,j)\tag{$\mQ$ is diagonally maximal and $\mT\neq \mathbb{I}$}\\
    =& \sum_i \mQ(i,i) = \Tr(\mQ)\tag{$\mT$ is column stochastic}
\end{align*}
Therefore, $\hat{\vx}\succ^\vx_{\hamming}\hat{\vx}'$.
The third on is straightforward by definition of refinement.
\end{proof}

\section{Proofs and Details in \cref{sec:impossible}}\label{app:impossible}
% \paragraph{Connection between detail-free and partial-knowledge setting}
We discuss the connection between detail-free setting and partial knowledge setting.  First note that as the order of data is not relevant, given $\hat{\vx}, \vy$ of size $N$, it is sufficient to consider the histogram of $\mR\in \R^{|\mathcal{Y}|\times|\mathcal{X}|}$ and $N$ where
$$\mR(y,x) = \frac{1}{N}\sum_{n} \mathbf{1}[y_n = y, \hat{x}_n = x].$$  By symmetrization, we can write a reliability score in detail-free setting as a stochastic function on the histogram $\mR$ and $N$ that have the same expected score.~\citep{DBLP:journals/corr/abs-2106-03176}  The expectation of $\mR$ over the randomness of experiment is $\E[\mR] = \mP\mQ$.  This leads to two key implications.  First when the data size $N$ is large, $\mR$ converges to $\mP\mQ$ so that the expectation of any smooth reliability score
$$\E[S(\mR)]\to S(\E[\mR]) =  S(\mP\mQ).$$
Second, if we consider any \emph{empirical risk-based scores} so that has $\ell:\mathcal{X}\times\mathcal{Y}\to \R$ so that
$$S(\hat{\vx}, \vy) = \frac{1}{N}\sum_n \ell(\hat{x}_n, y_n).$$
This includes common metrics like empirical risk and log-likelihood function.  We can rewrite it as a linear function of $\mR$
\begin{align*}
    S(\hat{\vx}, \vy) =& \frac{1}{N}\sum_n \ell(\hat{x}_n, y_n)\\
    =& \frac{1}{N}\sum_{x,y}\sum_n \mathbf{1}[\hat{x}_n = x, y_n = y]\ell(x,y)\\
    =& \sum_{x,y}\mR(y,x)\ell(x, y)
\end{align*}
which is simply the Frobenius inner product between $\mR$ and the score matrix based on $\ell$.

Finally, as \cref{def:gramdet}, our Gram determinant score is also a function of $\mP\mQ$.  Consequently, the impossibility results presented in \cref{sec:impossible} for the partial knowledge setting apply not only to the Gram determinant score but also to any empirical risk-based score.

We provide the proof of~\cref{prop:impossible} consists of three parts: exact, Blackwell, and Hamming and other dist orderings.

\paragraph{Proof of Exact orderings in \cref{prop:impossible}}
For the exact ordering setting, we motivate the independence condition on experiments and non-permutation condition on misreport matrix.  First we show that we need additional condition on experiments $\mathcal{P}$, even restricting to $\mathcal{Q}_{\text{nonperm}}$.  Second, we show that $\mathcal{Q}_{\text{nonperm}}$ is the maximal set of misreport matrices to have a reliability score that respects the exact match ordering on $\mathcal{P}_{\text{indep}}$.

Both parts use the idea that if two labels in $\mathcal{X}$ induce the same distribution over observations, it becomes impossible to determine whether the reports agree with the true data.

For the first part, if $\mP$ consists of identical columns, we can find a diagonal matrix $\mQ_\vx$ and a doubly stochastic $\mQ_{\hat{\vx}|\vx}\neq \mathbb{I}$ so that $\mP(\mQ_{\hat{\vx}|\vx}\mQ_\vx)^\intercal = \mP\mQ_{\vx}\mQ_{\hat{\vx}|\vx}^\intercal = \mP\mQ_{\vx}$.  Hence, we can set $\vx, \hat{\vx}$ with misreport matrices $\mQ = (\mQ_{\hat{\vx}|\vx}\mQ_\vx)^\intercal$ so that ${\vx}\succ^\vx_{\exact}\hat{\vx}$, but have the same joint distribution between reports and observations.  Therefore, no score in the partial knowledge setting can distinguish them and preserves exact match ordering.

For the second part, because we can only observe the observations and reports, it would be impossible to always score true data over relabeled reports (permutation).
Suppose not and there exists a score $S$ in partial knowledge setting that preserves all misreport matrices.
Given any $\mP\in \mathcal{P}_{\text{indep}}$, the uniform marginal distribution $\mQ_{\vx} := \frac{1}{d}\mathbb{I}$, and permutation $\mT\neq \mathbb{I}$, there exist $\vx$ and $\hat{\vx}$ so that the misreport matrix equals $\mT\mQ_{\vx} = \frac{1}{d}\mT$ and $\vx\succ^\vx_{\exact}\hat{\vx}$.  Because the joint distribution between reports and observations is $\frac{1}{d}\mP$ for $(\vx, \vy)$, and $\frac{1}{d}\mP\mT^\intercal$ for $(\hat{\vx},\vy)$, we have
$$S\left(\frac{1}{d}\mP\right)>S\left(\frac{1}{d}\mP\mT^\intercal\right).$$

Conversely, we can set an new experiment $\mP' = \mP\mT^\intercal$ and $\vx' = \hat{\vx}$ and $\hat{\vx}' = \vx$ so that the misreport matrix equals $\frac{1}{d}\mT^\intercal$ and $\vx'\succ^{\vx'}_{\exact}\hat{\vx}'$.  First, because $\mT$ is a permutation $\mP' = \mP\mT^\intercal\in \mathcal{P}_{\text{indep}}$ and the joint distributions becomes $\frac{1}{d}\mP' = \frac{1}{d}\mP\mT^\intercal$ for $(\vx', \vy')$ and $\frac{1}{d}\mP'\mT = \frac{1}{d}\mP\mT^\intercal\mT = \frac{1}{d}\mP$ for $(\hat{\vx}', \vy')$.  Therefore,
$$S\left(\frac{1}{d}\mP\mT^\intercal\right)> S\left(\frac{1}{d}\mP\right)$$
which is a contradiction.

\paragraph{Proof of Blackwell orderings in \cref{prop:impossible}}
For Blackwell ordering, we further show that the existence of \emph{any} linearly dependent experiment $\mP$ (i.e. columns of $\mP$ are linearly dependent) in $\mathcal{P}$ makes it impossible to preserve Blackwell ordering on $\mathcal{P}$ and $\mathcal{Q}_{\text{reg}}$. The $\mathcal{Q}_{\text{reg}}$ restriction rules out the possibility of improving data reliability through simple post-processing operations like (noisy) relabeling. In addition, recall that Blackwell ordering requires $\mathcal{Q}_{\text{reg}}$ to be a strict partial ordering.

The proof idea is similar to that of the exact ordering setting: it is impossible to detect misreporting when two labels induce identical observation distributions—i.e., when $\mP$ has identical columns. The main challenge in \cref{prop:impossible}, however, is to show that for any linearly dependent $\mP$ (which may not have identical columns), we can construct a misreport matrix $\mQ$ such that $\mP\mQ$ has identical columns.\footnote{
We require $\vv\in \mathbb{Q}^d$ to have rational coefficients to ensure the resulting $\mQ$ has rational coefficients to be a valid misreport matrix.}

    If we can find $\mP \in \mathcal{P}$, a misreport matrix $\mQ$, and column stochastic $\mT\neq \mathbb{I}$ with $\mP\mQ = \mP\mQ\mT^\intercal$, we have $\vx, \hat{\vx}, \hat{\vx}'$ with misreport matrices $\mQ$ and $\mQ \mT^\intercal$ so that $\hat{\vx}\succ^\vx_{\blackwell}\hat{\vx}'$, but have the same joint distribution between reports and observations.  Therefore, no score in the partial knowledge (and detail-free) setting can distinguish them and preserves the Blackwell ordering.

    Now we construct $\mP, \mQ$, and $\mT$.  By the condition in \cref{prop:impossible} there exists $\mP\in \mathcal{P}$ and $\vv\neq \vzero\in \mathbb{Q}^d$ so that $\mP\vv = \vzero$.  We decompose $\vv$ as $\vv = \vv_+-\vv_-$ where $\vv_+$ and $\vv_-$ are nonnegative and have disjoint support, so
    \begin{equation}\label{eq:impossible_indep1}
        \mP \vv_+ = \mP \vv_-
    \end{equation} and $\vv_+, \vv_-\neq \vzero$ because $\mP$ is a collection of distributions.   Let $\iota_+\in[d]$ be the index of the largest entry in $\vv_+$, and $\iota_-$ for $\vv_-$ similarly, breaking ties arbitrarily.   Note that $\iota_+\neq \iota_-$, because $\vv_+$ and $\vv_-$ have disjoint supports.  We first construct $\mA$ by replacing the $\iota_+$ column of the identity matrix $\mathbb{I}\in \R^d$ by $\vv_+$ and $\iota_-$ column by $\vv_-$, and set $\mQ = \frac{1}{Z}\mA$ where $Z = \sum_{i,j} \mA(i,j)$. This normalization ensures that $\mQ$ forms a distribution as $\vv_+$ and $\vv_-$ are non-negative.  By construction, $\mQ$ is diagonally maximized by the choice of $\iota_+, \iota_-$, and invertible because $\vv_+, \vv_-\neq \vzero$ and using Gaussian elimination.  Most importantly, the $\iota_+$ and $\iota_-$ columns of $\mP\mQ$ are identical by \cref{eq:impossible_indep1}.

    To complete the construction, given $\epsilon>0$ we set $\mT\neq \mathbb{I}$ so that
    $$\mT(i,j) = \begin{cases}
        1&\text{ if }i = j\text{ and } \{i, j\}\cap \{\iota_+, \iota_-\} = \emptyset\\
        0&\text{ if }i \neq j\text{ and } \{i, j\}\cap \{\iota_+, \iota_-\} = \emptyset\\
        \epsilon&\text{ if }i = \iota_+, j = \iota_- \text{ or }i = \iota_-, j = \iota_+\\
        1-\epsilon&\text{ if } i = j\in \{\iota_+, \iota_-\}\\
        0&\text{ if }i\neq j\text{ and }|\{i, j\}\cap \{\iota_+, \iota_-\}| = 1
    \end{cases}$$
    which is the identical matrix excepts for the $\iota_+$ and $\iota_-$ columns and rows.
     Note that $\mT$ is a column stochastic matrix, $\mQ\mT^\intercal$ is still invertible and diagonally maximal when $\epsilon$ is small enough.  Finally, $\mP\mQ\mT^\intercal$ mixes the $\iota_+$ and $\iota_-$ columns.  However, because the $\iota_+$ and $\iota_-$ columns of $\mP\mQ$ are identical, $\mP\mQ = \mP\mQ\mT^\intercal$ which completes our proof.

\paragraph{Proof of Hamming and $\dist$ orderings in \cref{prop:impossible}}

Finally, we show that there does not exist a reliability score that preserves the Hamming and $\dist$ distance ordering, even restricting to diagonally dominant misreport matrices $\mathcal{Q}_{\text{dom}}\subset\mathcal{Q}_{\text{reg}}$.

We begin the proof with the Hamming ordering.
Suppose we can find two settings: one has $\mQ_1, \mQ_1'\in \mathcal{Q}_{\text{dom}}$ and $\mP_1\in \mathcal{P}_{\text{indep}}$, the other has $\mQ_2, \mQ_2'\in \mathcal{Q}_{\text{dom}}$ and $\mP_2\in \mathcal{P}_{\text{indep}}$ so that
$$\Tr(\mQ_1)>\Tr(\mQ_1'), \Tr(\mQ_2)<\Tr(\mQ_2')\text{, but }\mP_1\mQ_1 = \mP_2\mQ_2, \mP_1\mQ_1' = \mP_2\mQ_2'.$$  Then we can find $\vx_1, \hat{\vx}_1, \hat{\vx}_1'$, $\vx_2, \hat{\vx}_2, \hat{\vx}_2'$ so that $\hat{\vx}_1\succ_{\hamming}^{\vx_1} \hat{\vx}_1'$ and $\hat{\vx}_2'\succ_{\hamming}^{\vx_2} \hat{\vx}_2$ by setting the misreport matrix of $\vx_1, \hat{\vx}_1$ be $\mQ_1$, the misreport matrix $\vx_1, \hat{\vx}_1'$ as $\mQ_1'$, the misreport matrix of $\vx_2, \hat{\vx}_2$ be $\mQ_2$, the misreport matrix $\vx_2, \hat{\vx}_2'$ as $\mQ_2'$.  If there is a reliability score that preserves the Hamming ordering on $\mathcal{P}_{\text{indep}}, \mathcal{Q}_{\text{dom}}$,
\begin{equation}\label{eq:ham1}
    \E[S(\mP_1\mQ_1)]>\E[S(\mP_1\mQ_1')]\text{ and }\E[S(\mP_2\mQ_2)]<\E[S(\mP_2\mQ_2')]
\end{equation}
which reaches a contradiction as $\mP_1\mQ_1 = \mP_2\mQ_2$ and $\mP_1\mQ_1' = \mP_2\mQ_2'$

We construct
$$\mP_1=\begin{pmatrix}
0.74 & 0 & 0.26\\[1mm]
0.26    & 0.74 & 0\\[1mm]
0 & 0.26    & 0.74
\end{pmatrix},
\mQ_1=\frac{1}{3}\begin{pmatrix}
0.8 & 0 & 0.2\\[1mm]
0.2   & 0.8 & 0\\[1mm]
0 & 0.2   & 0.8
\end{pmatrix},\quad
\mQ_1'=\frac{1}{3}\begin{pmatrix}
0.7 & 0.3   & 0\\[1mm]
0 & 0.7 & 0.3\\[1mm]
0.3   & 0 & 0.7
\end{pmatrix}.
$$
For the second setting, we define $\mP_2=\mathbb{I}$, and
$$
\begin{aligned}
\mQ_2 &= \mP_1\mQ_1 = \frac{1}{3}\begin{pmatrix}
0.592 & 0.052 & 0.356\\
0.356 & 0.592 & 0.052\\
0.052 & 0.356 & 0.592
\end{pmatrix} \\
\mQ_2' &= \mP_1\mQ_1' = \frac{1}{3}\begin{pmatrix}
0.596 & 0.222 & 0.182\\
0.182 & 0.596 & 0.222\\
0.222 & 0.182 & 0.596
\end{pmatrix}
\end{aligned}$$
Therefore, $\mP_1\mQ_1 = \mP_2\mQ_2$ and $\mP_1\mQ_1' = \mP_2\mQ_2'$.
By direct computation, we have
$\Tr(\mQ_1)=\frac{24}{30}>\Tr(\mQ_1')=\frac{21}{30}$ and $\Tr(\mQ_2) = \frac{1776}{3000}<\Tr(\mQ_2') = \frac{1788}{3000}$.  Finally, note that we can easily generalize this construction beyond three dimensions by padding the other dimension with identity.

Interestingly, the same construction works for general $\dist$-ordering, due to the symmetry in $\mQ_1, \mQ_1', \mQ_2$ and $\mQ_2'$.  First note that $\sum_{n = 1}^N\dist(\hat{x}_n, x_n) = N\sum_{i,j\in [d]} \mQ(i,j)\dist(i,j) = N \langle \mQ, \dist\rangle_F$ where $\langle \cdot, \cdot\rangle_F$ is the Frobenius inner product defined in~\cref{app:mat}. Hence, with \cref{eq:ham1}, it is sufficient to show the above construction satisfies
$$\langle \mQ_1, \dist\rangle_F>\langle \mQ_1', \dist\rangle_F\text{ and }\langle \mQ_2, \dist\rangle_F<\langle \mQ_2', \dist\rangle_F.$$
Let $A = \dist(1,2)+\dist(2,3)+\dist(3,1) = \dist(1,3)+\dist(2,1)+\dist(3,2)>0$ as $\dist(x,x') = \dist(x',x)$ for all $x,x'$.  By symmetry, we note the Frobenius inner product only depends on $A$,
\begin{align*}
    \langle \mQ_1, \dist\rangle_F-\langle \mQ_1', \dist\rangle_F =& \frac{1}{3} (0.2A-0.3A)<0\tag{$\dist(x,x) = 0$ for all $x$}\\
    \langle \mQ_2, \dist\rangle_F-\langle \mQ_2', \dist\rangle_F =& \frac{1}{3} (0.408A-0.404A)>0
\end{align*}
which completes the proof.
\section{Proofs and Details in \cref{sec:preserve}}

\begin{proof}[Proof of \cref{eq:gram2gram}]
    For all $\hat{x}, \hat{x}'\in \mathcal{X}$,
\begin{align*}
    \hat{\mG}(\hat{x},\hat{x}') =& \frac{1}{N^2}\sum_{n, n': \hat{x}_n = x, \hat{x}_{n'} = x'}\langle P_{x_n}, P_{x_{n'}}\rangle\\
    =& \frac{1}{N^2}\sum_{x,x'\in \mathcal{X}}\sum_{\substack{n,n':\\
    \hat{x}_n = \hat{x}, \hat{x}_{n'} = \hat{x}',\\
                  x_n = x, x_{n'} = x'}}\mG(x,x')\\
    =& \sum_{x,x'\in \mathcal{X}}\mQ(x,\hat{x})\mG(x,x')\mQ(x',\hat{x}')
\end{align*}
which proves \cref{eq:gram2gram}.
\end{proof}

\subsection{An Example of Gram Determinant Score}
\label{sec:delta_kernel}
Here we provide a simple example for Gram determinant score.
% Let $\mathcal{X} = \mathcal{Y} = [d]$. A delta kernel satisfies that for all $y, y'\in \mathcal{Y}$
% \begin{equation}\label{ex:one-hot}
%     K(y,y') = \mathbf{1}[y = y']
% \end{equation}   $\langle P_x, P_{x'}\rangle = \sum_{y\in \mathcal{Y}}P(y,x)P(y,{x'})$ is simply the inner product of the conditional distributions.  As a result, $\mG = \mP^\intercal\mP$, and
% $$\hat{\mG} = (\mP\mQ)^\intercal (\mP\mQ),$$
% by \cref{eq:gram2gram}.  Therefore, $det(\hat{\mG})$ is the Gram determinant of $\mP\mQ$.
Consider $\mathcal{X} = \mathcal{Y} = \{1,2\}$, $\mP = \begin{pmatrix}
    1-p_1& 1-p_2\\
    p_1 & p_2
\end{pmatrix}$, and the misreport matrix $\mQ = \begin{pmatrix}
    \frac{1-\delta}{4}& \frac{\delta}{4}\\
    \frac{\delta}{4}& \frac{1-\delta}{4}
\end{pmatrix}$ with $\delta\ge0$ where $\vx = \hat{\vx}$ if $\delta = 0$ whereas increasing
$\delta$ makes the reports less reliable.  By \cref{eq:gram2gram} and direct computation, the Gram determinant score is
\begin{equation}\label{eq:one-hot}
\begin{aligned}
    \det(\hat{\mG})
    &= \det(\mQ^\intercal)\det(\mG)\det(\mQ) \\
    &= \det(\mP)^2 \det(\mQ)^2 \\
    &= \frac{1}{2^8}(p_1-p_2)^2(1-2\delta)^2.
\end{aligned}
\end{equation}
Given a fixed experiment $\mP$, the Gram determinant score~\cref{eq:one-hot} decreases as $\delta$ increases from $\delta = 0$ to $1/2$.  In particular, it maximizes at $\delta = 0$, when the reported data exactly match the true data, and drops to zero at $\delta = 1/2$, where all reports contain the same uniform mixture of the true labels.  Additionally, the score also depends on the quality of the experiment $\mP$. If $p_1 = p_2$, columns of $\mP$ are linearly dependent and Gram determinant score become zero. In contrast, if $p_1\neq p_2$ and $\delta<1/2$, the score is strictly positive.

\subsection{Lemmas and Proofs for \cref{thm:gram_preserve}}\label{app:gram_preserve}

\begin{proof}[Proof of \cref{thm:gram_preserve}]
The key idea of proving \cref{thm:gram_preserve} is that the determinant has multiplicative property \cref{eq:gram2gram} which allows us to decouple the misreport matrix $\mQ$ from the quality of the experiment $\mP$, and $\det(\mG) = \det(\mP^\intercal\mP)>0$, for all $\mP\in \mathcal{P}_{\text{indep}}$.  Therefore,
$$\Gamma>\Gamma'\text{ if and only if }\det(\mQ^\intercal\mQ)>\det((\mQ')^\intercal\mQ').$$
Thus, the following \cref{lem:preserve_exact,lem:preserve_blackwell} prove the first and second cases.  For the third case, \cref{lem:preserve_hamming} proves the score preserves the approximate Hamming ordering, as $\Delta = 1$ for Hamming distance.  For general distance, let $\hamming(\vx, \hat{\vx}) = \sum_n \mathbf{1}[\hat{x}_n\neq  x_n]$ and $\dist(\vx, \hat{\vx}) = \sum_n \dist(\hat{x}_n, x_n)$ be the Hamming distance and $\dist$ between $\hat{\vx}$ and $\vx$.  Because
$$\min_{x\neq x'}\dist(x,x')\hamming(\hat{\vx}, \vx)\le \dist(\vx, \hat{\vx})\le \max_{x,x'}\dist(x,x')\hamming(\hat{\vx}, \vx),$$
and $\Delta = \frac{\max_{x,x'}\dist(x,x')}{\min_{x\neq x'}\dist(x,x')}$, $\hat{\vx}\succ^\vx_{\dist, 1/(4\Delta L)}\hat{\vx}'$ implies $\hat{\vx}\succ^\vx_{\hamming, 1/(4 L)}\hat{\vx}'$, which completes the proof.
\end{proof}

\begin{lemma}\label{lem:preserve_exact}
    For all $\vx,\hat{\vx}, \hat{\vx}'$ if $\hat{\vx}\succ^\vx_{\exact}\hat{\vx}'$ and  $\mQ, \mQ'\in \mathcal{Q}_{\text{nonperm}}$,
    $\det(\mQ^\intercal\mQ)>\det((\mQ')^\intercal\mQ').$

    % and Blackwell ordering on $\mP$ and any $\vx,\hat{\vx}, \hat{\vx}'$ with $\mQ, \mQ'\in \mathcal{Q}_{\text{reg}}$.

    % Given $\mP$, if $\mG$ is positive definite, the Gram determinant score preserves exact ordering on $\mP$ and
\end{lemma}

\begin{proof}[Proof of \cref{lem:preserve_exact}] As $\vx, \hat{\vx}$, and $\hat{\vx}'$ with $\hat{\vx}\succ_{\exact}^\vx \hat{\vx}'$, $\mQ_{\hat{\vx}|\vx} = \mathbb{I}$ and there is $\mT\neq \mathbb{I}$ so that $ \mQ_{\hat{\vx}\mid \vx}' = \mT \mQ_{\hat{\vx}\mid \vx} = \mT$.  By \cref{eq:m2j} we have $\mQ' = \mQ\mT^\intercal = \mQ_\vx\mT^\intercal$ and $\mQ = \mQ_\vx$.  Therefore
\begin{equation}\label{eq:preserve_exact1}
   \det((\mQ')^\intercal\mQ') = \det(\mT\mQ^\intercal \mQ \mT^\intercal) = \det(\mT\mT^\intercal)\det(\mQ^\intercal \mQ)
\end{equation}
Because the diagonal matrix $\mQ = \mQ_\vx$ has positive diagonals, and $\mT$ is column stochastic and not a permutation matrix, the Perron–Frobenius theorem (or~\citep{kong2020dominantly}) implies $|\det(\mT)|<1$ and $
   \det((\mQ')^\intercal\mQ') = \det(\mT\mT^\intercal)\det(\mQ^\intercal \mQ)<\det(\mQ^\intercal \mQ).$
\end{proof}

\begin{lemma}\label{lem:preserve_blackwell}
    For all $\vx,\hat{\vx}, \hat{\vx}'$ if $\hat{\vx}\succ^\vx_{\blackwell}\hat{\vx}'$ and  $\mQ, \mQ'\in \mathcal{Q}_{\text{reg}}$,
    $\det(\mQ^\intercal\mQ)>\det((\mQ')^\intercal\mQ').$
\end{lemma}
\begin{proof}[Proof of \cref{lem:preserve_blackwell}] As  $\hat{\vx}\succ_{\blackwell}^\vx \hat{\vx}'$, there is a column stochastic $\mT\neq \mathbb{I}$ so that $ \mQ_{\hat{\vx}\mid \vx}' = \mT \mQ_{\hat{\vx}\mid \vx}$.  By~\cref{eq:preserve_exact1},
$$ \det((\mQ')^\intercal\mQ') = \det(\mT\mQ^\intercal \mQ \mT^\intercal) = \det(\mT\mT^\intercal)\det(\mQ^\intercal \mQ)$$

Because $\mQ\in \mathcal{Q}_{\text{reg}}$ is invertible, $\det(\mQ)\neq 0$. By \cref{lem:non_perm}, $\mT$ is not a permutation matrix, so $|\det(\mT)|<1$, and $
   \det((\mQ')^\intercal\mQ') = \det(\mT\mT^\intercal)\det(\mQ^\intercal \mQ)<\det(\mQ^\intercal \mQ).$

%    $\Gamma'<\Gamma.$  Therefore, Gram determinant score preserves Blackwell ordering.

% Similarly, if $\hat{\vx}\succ_{\exact}^\vx \hat{\vx}'$, $\hat{\vx} = \vx$ and there is $\mT\neq \mathbb{I}$ and not a permutation so that $\mQ = \mQ_{\vx}$, $\mQ' = \mQ\mT^\intercal$.  By \cref{eq:m2j,eq:gram2gram} we have
% \begin{align*}
%     \Gamma' =&  \det(\mT\mT^\intercal)\Gamma<\Gamma.
% \end{align*}
% because $\mT$ is column stochastic matrix and not a permutation, $\det(\mT)^2<1$ by Perron–Frobenius theorem.  Therefore, the Gram determinant score also preserves the exact ordering.
\end{proof}

\begin{lemma}\label{lem:preserve_hamming}
    Given $\mathcal{X} = [d]$ and $L\ge 1$, for all $\vx,\hat{\vx}, \hat{\vx}'$ if $\hat{\vx}\succ^\vx_{\hamming,\frac{1}{4L}}\hat{\vx}'$ and $\mQ, \mQ'\in \mathcal{Q}_{L,1/(64L^2d^2)}$,
    $\det(\mQ^\intercal\mQ)>\det((\mQ')^\intercal\mQ').$

    % $\mG$ is positive definite, the Gram determinant score preserves $\frac{1}{4L}$-Hamming ordering on $\mP$ and any $\vx,\hat{\vx}, \hat{\vx}'$ with $\mQ, \mQ'\in \mathcal{Q}_{L,1/(64L^2d^2)}$.
\end{lemma}

\paragraph{Proof of \cref{lem:preserve_hamming}}

\Cref{lem:preserve_hamming} establishes that the Gram determinant score approximately preserves the Hamming ordering under balancedness and small Hamming distance conditions. The main technical challenge lies in deriving upper and lower bounds on the Gram determinant in terms of the Hamming distance~\cref{lem:approx_det_hamming}.

\begin{lemma}\label{lem:approx_det_hamming}
    For all $\delta\ge 0$, and $\vx, \hat{\vx}$ with diagonally dominant $\mQ$, if $\delta = 1-\Tr(\mQ)$ and $\delta<\frac{\min_i q_\vx(i)}{4}$,
    $$ \left(1-\frac{8d\delta^2}{\min_i q_\vx(i)^2}\right)\left(1-\frac{\delta}{\min_i q_\vx(i)} \right)\le \frac{\det(\mQ)}{\prod_i q(i)} \le \left(1+\frac{8d\delta^2}{\min_i q_\vx(i)^2}\right)\left(1-\frac{\delta}{2\max_i q_\vx(i)}\right).$$
    % Moreover, $\Gamma>\Gamma'$ if $\delta'<\frac{1}{16Ld}$ and $\delta<\frac{1}{4L}\delta'$.
\end{lemma}

\begin{lemma}\label{lem:bounded_ratio}
    Given $L\ge 1$, if $a_1,\dots,a_d\ge 0$, $\sum_{i\in [d]} a_i = 1$ and $a_i\le L a_j$ for all $i,j\in [d]$, then
    $$\frac{1}{Ld-L+1}\le a_i\le \frac{L}{d+L-1}, \text{ for all } i$$
\end{lemma}

\begin{proof}[Proof of \cref{lem:preserve_hamming}]
    If $\vx, \hat{\vx}, \hat{\vx}'$ with $\mQ, \mQ'\in \mathcal{Q}_{L,1/(64L^2d^2)}$, the true labels are $L$ balanced, and Hamming distances $\delta = 1-\Tr(\mQ), \delta' = 1-\Tr(\mQ')$ are less than $\frac{1}{64L^2d^2}$.  If $\hat{\vx}\succ^\vx_{\hamming, 1/(4L)}\hat{\vx}'$, we want to show $\left(\frac{\det(\mQ)}{\det(\mQ')}\right)^2>1$.

%         Because $\mG$ is positive definite and $\mQ, \mQ'$ are diagonally dominant, $\Gamma, \Gamma'>0$ and we can compute the ratio $\Gamma/\Gamma'$, and have
% $$\frac{\Gamma}{\Gamma'} = \frac{\det(\mQ^\intercal\mG\mQ)}{\det((\mQ')^\intercal\mG\mQ')}
%         = \frac{\det(\mQ^\intercal\mQ)}{\det((\mQ')^\intercal\mQ')} = \left(\frac{\det(\mQ)}{\det(\mQ')}\right)^2$$
%         by \cref{eq:gram2gram}.
    Note that as $\mQ, \mQ'$ are diagonally dominant $\det(\mQ), \det(\mQ') >0$, and we use \cref{lem:approx_det_hamming} to show that the lower bound of $\det(\mQ)$ is larger than the upper bound of $\det(\mQ')$,
    $$\left(1+\frac{8d(\delta')^2}{\min_i q_\vx(i)^2}\right)\left(1-\frac{\delta'}{2\max_i q_\vx(i)}\right)<\left(1-\frac{8d\delta^2}{\min_i q_\vx(i)^2}\right)\left(1-\frac{\delta}{\min_i q_\vx(i)} \right).$$  By taking the difference, we have
    \begin{align*}
        &\left(1-\frac{8d\delta^2}{\min_i q_\vx(i)^2}\right)\left(1-\frac{\delta}{\min_i q_\vx(i)} \right)-\left(1+\frac{8d(\delta')^2}{\min_i q_\vx(i)^2}\right)\left(1-\frac{\delta'}{2\max_i q_\vx(i)}\right)\\
        >&\frac{\delta'}{2\max_i q_\vx(i)}-\frac{8d\delta^2}{\min_i q_\vx(i)^2}-\frac{\delta}{\min_i q_\vx(i)}-\frac{8d(\delta')^2}{\min_i q_\vx(i)^2}\tag{The second order terms are positive}\\
        \ge& \frac{\delta'}{2\max_i q_\vx(i)}-\frac{\delta}{\min_i q_\vx(i)}-\frac{16d(\delta')^2}{\min_i q_\vx(i)^2}\tag{$\delta<\delta'$}\\
        \ge& \frac{\delta'}{4\max_i q_\vx(i)}-\frac{16d(\delta')^2}{\min_i q_\vx(i)^2}\tag{$\delta'>4L\delta$}\\
        =& \frac{\delta'}{4\max_i q_\vx(i)}\left(1-\frac{64d\max_i q_\vx(i)\delta'}{\min_i q_\vx(i)^2}\right)\\
        \ge& \frac{\delta'}{4\max_i q_\vx(i)}\left(1-\frac{64Ld}{\min_i q_\vx(i)}\delta'\right)\tag{$\max_i q_\vx(i)<L\min_i q_\vx(i)$}\\
        >& 0\tag{$\delta'<\frac{1}{64L^2d^2}$ and $\min_i q_\vx(i)\ge \frac{1}{Ld}$ by \cref{lem:bounded_ratio}}
    \end{align*}
\end{proof}

\begin{proof}[Proof of \cref{lem:approx_det_hamming}]
We want to estimate $\det(\mQ)$ by the Hamming distance.  Let $\mQ = \mD+\mE$ where $\mD$ is a diagonal matrix and $\mE$ has zero diagonal, and $\delta_i = \sum_{j\neq i} \mE(i,j) = q_\vx(i)-\mD(i,i)\ge 0$ for all $i\in \mathcal{X}$ which is the off-diagonal weight of row $i$.  With above notations, $1-\Tr(\mQ) = \sum_{i\in \mathcal{X}} \delta_i = \delta$ and $\det(\mD) = \prod (q_\vx(i)-\delta_i)$.  If $\rho =\|\mD^{-1}\mE\|_2$ and
    $\delta_Q = -\rho d\ln(1-\rho)$, by \cref{thm:appro_det}
\begin{equation}\label{eq:approx_hamming2}
    1-\delta_Q \le \frac{\det(\mQ)}{\det(\mD)}\le 1+\delta_Q.
\end{equation}
As $\mD^{-1}\mE$ is a nonnegative matrix, by Gershgorin theorem, the spectral radius $\rho$ can be bounded by the row sum $\delta_i/\mQ(i,i)\le \frac{2\delta}{\min_i q_\vx(i)}$ since $\mQ$ is diagonally dominant.  Because $-\ln(1-t)\le 2t$ for all $t<1/2$ and $\delta\le \frac{\min_i q_\vx(i)}{4}$, we have
\begin{equation}\label{eq:approx_hamming21}
    \delta_Q\le 2d\rho^2\le \frac{8d\delta^2}{\min_i q_\vx(i)^2}
\end{equation}

% Therefore, $\delta_Q\le (1+o(1))(L+d-1)^2\delta^2$

Now we bound the ratio $\frac{\det(\mD)}{\prod_i q_\vx(i)} = \prod_i \left(1-\frac{\delta_i}{q_\vx(i)}\right)$.  By union bound,
\begin{align*}
    \prod_i \left(1-\frac{\delta_i}{q_\vx(i)}\right)\ge& 1-\sum_i \frac{\delta_i}{q_\vx(i)}\ge 1-\frac{\delta}{\min_i q_\vx(i)}\tag{$\delta = \sum \delta_i$}
    % \\
    % \ge& 1-(Ld-L+1)\delta\tag{by \cref{lem:bounded_ratio}}
\end{align*}  On the other hand,
\begin{align*}
    \prod_i \left(1-\frac{\delta_i}{q_\vx(i)}\right)\le& \exp\left(-\sum \frac{\delta_i}{q_\vx(i)}\right)\tag{$1-t\le e^{-t}$ for all $t$}\\
    \le& \exp\left(- \frac{\delta}{\max_i q_\vx(i)}\right)\tag{$\delta = \sum \delta_i$}\\
    \le&1-\frac{\delta}{2\max_i q_\vx(i)}\tag{$\delta<\max q_\vx(i)$ and $e^{-t}\le 1-\frac{1}{2}t$ if $0\le t\le 1$}
    % \le& 1- \frac{(L+d-1)\delta}{2L}\tag{by \cref{lem:bounded_ratio}}
\end{align*}
Therefore,
\begin{equation}\label{eq:approx_hamming3}
   1-\frac{\delta}{\min_i q_\vx(i)} \le \frac{\det(\mD)}{\prod_i q(i)}\le 1-\frac{\delta}{2\max_i q_\vx(i)}
\end{equation}
By \cref{eq:approx_hamming2,eq:approx_hamming3}, we have
\begin{align*}
    (1-\delta_Q)\left(1-\frac{\delta}{\min_i q_\vx(i)} \right)\le \frac{\det(\mQ)}{\prod_i q(i)} \le (1+\delta_Q)\left(1-\frac{\delta}{2\max_i q_\vx(i)}\right)
\end{align*}
which completes the proof by plugging in \cref{eq:approx_hamming21}
\end{proof}

\begin{proof}[Proof of \cref{lem:bounded_ratio}]
    Because $a_j\ge \frac{1}{L}a_i$ for all $i\neq j$, $1 = \sum a_j \ge a_i+\frac{d-1}{L}a_i \le \frac{L+d-1}{L}a_i$, and
    $$a_i\le \frac{L}{L+d-1}.$$
    On the other hand, because $a_j\le L a_i$, $1 = \sum a_j \le a_i+(d-1)La_i$, and
    $$a_i\ge \frac{1}{Ld-L+1}.$$
\end{proof}

\subsection{Proofs for Experiment Agnostic}\label{app:exp_agn}
\begin{proof}[Proof of \cref{prop:exp_agn}]
We first show that there is $\alpha = 1/S(\mathbb{I})>0$ so that for all $\mP, \mQ\in GL_d$
\begin{equation}\label{eq:prod0}
    S(\mP\mQ) = \alpha S(\mP)S(\mQ).
\end{equation}

Since $S$ is experiment agonistic, given any $\mP$, $S(\mP\mQ)$ is increasing in $S(\mQ)$, and there exists an increasing function $g_\mP$ so that
    $S(\mP\mQ) = g_\mP(S(\mQ))$
    Because for any $s,t>0$ and $\mQ$, $S(st\mQ) = c(st)S(\mQ) = c(s)c(t)S(\mQ)$ and $S(\mQ)>0$, we have $c(st) = c(s)c(t)$ for all $s,t>0$.  Therefore,
    \begin{equation}\label{eq:prod1}
        c(t) =  t^\gamma\text{ for some }\gamma\in \R.
    \end{equation}

    For any $t>0$ and $\mP,\mQ$, we have $S(\mP t\mQ) = c(t)S(\mP\mQ) = c(t)g_\mP(S(\mQ))$, and $S(\mP t\mQ) = g_{\mP}(S(t\mQ)) = g_{\mP}(c(t)S(\mQ))$.  Hence
    $$g_{\mP}(c(t)S(\mQ)) = c(t)g_\mP(S(\mQ)).$$
    For any $\mP$ and $\mQ$, we have
$$S(\mP\mQ) = g_\mP\left(S(\mQ)\cdot\frac{1}{S(\mQ)}S(\mQ)\right) = S(\mQ)g_\mP(1)$$ by \cref{eq:prod1} and taking $t = S(\mQ)^{-\gamma}$.  By taking $\mQ = \mathbb{I}$ we have $g_\mP(1) = \frac{S(\mP\mQ)}{S(\mQ)} = \frac{S(\mP)}{S(\mathbb{I})}$, and prove \cref{eq:prod0}.

    By \cref{eq:prod0}, $\tilde{S}(\mQ):=\alpha S(\mQ)$ is a continuous homomorphism between $GL_d$ and $(\R_{>0}, \cdot)$ so that for all $\mP, \mQ$
    $\tilde{S}(\mP\mQ) = \tilde{S}(\mP)\tilde{S}(\mQ).$
    Thus, there exists a continuous $f:\R\setminus\{0\}\to \R_{>0}$ so that $\tilde{S}(\mQ) = f(\det(\mQ))$.~\citep{727050}  We now pin down the function $f$.  First, $S(t\mQ) = \alpha f(t^d\det(\mQ))$ and by \cref{eq:prod1}, $S(t\mQ) = \alpha c(t)f(\det(\mQ)) = \alpha t^\gamma f(\det(\mQ))$ for all $t>0$ and $\mQ$. Given $\beta = \gamma/(2d)$, for all $t>0$, $f(t) = t^{2\beta} f(1)$ and $f(-t) = t^{2\beta} f(-1)$.  Moreover, because $f$ is a homomorphism $f(-1)^2 = f((-1)^2) = f(1) = 1$ and $f(-1)>0$, we have for all $z\neq 0$, $f(z) = |z|^{2\beta}$ and
    $$S(\mQ) = \alpha f(\det(\mQ)) = \alpha |\det(\mQ)|^{2\beta} = \alpha \det(\mQ^\intercal \mQ)^\beta.$$
\end{proof}

% \begin{example}
%     A counter example for MI showing not experiment agonistic.
% \end{example}

\section{Lemmas and Proofs for \cref{sec:estimator}}\label{app:estimator}

% \paragraph{Small biased Gram determinant score estimator}

\paragraph{Proofs for \cref{thm:plug-in}}
\begin{lemma}\label{lem:concent_emp}
Given $\delta>0$ and report size $N$,
    $$\Pr\left[\|\bar{\mG}-\hat{\mG}\|_2\le 4\sqrt{\frac{\log 2d/\delta}{N}}+2\frac{\log 2d/\delta}{N}\right]\ge 1-\delta.$$
\end{lemma}
\begin{proof}[Proof of \cref{thm:plug-in}]
By~\cref{thm:appro_det2}, we have
$$\frac{|\det(\bar{\mG})-\det(\hat{\mG})|}{\det(\hat{\mG})}\le \left(1+\|\hat{\mG}^{-1}\|_2{\|\bar{\mG}-\hat{\mG}\|_2}\right)^d-1.$$
Hence with \cref{lem:concent_emp} and $\delta = 1/N$, we have $\frac{|\det(\bar{\mG})-\det(\hat{\mG})|}{\det(\hat{\mG})}= o(1)$, with probability greater than $1-1/N$.  Additionally, because the random variable $\det(\bar{\mG})$ is always bounded by $1$, the expectation \begin{equation}\label{eq:plugin1}
    \E[\det(\bar{\mG})] = (1+o(1))\det(\hat{\mG})
\end{equation}
when $\hat{\mG}^{-1}$ exists.  If $\hat{\mG}$ is not invertible, by~\cite[Theorem 2.12]{doi:10.1137/070704770},
\begin{equation}\label{eq:plugin2}
    |\det(\bar{\bm{G}})-\det(\hat{\bm{G}})| \le d\|\bar{\bm{G}}-\hat{\bm{G}}\|_2 \max\{\|\hat{\bm{G}}\|_2, \|\bar{\bm{G}}\|_2\}^{d-1} = o(1).
    \end{equation}
For all $\mP$ and $\mQ, \mQ'$, if
$\det(\hat{\mG}) = \det(\mQ^\intercal \mG\mQ)> \det((\mQ')^\intercal \mG\mQ') = \det(\hat{\mG}')>0$, by \cref{eq:plugin1} there exists a large enough $N_0$ so that any $\vx, \hat{\vx}, \hat{\vx}'$ with length at least $N_0$ and $\mQ, \mQ'$ so that $\E[\det(\bar{\mG})]>\E[\det(\bar{\mG}')]$.  Similar argument holds for $\det(\hat{\mG}) >\det(\hat{\mG}') = 0$ by~\cref{eq:plugin2}.  Therefore, the plug-in estimator asymptotically preserves all reliability orderings as \cref{thm:gram_preserve}.
\end{proof}

\begin{proof}[Proof of \cref{lem:concent_emp}]
Let $N_i = Nq_{\hat{\vx}}(i)$ be the number of report $i$ which is nonzero as $\mQ\in \mathcal{Q}_{\text{reg}}$.  Let $|\mathcal{Y}| = k$, and we can set $\phi: \mathcal{Y}\to \R^{k}$ be the delta vector $y\mapsto \mathbf{1}_y$.  We define $\bar{\vv}_i = \sum_{n: \hat{x}_n = i} \phi(y_n)$ and $\vv_i = \E[\bar{\vv}_i]$ as the sum of (empirical) mean of report $i\in \mathcal{X}$, and error $\ve_i = \bar{\vv}_i-\vv_i$.  Hence
    for all $i,j$,
    $\bar{\mG}(i,j) = \frac{1}{N^2}\sum_{n, n': \hat{x}_n = i, \hat{x}_{n'} = j}\langle \phi(y_n), \phi(y_{n'})\rangle = \frac{1}{N^2}\bar{\vv}_i^\intercal\bar{\vv}_j$, $\hat{\mG}(i,j) = \frac{1}{N^2}{\vv}_i^\intercal{\vv}_j$, and
    \begin{equation}\label{eq:concent_emp1}
        \bar{\mG}(i,j)-\hat{\mG}(i,j) = \frac{1}{N^2}\left(\vv_i^\intercal\ve_j+\ve_i^\intercal \vv_j + \ve_i^\intercal\ve_j \right)
    \end{equation}
    To bound the spectral norm of $\bar{\mG}-\hat{\mG}\in \R^{d\times d}$, for any $a\in \R^d$ with $\|a\|_2 = 1$, we define $\vv(a) = \sum_i a_i \vv_i$, $\ve(a) = \sum_i a_i\ve_i\in \R^{k}$, and $R_\vv = \sup_{\|a\| = 1} \|\vv(a)\|, R_\ve = \sup_{\|a\| = 1} \|\ve(a)\|$.  By \cref{eq:concent_emp1}
\begin{equation}\label{eq:concent_emp2}
    a^\intercal (\bar{\mG}-\hat{\mG})a = \frac{1}{N^2}(2 \vv(a)^\intercal \ve(a)+ \ve(a)^\intercal \ve(a))
    \le \frac{1}{N^2}(2R_\vv R_\ve+R_\ve^2).
\end{equation}
We first bound $R_\vv$.  For all $a$ with $\|a\| = 1$, let $\mV = N^2\hat{\mG}\in \R^{d\times d}$ where $\mV(i,j) = \vv_i^\intercal\vv_j$ which is positive semi definite
\begin{align*}
    \|\vv(a)\|^2 =& \sum_{i,j} a_i a_j \vv_i^\intercal\vv_j\\
    =& a^\intercal \mV a\\
    \le& \sum_i \vv_i^\intercal \vv_i\tag{Rayleigh quotient is upper bounded by the trace}\\
    =& \sum_i N^2\hat{\mG}(i,i)\tag{definition of $\vv_i$}\\
    \le& \sum_i N^2q_{\hat{\vx}}(i)^2\tag{because $\langle P_{x}, P_{x'}\rangle\le 1$ for any $x,x'$}\\
    \le& N^2\max_i q_{\hat{\vx}}(i)
\end{align*}
Therefore,
\begin{equation}\label{eq:concent_emp3}
    R_\vv\le N\sqrt{\max_i q_{\hat{\vx}}(i)}\le N
\end{equation}
We bound $R_\ve$ using Chernoff bound.   For each $i\in \mathcal{X}$, $\ve_i = \bar{\vv}_i-\vv_i = \sum_{n: \hat{x}_n = i}\phi(y_n)-\E \phi(y_n)$ is sum of $Nq_{\vx}(i)$ independent vectors in $\R^k$, and the norm of each vector is bounded by $1$.  Therefore, by~\cite[Theorem 3.5]{pinelis1994optimum}, for all $r_i\ge 0$
$$\Pr[\|\ve_i\|\ge r_i]\le 2\exp\left(-\frac{r_i^2}{2Nq_{\vx}(i)}\right)$$
Given any $\delta>0$ and $a$ with $\|a\| = 1$, we set $r_i = \sqrt{2Nq_{\vx}(i)\ln(2d/\delta)}$, and we have
$$\|\ve(a)\|^2\le \sum_i \|\ve_i\|^2\le \sum_i 2Nq_{\vx}(i)\ln(\frac{2d}{\delta}) = 2N\ln \frac{2d}{\delta}$$
Therefore,
\begin{equation}\label{eq:concent_emp4}
    R_\ve\le \sqrt{2N\ln \frac{2d}{\delta}}
\end{equation}
with probability at least $1-\delta$.  Plugging in \cref{eq:concent_emp3,eq:concent_emp4} to \cref{eq:concent_emp2}, we have
$$\|\bar{\mG}-\hat{\mG}\|_2\le \frac{1}{N^2}\left(2N\sqrt{2N\ln \frac{2d}{\delta}}+2N\ln \frac{2d}{\delta}\right) \le \frac{4\sqrt{\ln 2d/\delta}}{\sqrt{N}}+\frac{2\ln 2d/\delta}{N}.$$
\end{proof}

With~\cref{lem:concent_emp}, we further provide sample complexity bounds for both multiplicative and additive errors of the plug-in Gram determinant score which are polynomial in all relevant parameters.

\begin{proposition}
There exists $\epsilon_0>0$, so that the following holds
\begin{itemize}
    \item For any $d\in \mathbb{N}, \delta, \lambda>0$, and $0<\epsilon\le \epsilon_0$, if $\hat{\bm{G}}\succ \lambda \mathbb{I}$, the plug-in Gram determinant score satisfies
$$\Pr[|\bar{S}(\hat{\bm{x}}, \bm{y})-\Gamma|\le \epsilon \Gamma]\ge 1-\delta$$
for all $N\ge \frac{100 d^2 \ln \frac{d}{\delta}}{\lambda^2 \epsilon^2}$ where $\Gamma = det(\hat{\bm{G}})$.
\item For any $d\in \mathbb{N}, \delta>0$, $0<\epsilon\le \epsilon_0$ and $\hat{\bm{G}}$, the plug-in Gram determinant score satisfies
$$\Pr[|\bar{S}(\hat{\bm{x}}, \bm{y})-\Gamma|\le \epsilon]\ge 1-\delta$$
for all $N\ge \frac{100 d^2 \ln \frac{d}{\delta}}{\epsilon^2}$.
\end{itemize}
\end{proposition}

For the multiplicative error, we show that when $\hat{\bm{G}}$ is invertible, the required sample size is $O(\frac{d^2 \ln \frac{d}{\delta}}{\lambda^2 \epsilon^2})$ where $\lambda$ can be the smallest eigenvalue of $\hat{\bm{G}}$ as $\hat{\bm{G}}$ is positive semidefinite.
For the additive error, we can get a stronger bound $O(\frac{d^2 \ln \frac{d}{\delta}}{\epsilon^2})$ independent of the minimum eigenvalue.

This result implies \cref{thm:plug-in}, which shows that the reliability order is asymptotically preserved for any fixed pair of reports.
However, it does not yield a uniform, non-asymptotic guarantee on order preservation, as two report Gram matrices can be arbitrarily close, $\hat{\bm{G}} \approx \hat{\bm{G}}'$.
\begin{proof}
First, \Cref{lem:concent_emp} implies that
$$\|\bar{\bm{G}}-\hat{\bm{G}}\|_2\le 4\sqrt{\frac{\log 2d/\delta}{N}}+2\frac{\log 2d/\delta}{N}$$
with probability larger than $1-\delta$.  By taking $N_{\delta,\epsilon} = 100\cdot \frac{d^2 \ln \frac{d}{\delta}}{\lambda^2 \epsilon^2}$ and $N\ge N_{\delta,\epsilon}$, we can bound the error as
$$4\sqrt{\frac{\log 2d/\delta}{N}}+2\frac{\log 2d/\delta}{N} \le 4\sqrt{\frac{\lambda^2 \epsilon^2 \ln 2}{100 d^2}}+2 \frac{\lambda^2 \epsilon^2 \ln 2}{100 d^2}\le 5\sqrt{\frac{\lambda^2 \epsilon^2}{100 d^2}} = \frac{\lambda \epsilon}{2d}.$$
The first inequality is due to $N\ge N_{\delta,\epsilon}$.  The second holds when $\epsilon_0$ is small enough and $\epsilon\le \epsilon_0$ as $4\sqrt{\ln 2}<5$.  Note that the choice of $\epsilon_0$ can be independent of other parameters, because $\lambda\le 1$ and $d\ge 1$.

Second, by Theorem A.2, we have
\begin{align*}
    \frac{|\det(\bar{\bm{G}})-\det(\hat{\bm{G}})|}{\det(\hat{\bm{G}})}\le& \left(1+{\|\hat{\bm{G}}^{-1}\|_2}{\|\bar{\bm{G}}-\hat{\bm{G}}\|_2}\right)^d-1.\tag{by Theorem A.2}\\
    \le& \left(1+\frac{1}{\lambda}\cdot\frac{\lambda \epsilon}{2d}\right)^d-1.\tag{$\|\hat{\bm{G}}^{-1}\|_2\le 1/\lambda$}\\
    \le& e^{\epsilon/2}-1.\tag{$(1+x/d)^d\le e^x$}\\
    \le& \epsilon.\tag{by taking $\epsilon_0\le \ln 2$}
\end{align*}
This proves the first part.

For the second, part, we use \citet[Theorem 2.12]{doi:10.1137/070704770},
\begin{align*}
    |\det(\bar{\bm{G}})-\det(\hat{\bm{G}})|\le& d\|\bar{\bm{G}}-\hat{\bm{G}}\|_2 \max\{\|\hat{\bm{G}}\|_2, \|\bar{\bm{G}}\|_2\}^{d-1}\\
    \le& d\|\bar{\bm{G}}-\hat{\bm{G}}\|_2\tag{by Perron–Frobenius theorem}\\
    \le& d\cdot \frac{\epsilon}{2d}\le \epsilon\tag{by taking $N_{\delta,\epsilon}' = 100\cdot \frac{d^2 \ln \frac{d}{\delta}}{\epsilon^2}$}
\end{align*}
\end{proof}

\paragraph{Proof of \cref{thm:stratified}}
The core idea relies on the multi-linearity of the determinant, and we can approximately get samples of $\hat{\mG} = \mQ^\intercal\mG\mQ$ in the detail-free setting.  However, one caveat is that we may not have access to multiple independent samples from $\hat{\mG}$ as $\vx, \hat{\vx}$ are deterministic.  To circumvent this issue, we first observe that if $\hat{\vx} = \vx$, the observations are independently and identically distributed for each label, allowing an unbiased estimator for $\hat{\mG}$ and thus $\det(\hat{\mG})$.  If $\hat{\vx} \neq \vx$, our sampling scheme ensures that the expectation is bounded above by the Gram determinant score. This guarantees that exact match orderings are preserved, as truthful reports yield higher scores in expectation than any nontruthful reports.
\begin{proof}[Proof of \cref{thm:stratified}]
    By the definition of exact ordering, it is sufficient to show for any $\vx, \hat{\vx}$ with $\vx\succ_{\exact}^\vx \hat{\vx}$ and $\mP\in \mathcal{P}_{\text{indep}}$, $$\E_{\vy\sim P(\vx)}[score(\vx, \vy)]>\E_{\vy\sim P(\vx)}[score(\hat{\vx}, \vy)].$$
When the minimum occurrence is at least two, the expectation of \cref{eq:stratified} involves three sources of randomness: observation $\vy$, permutations $\sigma$, and the choice of $Col$ and $Row$.  The expectation of $score(\vx, \vy)$ only depends on the first two as difference indexing does not change the distribution of score.  However, for $score(\hat{\vx}, \vy)$, the third part will kick in.

Given the index sets $Col, Row$, we define $\mQ^{Col}, \mQ^{Row}\in \R^{d\times d}$ so that
\begin{equation}\label{eq:stratified1}
    \mQ^{Col}(i,j) = {q_{\hat{\vx}}(j)}\sum_{n\in Col} \mathbf{1}[x_n = i, \hat{x}_n = j] = q_{\hat{\vx}}(j) \mathbf{1}[x_{j,Col} = i]
\end{equation}
and $\mQ^{Row}(i,j)$ similarly
which are the misreport matrix when restricting reports in $Col$ and $Row$ respectively. As $Col$ can be seen as stratified sampling where each report has exactly one element in $Col$, $\sum_{n\in Col} \mathbf{1}[x_n = i, \hat{x}_n = j] = \mQ_{\vx|\hat{\vx}}(i,j)$, and the expectation over the choice of index is
\begin{equation}\label{eq:stratified2}
    \E[\mQ^{Col}] = \E[\mQ^{Row}] = \mQ.
\end{equation}
With the above notation, we first compute the expectation of \cref{eq:stratified} conditional on $Col$ and $Row$.
\begin{align*}
    &\E[score(\hat{\vx}, \vy)\mid Col, Row]\\
    =& \E\left[d! sgn(\sigma)\prod_{k,l\in [d], l = \sigma(k)} \mathbf{1}[y_{k,Row} = y_{l,Col}]q_{\hat{\vx}}(k)q_{\hat{\vx}}(l)\mid Col, Row\right]\\
    =& \E\left[\sum_{\sigma\in sym(d)} sgn(\sigma)\prod_{k,l\in [d], l = \sigma(k)} \mathbf{1}[y_{k,Row} = y_{l,Col}]q_{\hat{\vx}}(k)q_{\hat{\vx}}(l)\mid Col, Row\right]\tag{random $\sigma$}\\
    =& \E\left[\sum_{\sigma\in sym(d)} sgn(\sigma)\prod_{k,l\in [d], l = \sigma(k)} \langle P_{x_{k,Row}}, P_{x_{l,Col}}\rangle q_{\hat{\vx}}(k)q_{\hat{\vx}}(l)\mid Col, Row\right]\tag{$Col\cap Row = \emptyset$}\\
    =& \E\left[\sum_{\sigma\in sym(d)} sgn(\sigma)\prod_{k,l\in [d], l = \sigma(k)} \sum_{i,j} \mQ^{Row}(i,k)\mQ^{Col}(j,l)\langle P_{i}, P_{j}\rangle \mid Col, Row\right]\tag{by \cref{eq:stratified1}}\\
    =&  \E\left[\sum_{\sigma\in sym(d)} sgn(\sigma)\prod_{k,l\in [d], l = \sigma(k)} \left((\mQ^{Row})^\intercal \mG \mQ^{Col}\right)(k,l) \mid Col, Row\right]\\
    =&  \E\left[\det\left((\mQ^{Row})^\intercal \mG \mQ^{Col}\right) \mid Col, Row\right]
\end{align*}
Therefore,
\begin{equation}\label{eq:stratified3}
    \E[score(\hat{\vx}, \vy)] = \E\left[\det\left((\mQ^{Row})^\intercal \mG \mQ^{Col}\right)\right] = \E\left[\det\left((\mQ^{Row})^\intercal \mQ^{Col}\right)\right]\det( \mG)
\end{equation}
First, when $\hat{\vx} = \vx$, because $\mQ\in \mathcal{Q}_{L}$, and $N\ge 2Ld$, every label has at least $N\min_{i}q_{\vx}(i)\ge 2Ld\frac{1}{Ld-L+1}\ge 2$ reports by \cref{lem:bounded_ratio}, and the minimum occurrence is at least two.  Moreover, $\mQ^{Col} = \mQ^{Row} = \mQ$ are identity matrices regardless the choice of $Col$ and $Row$, by \cref{eq:stratified3}, we have
\begin{equation}\label{eq:stratified4}
    \E[score(\vx, \vy)] = \det(\mG).
\end{equation}

On the other hand, for $\hat{\vx}$ with $\vx\succ^\vx_{\exact}\hat{\vx}$, if the minimum occurrence is less than two, the score would be zero and less than \cref{eq:stratified_kernel}.  Otherwise, by Cauchy–Schwarz inequality, we have
\begin{equation}\label{eq:stratified5}
\E\left[\det\left((\mQ^{Row})^\intercal \mQ^{Col}\right)\right]\le \E\left[\det\left((\mQ^{Row})\right)\right]\E\left[\det\left(\mQ^{Col}\right)\right]
\end{equation}
Formally, consider $\mathcal{I}$ the collection of all possible index set of size $d$ where each label occurs exactly once.  Then we can generate $Col$ and $Row$ by sampling two distinct $(i,j)$ element of $\mathcal{I}$ uniformly at random.  In particular, if we set $a_i$ be the determinant of the misreporting matrix of the $i$-th index set in $\mathcal{I}$, the joint distribution of $(\det(\mQ^{Col}), \det(\mQ^{Row}))$ equals $(a_i, a_j)$ and
\begin{align*}
&\E\left[\det\left((\mQ^{Row})\right)\right]\E\left[\det\left(\mQ^{Col}\right)\right]-\E\left[\det\left((\mQ^{Row})^\intercal \mQ^{Col}\right)\right]\\
=& \left(\frac{1}{|\mathcal{I}|}\sum_{i} a_i\right)\left(\frac{1}{|\mathcal{I}|}\sum_{j} a_j\right)-\frac{1}{|\mathcal{I}|(|\mathcal{I}|-1)}\sum_{i\neq j\in \mathcal{I}} a_i a_j\\
=& \frac{1}{|\mathcal{I}|^2}\sum_i a_i^2-\frac{1}{|\mathcal{I}|^2(|\mathcal{I}|-1)}\sum_{i\neq j} a_ia_j\\
=& \frac{1}{2|\mathcal{I}|^2(|\mathcal{I}|-1)}\sum_{i\neq j}(a_i-a_j)^2\ge 0
\end{align*}
which proves \cref{eq:stratified5}.
Finally, using the first part of \cref{thm:gram_preserve} and \cref{eq:stratified2,eq:stratified3,eq:stratified4,eq:stratified5}
$$\E[score(\hat{\vx}, \vy)]\le \det(\mQ^\intercal \mG\mQ) = \det(\hat{\mG})< \det(\mG) = \E[score(\vx, \vy)].$$
\end{proof}

\section{Proofs for \cref{sec:examples}}\label{app:examples}

% \paragraph{Proof of \cref{thm:gram_preserve_kernel}}
\begin{proof}[Proof of \cref{thm:gram_preserve_kernel}]
To use \cref{lem:preserve_exact,lem:preserve_hamming,lem:preserve_blackwell}, it is sufficient to show that $\mG_K$ is positive definite for all $\mP\in \mathcal{P}_{\text{indep}}$ so that for any nonzero sequence $a:\mathcal{X}\to \R$, the quadratic form is positive,
\begin{equation}\label{eq:quadratic_form}
    \sum_{x,x'\in \mathcal{X}} a(x)G(x,x')a(x')>0.
\end{equation}

First for any integrally strictly positive definite kernel, recall that the kernel mean embedding of $P_x$ is $\phi(P_x) = \E_{y\sim P_x}[\phi(y)]\in \mathcal{H}$, so
$\sum_{x,x'} a(x)G_K(x,x')a(x') = \|\sum_x a(x)\phi(P_x)\|^2\ge 0$ which shows $\mG_K$ is positive semi definite.
If the equality happens, by linearity of integration, $0 = \|\sum_x a(x)\phi(P_x)\|^2 = \iint_\mathcal{Y} K(y,y') d\mu(y)d\mu(y')$ where $\mu = \sum_x a(x) P_x$ is a finite signed measure.  Therefore, $\mu = \sum_x a(x) P_x = 0$ because $K$ is integrally strictly positive definite.  Finally $a(x) = 0$ as columns of $\mP$ are linearly independent.  Therefore the statement holds for integrally strictly positive definite kernels.  Additionally, by \cite[Theorem 3.1]{ziegel2022characteristickernelshilbertspaces}, the Gaussian kernel is integrally strictly positive definite.

Second, given a feature map $\phi:\mathcal{Y}\to \R^k$, \cref{eq:quadratic_form} becomes $\|\sum_{x,y} a(x)P(y,x)\phi(y)\|_2^2$.  Because $\mP\in \mathcal{P}_{\text{indep}}$ and $\phi$ is injective, the quadratic form equals zero if and only if $a(x) = 0$ for all $x$.  Moreover, delta kernel is injective, so the statement also holds.

Finally, for any pseudo-posterior observations, \cref{eq:quadratic_form} can be written as
\begin{align*}
    &\langle \sum_{x,y} a(x) \mP(y,x)y, \sum_{x',y'} a(x') \mP(y',x')y'\rangle\\
    =& \sum_{x,x',y,y'} a(x)a(x')\mP(y,x)\mP(y',x')\langle \tilde{P}[\rx|y],\tilde{P}[\rx|y']\rangle\\
    =& \sum_{x,x',y,y'} a(x)a(x')\mP(y,x)\mP(y',x')\sum_{x''}\tilde{P}[\rx = x''|y]\tilde{P}[\rx = x''|y']
\end{align*}
Let $b(y) = \sum_x a(x)\mP(y,x)$, $w(y) = \sum_{x} \mP(y,x)\tilde{q}(x)$. Then $\sum_{x}\tilde{P}[\rx = x|y]\tilde{P}[\rx = x|y']= \sum_{x} \mP(y,x)\\\frac{\tilde{q}(x)}{w_(y)}\mP(y',x)\frac{\tilde{q}(x)}{w_(y')}$\footnote{Here we set $0/0 = 0$ if $w(y)= 0$} and
\begin{align*}
    &\sum_{x,x',y,y'} a(x)a(x')\mP(y,x)\mP(y',x')\sum_{x''}\tilde{P}[\rx = x''|y]\tilde{P}[\rx = x''|y']\\
    =& \sum_{y,y'} b(y)b(y')\sum_{x} \mP(y,x)\frac{\tilde{q}(x)}{w_(y)}\mP(y',x)\frac{\tilde{q}(x)}{w_(y')}\\
    =& \sum_{x} \tilde{q}(x)^2\left(\sum_y \mP(y,x)\frac{b(y)}{w_(y)}\right)^2
\end{align*}
Because $\tilde{q}$ has full support, the quadratic form equals zeros if and only if $\sum_y \mP(y,x)\frac{b(y)}{w_(y)} = 0$ for all $x$.  Equivalently, if we set vector $\vb = \mP a\in \R^{|\mathcal{Y}|}$ and $\mD_w$ be the diagonal matrix with $w$, we have $\vzero = \vb^\intercal \mD_w \mP = a^\intercal \mP^\intercal\mD_w\mP$. Since $\mP$ has full column rank and $w(y) = 0$ when $\mP(x,y) = 0$ for all $x$, $a(x) = 0$ for all $x$.
\end{proof}

\section{Alternatives to Gram Determinant Score}
\label{app:alternative_score}
There is a long line of research on measuring the stochastic relationship between random variables.  We may view them as data reliability scores applied to the reported data $\hat{\vx}$ and observations $\vy$.  While we do not attempt to prove a formal reliability ordering among these candidates, we introduce them here to support the comparative analysis.

\paragraph{$\Phi$-mutual information}
\begin{definition}[$\Phi$-divergence~\citep{csiszar1964informationstheoretische, Morimoto1963, ali1966general}]\label{def:fdiv} Let $\Phi:[0,\infty)\to \R$ be a convex function with $\Phi(1) = 0$.  Let $P$ and $Q$ be two probability distributions on a common measurable space $(\Omega, \mathcal{F})$.  The \emph{$\Phi$-divergence of $Q$ from $P$} where $P\ll Q$\footnote{$P$ is absolutely continuous with respect to $Q$: for any measurable set $A\in \mathcal{F}$, $Q(A) = 0\Rightarrow P(A) = 0$.} is defined as
$D_\Phi(P\| Q):= \E_Q\left[\Phi\left(P/Q\right)\right].$\footnote{$P/Q$ is the Radon-Nikodym derivative between measures $P$ and $Q$, and it is equal to the ratio of density function.}
\end{definition}
We can use these divergences to measure how interdependent two random variables $\rx$ and $\ry$ are.  Formally, Let $P_{\rx,\ry}$ be a distribution over $(\rx,\ry)\in \mathcal{X}\times\mathcal{Y}$, and $P_\rx$ and $P_\ry$ be marginal distributions of $\rx$ and $\ry$ respectively.  We set $P_\rx P_\ry$ be the tensor product between $P_\rx$ and $P_\ry$ such that $P_\rx P_\ry(x,y) = P_\rx(x)P_\ry(y)$.  We call $D_\Phi(P_{\rx,\ry}\|P_\rx P_\ry)$ the \emph{$\Phi$-mutual information between $\rx$ and $\ry$}.
\begin{enumerate}
    \item Total variation has $\Phi(a)$ as $\frac{1}{2}|a-1|$.
    \item KL-divergence has $a\log a$
    \item $\chi^2$-divergence has $a^2-1$
    \item Squared Hellinger distance has $\left(1-\sqrt{a}\right)^2$
\end{enumerate}

In the partial knowledge setting, we can access the $\mJ:=\mP\mQ$ which can be seen as a joint distribution between reported data and observation $\mJ = P_{\rx,\ry}$, and set $$S_\Phi(\mP\mQ) = D_\Phi(P_{\rx,\ry}\|P_\rx P_\ry).$$

This family of score satisfy the data processing inequality, which is analogous to our \emph{weak} Blackwell dominant ordering so that garbling the report can only decrease the score.  Nevertheless, the impossibility results in \cref{sec:impossible} still apply. In addition, they are generally not experiment‑agnostic, and lack kernelized extensions as in \cref{sec:examples} for complicated observation space $\mathcal{Y}$.

% \begin{proposition}[weak Blackwell ordering]
%     For any column stochastic matrix $\mT$ and convex $\Phi$,
%     $$S_\Phi(\mJ\mT^\intercal)\le S_\Phi(\mJ).$$
% \end{proposition}
% \begin{proof}
%     The statement is exactly data processing inequality.
% \end{proof}

\paragraph{Family of symmetric gauge on singular values}
Our Gram determinant is essentially a functional on the singular values of $\mJ = \mP\mQ$ and sub multiplicative under right multiplication by contraction.  One may additionally consider functional on the singular values of the whitened matrix.  Formally, given a joint distribution $\mJ:= \mP\mQ$, let
$$\bar{\mJ} = \mD_{\vy}^{-1/2}(\mJ-\mu_{\vy}\mu_{\hat{\vx}}^\intercal)\mD_{\hat{\vx}}^{-1/2}$$
where $\mu_{\hat{\vx}}$ and $\mu_\vy$ are marginal distributions and $\mD_{\hat{\vx}}, \mD_\vy$ are diagonal matrix of them respectively.
Given a matrix $\mA$, $\sigma(\mA)$ denote the singular value list of $\mA$, we can find a symmetric gauge $\psi$ and define our score as
$$S_\psi(\mJ) = \psi(\sigma(\bar{\mJ})).$$
Let $\bar{\vs} = \sigma(\bar{\mJ}) = (\bar{s}_1,\dots,\bar{s}_d)$ with $\bar{s}_1\ge \bar{s}_2\ge\dots \bar{s}_d\ge 0$.
\begin{enumerate}
    \item Top-$k$ volume has $\psi_{\wedge k}(\vs) = \prod_{i = 1}^k \bar{s}_i$
    \item  Maximal correlation $\psi_{\max} = \bar{s}_1$.  The maximum correlation can be also written as $$\max_{(f, g)\in \mathcal{F}} \E [f(\rx)g(\ry)]$$ where $\mathcal{F}$ is the collection of real-valued random variables so that $\E f(\rx) = \E g(\ry) = 0$ and $\E f(\rx)^2 = \E g(\ry)^2 = 1$.
    \item Ky-Fan $k$-sum $\sum_{i = 1}^k \bar{s}_i$
    \item $\chi^2$-mutual information $I_{\chi^2}(\rx,\ry) = \sum_{x,y}\mu_{\hat{\vx}}(x)\mu_\vy(y)(\frac{\mJ(y,x)}{\mu_{\hat{\vx}}(x)\mu_\vy(y)}-1)^2 = \|\bar{\mJ}\|_F = \sum_i \bar{s}_i^2$
\end{enumerate}
Similarly, the impossibility results in \cref{sec:impossible} still apply and they are generally not experiment‑agnostic.

\section{Additional Experiments and Details}

\subsection{Experiment Details}

Due to space limitations, we omit some settings in the main paper. Here, we provide the details of how we compute error bars and how we obtain the ranking‐accuracy across sample sizes in \cref{fig:delta2}.

\paragraph{Error bars.}
Let $M$ be the number of independent trials.  For each trial $m\in[M]$, let $score^{(m)}$ denote the determinant score, and similarly let $\hamming^{(m)}$ and $\ell_2^{(m)}$ denote the Hamming distance and $\ell_2$‐norm error, respectively.  We compute the sample mean
$$
\overline{score} = \frac{1}{M}\sum_{m=1}^M score^{(m)}
$$
and the standard error of the mean
$$
SE(\overline{score})
=
\frac{1}{\sqrt{M(M-1)}}
\sqrt{\sum_{m=1}^M\bigl(score^{(m)}-\overline{score}\bigr)^2}.
$$
Under approximate normality, we report a 95\% confidence interval as
$
\overline{score}\pm1.96 SE(\overline{score})
$ in \cref{fig:det_vs_p2}, \cref{fig:hamming_vs_det2} and \cref{fig:l2_vs_det2}.
The same procedure is applied to $\hamming^{(m)}$ and $\ell_2^{(m)}$ to yield their error bars in \cref{fig:hamming_vs_det2,fig:l2_vs_det2}.

\paragraph{Ranking accuracy across sample sizes.}
In \cref{fig:delta2}, we plot the fraction of trials in which the reversed ranking induced by the determinant score agrees with the ranking induced by each baseline metric—namely the reporting probability $p$, the Hamming distance, and the $\ell_2$‐norm—over six noise levels
$
\mathcal{P} = \{0.0,\,0.1,\,0.2,\,0.3,\,0.4,\,0.5\}.
$
Concretely, in each trial $m$ we form three vectors
\[
\bigl(score^{(m)}_p\bigr)_{p\in\mathcal{P}},\quad
\bigl(\hamming^{(m)}_p\bigr)_{p\in\mathcal{P}},\quad
\bigl({\ell_2}^{(m)}_p\bigr)_{p\in\mathcal{P}}.
\]
We then check whether the total order of $\bigl(score^{(m)}_p\bigr)$ in decreasing order matches the order of $\bigl(\hamming^{(m)}_p\bigr)$ in increasing order (and similarly for $\ell_2$ and for $p$ itself).  If they coincide, trial $m$ is counted as a “correct” ranking.  The plotted accuracy is
\[
\frac{1}{M}\sum_{m=1}^M \mathbf{1}\bigl\{\text{orders agree in trial }m\bigr\}.
\]
A random guess among the $6!$ possible orderings yields a baseline accuracy of $1/6!\approx 0.00139$.

\subsection{Additional Experiments}

\subsubsection{Comparison on Delta and Gaussian Kernel}

\begin{table}[htbp]
    \centering
    \begin{tabular}{c c c c}
    \toprule
        Figure & Experiments & Manipulation & Kernel \\
        \hline
        \Cref{fig:all_results2} &  Random exp & uniform~\cref{eq:uniform_misreport} & delta\\
        \Cref{fig:all_results_exp2nd} &  Random exp & normal~\cref{eq:normal_misreport} & delta\\
        \Cref{fig:all_results_exp2ug} &  Random exp & uniform & Gaussian\\
        \Cref{fig:all_results_exp2ng} &  Random exp & normal & Gaussian\\
        \Cref{fig:all_results_exp1gu} &  Gaussian & uniform & Gaussian\\
        \Cref{fig:all_results_exp1gn} &  Gaussian & normal & Gaussian\\
        \Cref{fig:all_results_exp1du} &  Gaussian & uniform & delta\\
        \Cref{fig:all_results_exp1dn} &  Gaussian & normal & delta\\
        \bottomrule
    \end{tabular}
    \caption{Table for settings of the experiments on synthetic data.}
\end{table}
In this section, we still use the dataset and experiment settings from Experiment 1. However, besides the uniform random manipulation introduced in \cref{eq:uniform_misreport}, we consider \emph{normal manipulation}:
\begin{equation}\label{eq:normal_misreport}
    \hat{\vx}_k \;\sim\;
    \operatorname{clip}\bigl(1,\,d,\,
        \operatorname{round}\bigl(\mathcal{N}(\vx,\sigma_0)\bigr)
    \bigr).
\end{equation}
This manipulation introduces localized perturbations to the data.
We adopt $\sigma_0 \in \{0.30, 0.37, 0.44, \dots, 1.00\}$ in our experiments and refer to this type of manipulation as normal manipulation. For the matched rankings results, the rankings are computed for 6 data points with $\sigma_0 \in \{0.30, 0.44, 0.58, 0.72, 0.86, 1.00\}$.

Moreover, we also use the \emph{Gaussian kernel}
$$K(y,y') = \exp\left(-100\left\|y - y'\right\|_2^2\right)$$
besides the delta kernel $K(y,y') = \mathbf{1}[y = y']$.

\begin{figure}[ht]
  \centering
  \begin{subfigure}[b]{0.24\textwidth}
    \includegraphics[width=\textwidth]{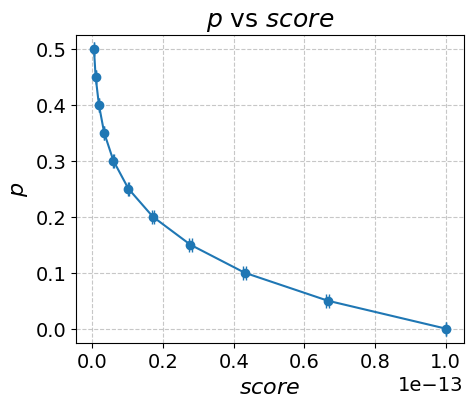}
    \caption{$score$ vs.\ $\sigma_0$}
    \label{fig:exp2-p}
  \end{subfigure}
  \hfill
  \begin{subfigure}[b]{0.24\textwidth}
    \includegraphics[width=\textwidth]{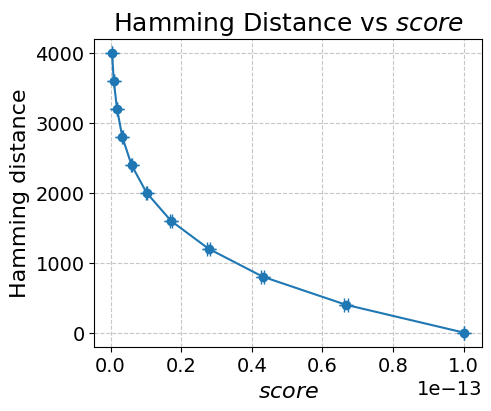}
    \caption{Hamming vs.\ $score$}
    \label{fig:exp2-hamming}
  \end{subfigure}
  \hfill
  \begin{subfigure}[b]{0.24\textwidth}
    \includegraphics[width=\textwidth]{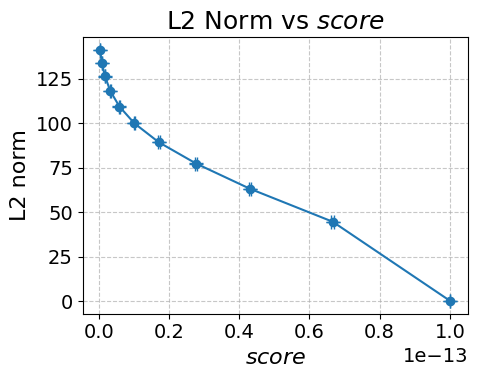}
    \caption{$\ell_2$ norm vs.\ $score$}
    \label{fig:exp2-norm}
  \end{subfigure}
  \hfill
  \begin{subfigure}[b]{0.24\textwidth}
    \includegraphics[width=\textwidth]{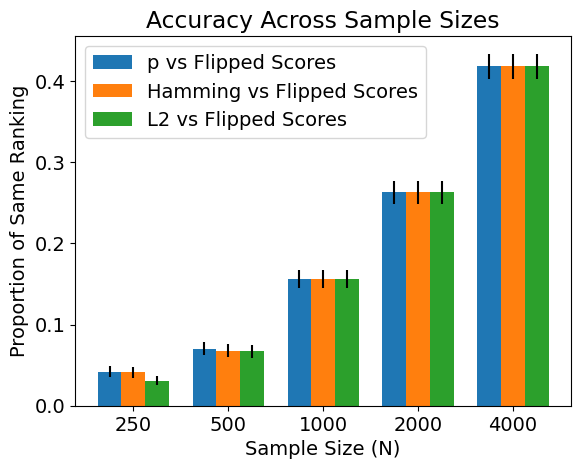}
    \caption{matched rankings}
    \label{fig:exp2}
  \end{subfigure}
  \caption{Experiments of plug-in Gram determinant reliability
score with delta kernel on categorical synthetic data with uniform manipulation in \cref{eq:uniform_misreport}.}
  % \label{fig:all_results2}
\end{figure}

\begin{figure}[ht]
  \centering
  \begin{subfigure}[b]{0.24\textwidth}
    \includegraphics[width=\textwidth]{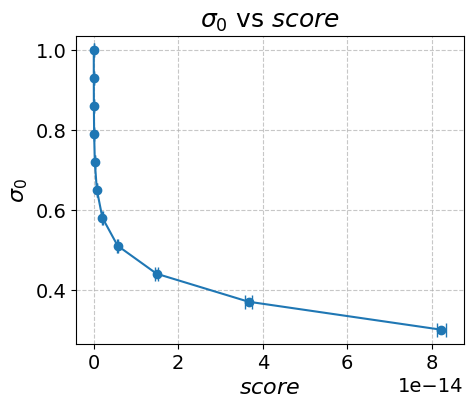}
    \caption{$score$ vs.\ $\sigma_0$}
    \label{fig:exp2nd-p}
  \end{subfigure}
  \hfill
  \begin{subfigure}[b]{0.24\textwidth}
    \includegraphics[width=\textwidth]{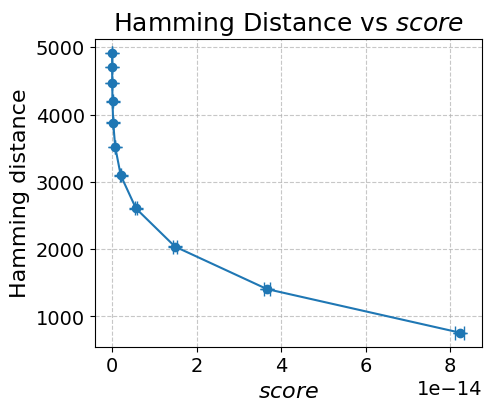}
    \caption{Hamming vs.\ $score$}
    \label{fig:exp2nd-hamming}
  \end{subfigure}
  \hfill
  \begin{subfigure}[b]{0.24\textwidth}
    \includegraphics[width=\textwidth]{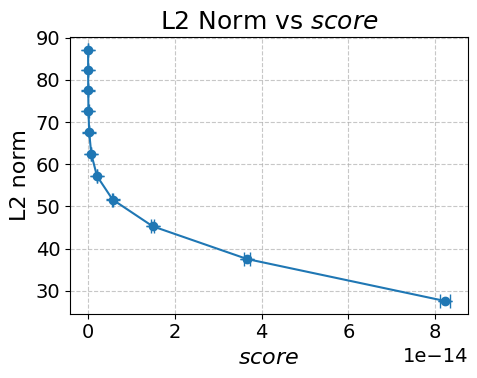}
    \caption{$\ell_2$ norm vs.\ $score$}
    \label{fig:exp2nd-norm}
  \end{subfigure}
  \hfill
  \begin{subfigure}[b]{0.24\textwidth}
    \includegraphics[width=\textwidth]{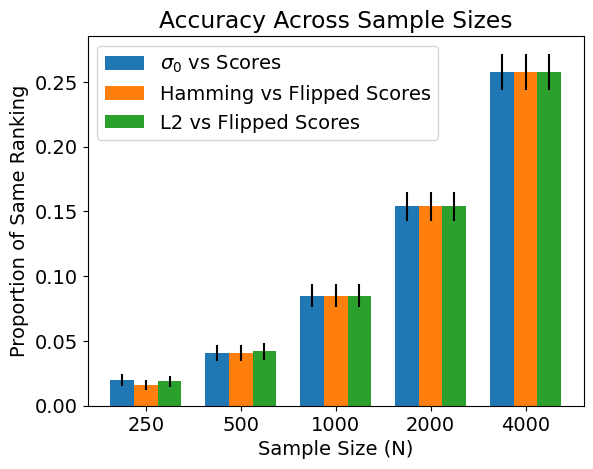}
    \caption{matched rankings}
    \label{fig:exp2nd}
  \end{subfigure}
  \caption{Experiments of plug-in Gram determinant reliability
score with delta kernel on categorical synthetic data with normal manipulation in \cref{eq:normal_misreport}.}
  \label{fig:all_results_exp2nd}
\end{figure}

\begin{figure}[ht]
  \centering
  \begin{subfigure}[b]{0.24\textwidth}
    \includegraphics[width=\textwidth]{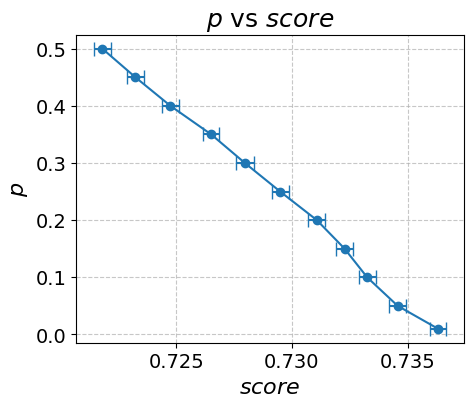}
    \caption{$score$ vs.\ $p$}
    \label{fig:exp2ug-p}
  \end{subfigure}
  \hfill
  \begin{subfigure}[b]{0.24\textwidth}
    \includegraphics[width=\textwidth]{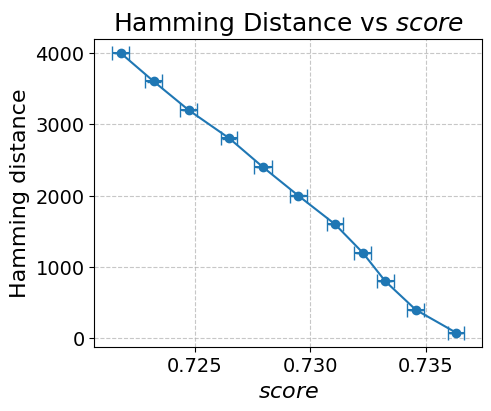}
    \caption{Hamming vs.\ $score$}
    \label{fig:exp2ug-hamming}
  \end{subfigure}
  \hfill
  \begin{subfigure}[b]{0.24\textwidth}
    \includegraphics[width=\textwidth]{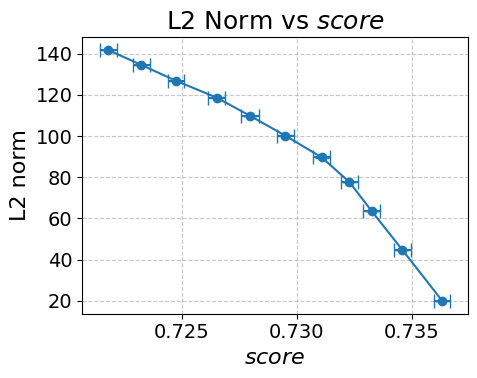}
    \caption{$\ell_2$ norm vs.\ $score$}
    \label{fig:exp2ug-norm}
  \end{subfigure}
  \hfill
  \begin{subfigure}[b]{0.24\textwidth}
    \includegraphics[width=\textwidth]{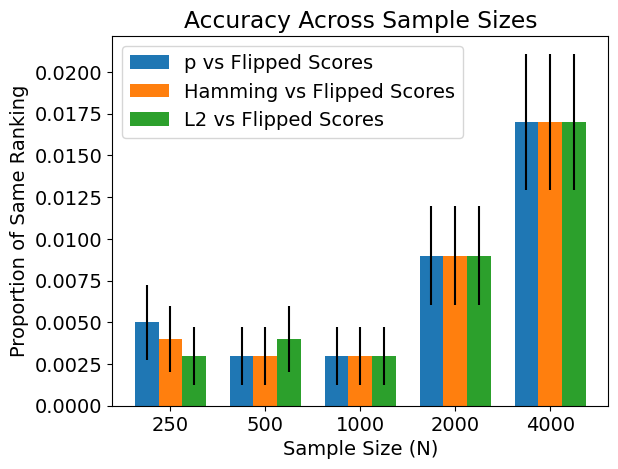}
    \caption{matched rankings}
    \label{fig:exp2ug}
  \end{subfigure}
  \caption{Experiments of plug-in Gram determinant reliability
score with Gaussian kernel on categorical synthetic data with uniformly random manipulation in \cref{eq:uniform_misreport}.}
  \label{fig:all_results_exp2ug}
\end{figure}

\begin{figure}[ht]
  \centering
  \begin{subfigure}[b]{0.24\textwidth}
    \includegraphics[width=\textwidth]{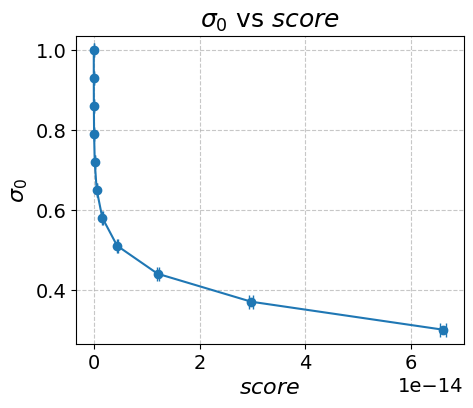}
    \caption{$score$ vs.\ $\sigma_0$}
    \label{fig:exp2ng-p}
  \end{subfigure}
  \hfill
  \begin{subfigure}[b]{0.24\textwidth}
    \includegraphics[width=\textwidth]{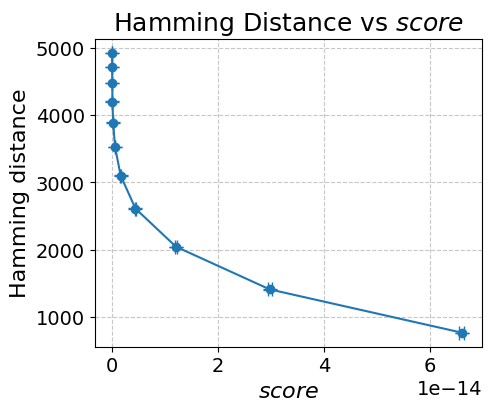}
    \caption{Hamming vs.\ $score$}
    \label{fig:exp2ng-hamming}
  \end{subfigure}
  \hfill
  \begin{subfigure}[b]{0.24\textwidth}
    \includegraphics[width=\textwidth]{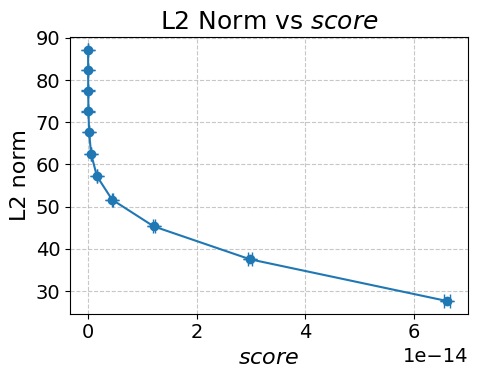}
    \caption{$\ell_2$ norm vs.\ $score$}
    \label{fig:exp2ng-norm}
  \end{subfigure}
  \hfill
  \begin{subfigure}[b]{0.24\textwidth}
    \includegraphics[width=\textwidth]{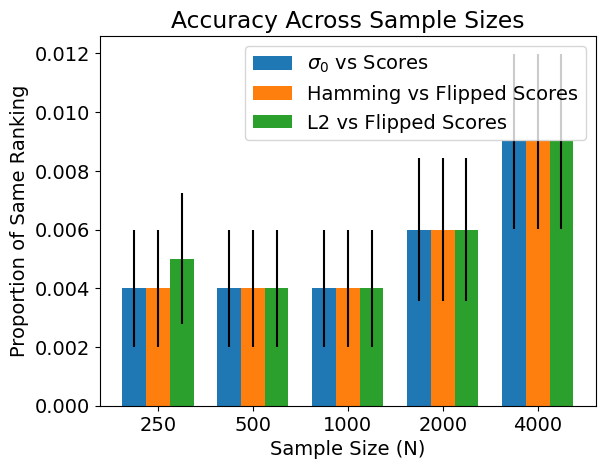}
    \caption{matched rankings}
    \label{fig:exp2ng}
  \end{subfigure}
  \caption{Experiments of plug-in Gram determinant reliability
score with Gaussian kernel on categorical synthetic data with normal manipulation in \cref{eq:normal_misreport}.}
  \label{fig:all_results_exp2ng}
\end{figure}

From \cref{fig:all_results_exp2nd,fig:all_results_exp2ug,fig:all_results_exp2ng}, we find that the Gram determinant score behaves consistently as a measure of reliability across the kernels considered. It is negatively correlated with all reliability metrics. As the sample size increases, its alignment with the other reliability metrics also improves. For this categorical dataset, the delta kernel outperforms the Gaussian kernel, achieving both lower standard error and higher accuracy for matched rankings across sample sizes.

We next create a synthetic dataset with continuous observations. Instead of a random categorical experiment, each $y_i$ is sampled from a normal distribution $\mathcal{N}(x_i, \sigma)$ centered at $x_i$. We set $\sigma = 0.1$ and $d=4$; all other experimental settings remain the same as in Experiment~1.

Because $\mathcal{Y}$ is continuous rather than categorical, we cannot directly apply the plug-in Gram determinant score with the delta kernel. Instead, we create a sequence $\bar{\mathbf{y}}$ by bucketing $\mathbf{y}$ into $d$ bins using empirical quantiles, and then apply the plug-in Gram determinant score to $\hat{\vx}$ and $\bar{\mathbf{y}}$.

\begin{figure}[htbp]
  \centering
  \begin{subfigure}[b]{0.24\textwidth}
    \includegraphics[width=\textwidth]{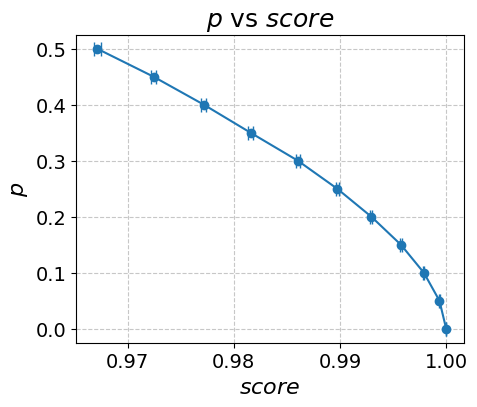}
    \caption{$score$ vs.\ $p$}
    \label{fig:exp1gu-p}
  \end{subfigure}
  \hfill
  \begin{subfigure}[b]{0.24\textwidth}
    \includegraphics[width=\textwidth]{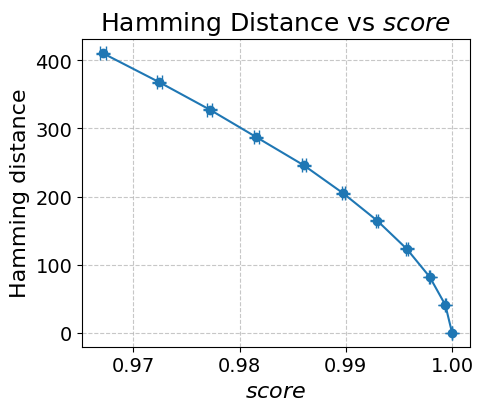}
    \caption{Hamming vs.\ $score$}
    \label{fig:exp1gu-hamming}
  \end{subfigure}
  \hfill
  \begin{subfigure}[b]{0.24\textwidth}
    \includegraphics[width=\textwidth]{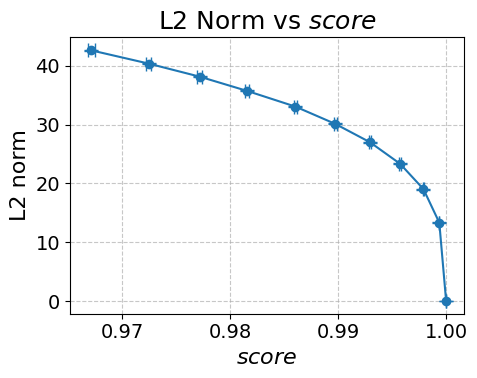}
    \caption{$\ell_2$ norm vs.\ $score$}
    \label{fig:exp1gu-norm}
  \end{subfigure}
  \hfill
  \begin{subfigure}[b]{0.24\textwidth}
    \includegraphics[width=\textwidth]{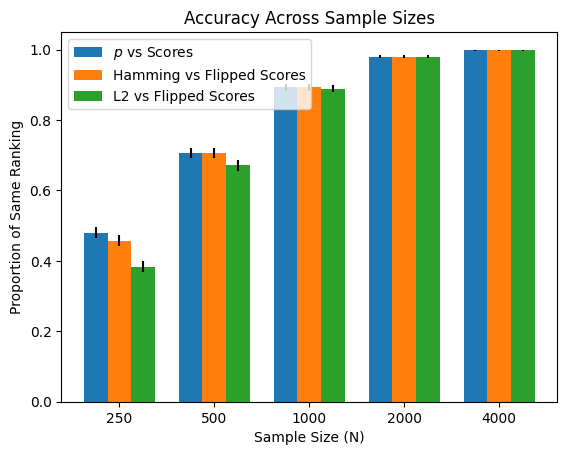}
    \caption{matched rankings}
    \label{fig:exp1gu}
  \end{subfigure}
  \caption{Experiments of plug-in Gram determinant reliability
score with Gaussian kernel on Gaussian synthetic data with uniformly random manipulation in \cref{eq:uniform_misreport}.}
  \label{fig:all_results_exp1gu}
\end{figure}

\begin{figure}[htbp]
  \centering
  \begin{subfigure}[b]{0.24\textwidth}
    \includegraphics[width=\textwidth]{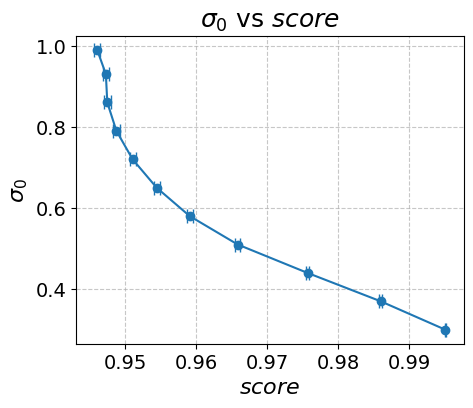}
    \caption{$score$ vs.\ $\sigma_0$}
    \label{fig:exp1gn-p}
  \end{subfigure}
  \hfill
  \begin{subfigure}[b]{0.24\textwidth}
    \includegraphics[width=\textwidth]{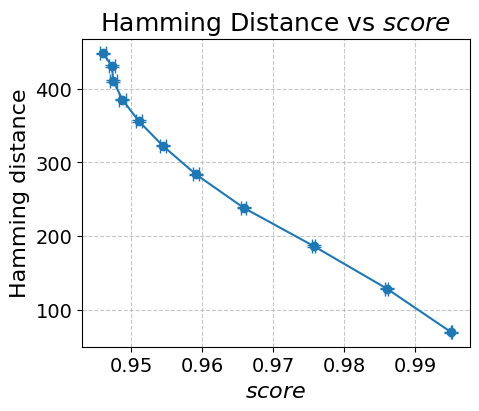}
    \caption{Hamming vs.\ $score$}
    \label{fig:exp1gn-hamming}
  \end{subfigure}
  \hfill
  \begin{subfigure}[b]{0.24\textwidth}
    \includegraphics[width=\textwidth]{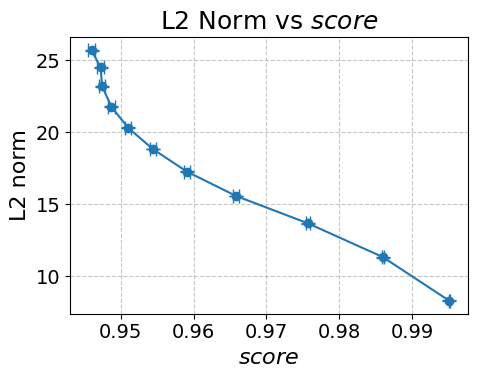}
    \caption{$\ell_2$ norm vs.\ $score$}
    \label{fig:exp1gn-norm}
  \end{subfigure}
  \hfill
  \begin{subfigure}[b]{0.24\textwidth}
    \includegraphics[width=\textwidth]{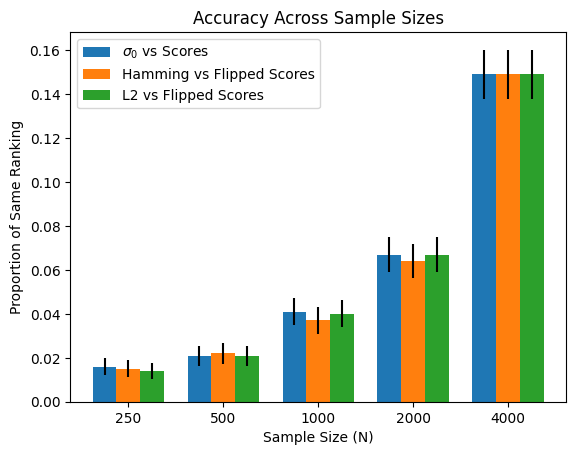}
    \caption{matched rankings}
    \label{fig:exp1gn}
  \end{subfigure}
  \caption{Experiments of plug-in Gram determinant reliability
score with Gaussian kernel on Gaussian synthetic data with normal manipulation in \cref{eq:normal_misreport}.}
  \label{fig:all_results_exp1gn}
\end{figure}

\begin{figure}[htbp]
  \centering
  \begin{subfigure}[b]{0.24\textwidth}
    \includegraphics[width=\textwidth]{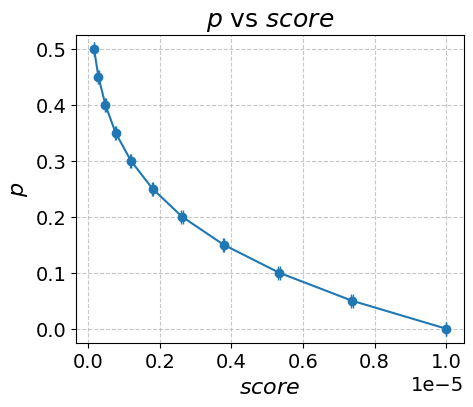}
    \caption{$score$ vs.\ $p$}
    \label{fig:exp1du-p}
  \end{subfigure}
  \hfill
  \begin{subfigure}[b]{0.24\textwidth}
    \includegraphics[width=\textwidth]{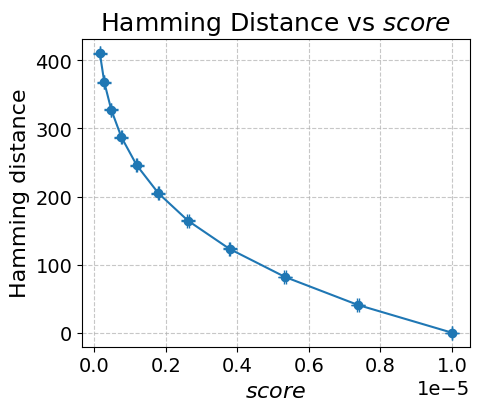}
    \caption{Hamming vs.\ $score$}
    \label{fig:exp1du-hamming}
  \end{subfigure}
  \hfill
  \begin{subfigure}[b]{0.24\textwidth}
    \includegraphics[width=\textwidth]{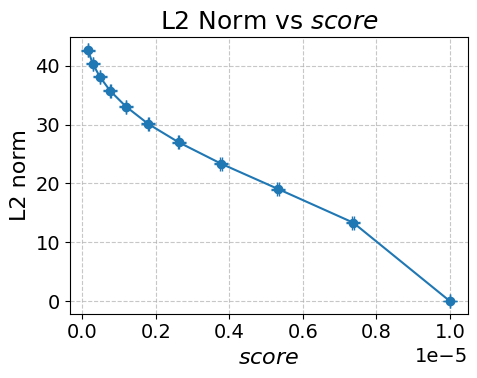}
    \caption{$\ell_2$ norm vs.\ $score$}
    \label{fig:exp1du-norm}
  \end{subfigure}
  \hfill
  \begin{subfigure}[b]{0.24\textwidth}
    \includegraphics[width=\textwidth]{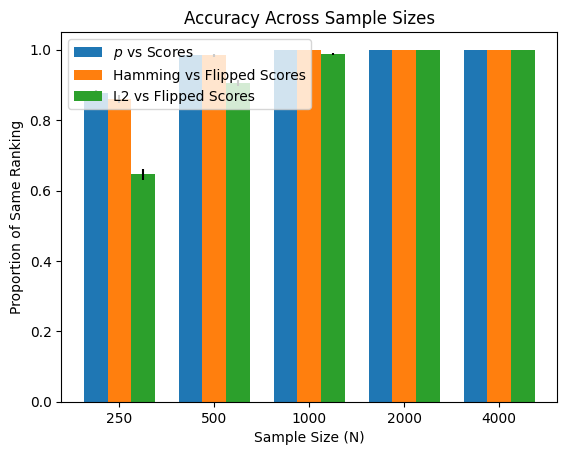}
    \caption{matched rankings}
    \label{fig:exp1du}
  \end{subfigure}
  \caption{Experiments of plug-in Gram determinant reliability
score with delta kernel on Gaussian synthetic data with uniformly random manipulation in \cref{eq:uniform_misreport}.}
  \label{fig:all_results_exp1du}
\end{figure}

\begin{figure}[htbp]
  \centering
  \begin{subfigure}[b]{0.24\textwidth}
    \includegraphics[width=\textwidth]{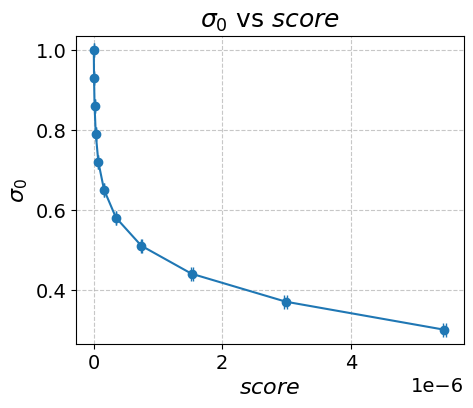}
    \caption{$score$ vs.\ $\sigma_0$}
    \label{fig:exp1dn-p}
  \end{subfigure}
  \hfill
  \begin{subfigure}[b]{0.24\textwidth}
    \includegraphics[width=\textwidth]{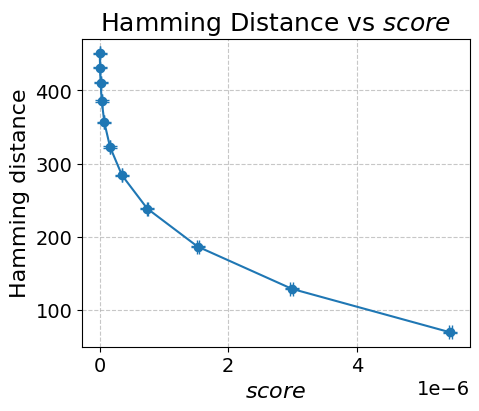}
    \caption{Hamming vs.\ $score$}
    \label{fig:exp1dn-hamming}
  \end{subfigure}
  \hfill
  \begin{subfigure}[b]{0.24\textwidth}
    \includegraphics[width=\textwidth]{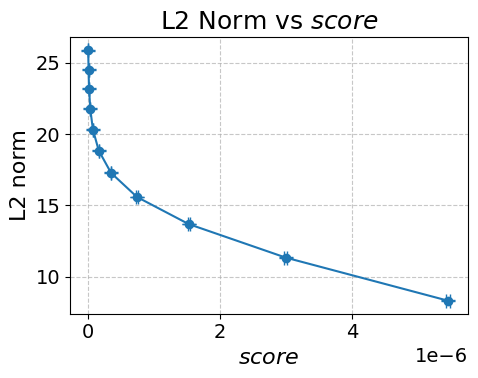}
    \caption{$\ell_2$ norm vs.\ $score$}
    \label{fig:exp1dn-norm}
  \end{subfigure}
  \hfill
  \begin{subfigure}[b]{0.24\textwidth}
    \includegraphics[width=\textwidth]{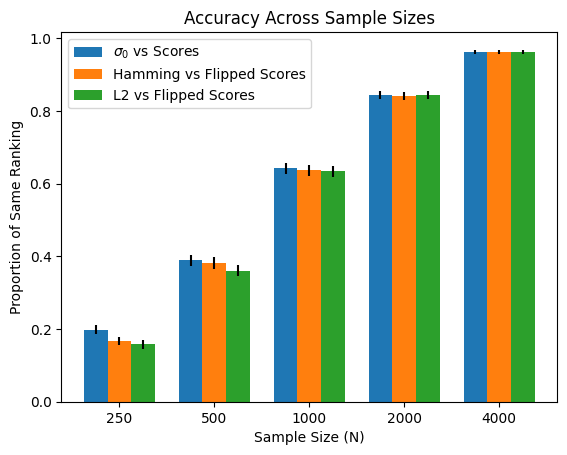}
    \caption{matched rankings}
    \label{fig:exp1dn}
  \end{subfigure}
  \caption{Experiments of plug-in Gram determinant reliability
score with delta kernel on Gaussian synthetic data with normal manipulation in \cref{eq:normal_misreport}.}
  \label{fig:all_results_exp1dn}
\end{figure}

From \cref{fig:all_results_exp1dn,fig:all_results_exp1du,fig:all_results_exp1gn,fig:all_results_exp1gu}, we observe that both the delta kernel and the Gaussian kernel perform well as reliability measures across all three metrics, under both normal and uniformly random manipulations. In particular, the delta kernel variant using bucketed $\vy$ achieves consistently strong performance. Empirically, the plug-in Gram determinant score with the delta kernel generally outperforms the version with the Gaussian kernel in most situations, despite the lack of theoretical guarantees for this delta kernel variant. For reported data with small $\ell_2$ norm error, the Gaussian kernel outperforms the approximate delta kernel score.

\subsubsection{Robustness Under Conditional Linear Dependence}
\begin{figure}[t]
    \centering
    \begin{subfigure}{0.48\linewidth}
        \centering
        \includegraphics[width=\linewidth]{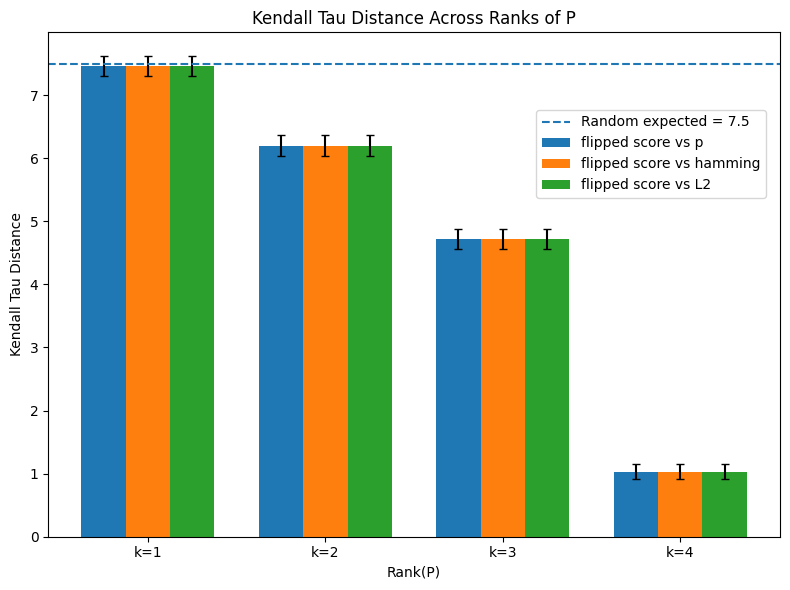}
        \caption{Kendall Tau distance between reliability order and Gram determinant score.}
        \label{fig:rank_gram}
    \end{subfigure}
    \hfill
    \begin{subfigure}{0.48\linewidth}
        \centering
        \includegraphics[width=\linewidth]{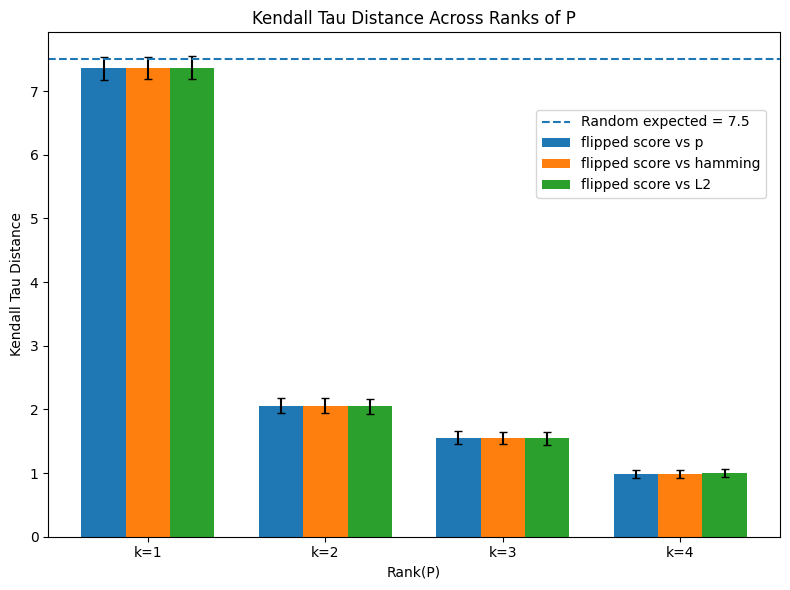}
        \caption{Kendall Tau distance between reliability order and Top-$k$ score ($k=2$).}
        \label{fig:rank_topk}
    \end{subfigure}
    \caption{Comparison of robustness under rank-deficient conditional structures.}
\end{figure}

In this experiment, we investigate how violations of conditional linear
independence affect the robustness of our proposed reliability scores.
To introduce controlled linear dependence, we construct conditional
distributions by generating a stochastic matrix $\mP$ of a prescribed
rank $\text{k}$: we sample $\text{k}$ independent basis rows on the simplex
and express the remaining rows as random convex combinations of these bases,
ensuring that $\mathrm{rank}(\mP)=\text{k}$.

We run 300 independent trials.
In each trial, we sample a ground-truth dataset $(\vx,\vy)$ of size $N=2000$
with $d=4$ using the same procedure as in Experiment~1, except that the
conditional distribution $\mP(\cdot\mid x)$ is now determined by the
rank-$\text{k}$ matrix $\mP$.
To model varying corruption, for
$p\in\{0.50,0.60,\dots,1.00\}$ we use uniformly random manipulation to generate perturbed labels
\[
\hat{x}_n =
\begin{cases}
x_n, & \text{with probability } p,\\[2pt]
Z_n, & \text{with probability } 1-p,
\end{cases}
\qquad Z_n\sim\mathrm{Uniform}\{1,\dots,d\}.
\]

For each corruption level, we compute the plug-in Gram determinant score and the
top-$\text{k}$ singular-value score and rank the six corrupted datasets accordingly,
allowing us to assess whether each scoring function responds monotonically to
increasing corruption even when $\mP$ is rank-deficient.

From \Cref{fig:rank_gram}, we observe that even when the conditional
distribution of $\vy$ given $\vx$ exhibits linear dependence, the Gram determinant
score retains nontrivial discriminative power: datasets with higher corruption
levels consistently yield lower scores, resulting in a meaningful correlation
with the true reliability order.
This indicates that the determinant-based score does not rely on full-rank
structure in order to capture relative reliability.

When the linear dependence is strong (e.g., the experiment matrices $\mP$
we construct are explicitly rank-deficient), a top-$k$ singular-value score
becomes a natural alternative.
As shown in \Cref{fig:rank_topk}, choosing $\text{k}$ equal to the true underlying
rank yields performance comparable to the Gram determinant score in the
full-rank setting, while being substantially more stable under rank
deficiency.
This confirms that top-$k$ volume scores can better adapt to structured,
low-rank conditional models, especially when only a subset of singular
directions carries meaningful information.

\subsubsection{Imbalanced Data}

\begin{figure}[htbp]
    \centering
    \begin{subfigure}[b]{0.33\textwidth}
        \includegraphics[width=\textwidth]{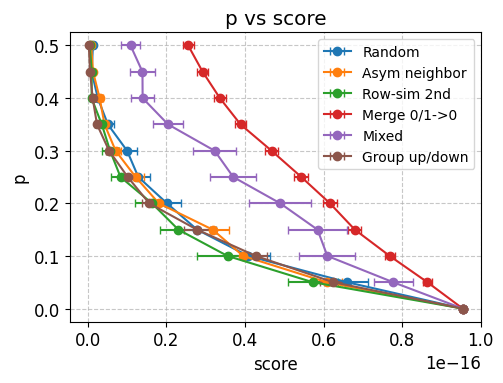}
        \caption{Gram determinant score vs.\ corruption probability $p$.}
        \label{fig:kyfan-p}
    \end{subfigure}\hfill
    \begin{subfigure}[b]{0.33\textwidth}
        \includegraphics[width=\textwidth]{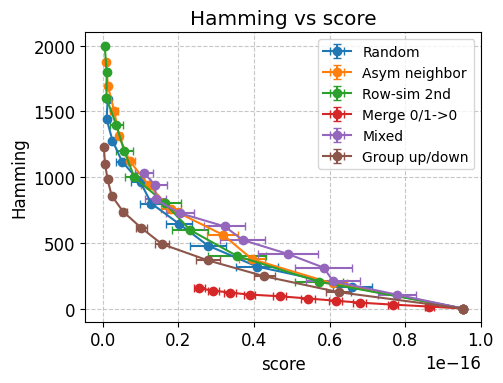}
        \caption{Hamming error vs.\ Gram determinant score.}
        \label{fig:kyfan-hamming}
    \end{subfigure}\hfill
    \begin{subfigure}[b]{0.33\textwidth}
        \includegraphics[width=\textwidth]{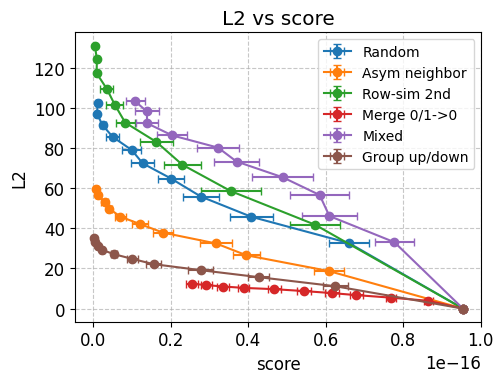}
        \caption{$\ell_2$ deviation vs.\ Gram determinant score.}
        \label{fig:kyfan-l2}
    \end{subfigure}
    \caption{Performance of the Gram determinant score under data imbalance.}
    \label{fig:imbalance}
\end{figure}

In this experiment, we study how data imbalance in the marginal distribution of
$\vx$ affects the behavior of the Gram determinant score.
This complements Experiment~1 by examining robustness not only to corruption
mechanisms but also to skewed label frequencies, a common characteristic of
real-world datasets.

The setting follows Experiment~1 exactly, except that the marginal
distribution of $\vx$ is now imbalanced: we draw $x_n=1$ with probability
$0.7$, and sample uniformly from $\{2,3,4,5\}$ with probability $0.3$.
The conditional distribution $\vy\mid \vx$ and all corruption schemes remain
unchanged.

As shown in \Cref{fig:imbalance}, the Gram determinant score continues to
separate datasets of different reliability levels even under substantial class
imbalance.
The monotonic decrease of the score with increasing corruption probability $p$
is preserved, and higher scores still correspond to lower Hamming error and
smaller $\ell_2$ deviation.
However, compared to the balanced case in Experiment~1, the score exhibits
higher variance due to the reduced effective sample size in underrepresented
classes.
This indicates that while the Gram determinant score remains informative under
imbalance, larger dataset sizes are beneficial for stabilizing the estimate.

\subsection{Experiments of Experiment Agnostic on Alternative Scores}

\begin{figure}[t]
    \centering
    \includegraphics[width=0.7\linewidth]{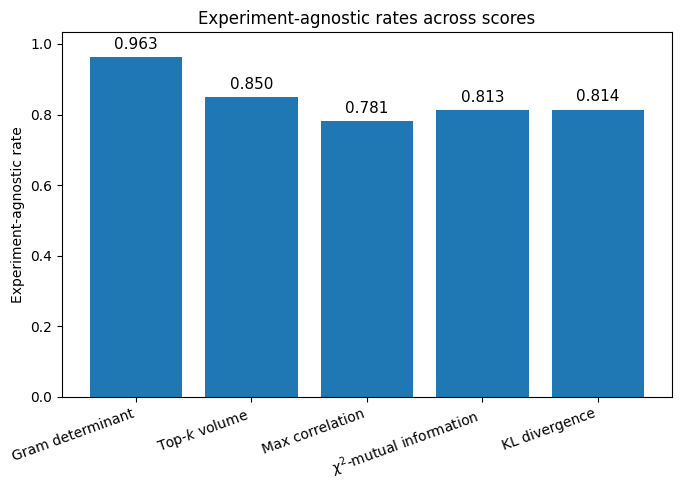}
    \caption{Experiment-agnostic rates across different scoring functions.}
    \label{fig:experiment_agnostic}
\end{figure}

Then we evaluate the robustness of each scoring function to
changes in the underlying representation.
Our goal is to assess whether a score preserves the relative quality ordering
between two corruption channels even after the data undergoes an additional,
task-specific transformation.

For each trial, we first sample a marginal distribution $\pi_X$ over
$\{1,\dots,d\}$ with $d=4$, together with a ground-truth conditional model
$\mP_{Y|X}\in\mathbb{R}^{d\times d}$ whose rows are normalized to be
stochastic.
We then draw two independent corruption channels
$\mQ_1,\mQ_2\in\mathbb{R}^{d\times d}$, each representing a
distinct noisy mapping from $x$ to corrupted reports.
All matrices are row-stochastic, so they define valid conditional
distributions.

We generate $1\times 10^8$ i.i.d.\ samples by drawing
$x\sim\pi_X$, sampling
$y\sim\mP_{Y|X}(\cdot\mid x)$, and obtaining two corrupted versions
$\hat{x}\sim\mQ_1(\cdot\mid x)$ and
$\hat{x}'\sim\mQ_2(\cdot\mid x)$.
Collecting these over all draws yields the datasets
$\vx, \vy, \hat{\vx}, \hat{\vx}'$.
For each scoring function $S(\cdot)$ (Gram determinant, top-$\text{k}$
volume, max correlation, $\chi^2$-mutual information, and KL divergence),
we compute
\[
S(\hat{\vx},\vy),\quad S(\hat{\vx}',\vy),\qquad
S(\hat{\vx},\vx),\quad S(\hat{\vx}',\vx),
\]
corresponding to evaluating channel quality before and after the transformation
$x\!\to\! y$.

A scoring function is said to preserve the ordering between the two channels
if
\[
\bigl(S(\hat{\vx},\vx)-S(\hat{\vx}',\vx)\bigr)\,
\bigl(S(\hat{\vx},\vy)-S(\hat{\vx}',\vy)\bigr)
> 0,
\]
that is, if the sign of the score difference is invariant under the
data-processing transformation.
A violation corresponds to a \emph{ranking flip}.
The experiment-agnostic rate is defined as
$1 - \text{flip rate}$, representing how often the ordering is preserved.

We repeat this procedure for $1000$ independent trials and report the average
experiment-agnostic rate for each scoring function in
\Cref{fig:experiment_agnostic}.
Higher rates indicate greater robustness to changes in representation, i.e.,
a stronger ability to maintain consistent channel rankings across different
tasks and observation models.

As shown in \Cref{fig:experiment_agnostic}, the Gram determinant score achieves a
substantially higher experiment-agnostic rate than the alternative scoring
functions.
This empirical advantage is consistent with our theoretical guarantee in
\Cref{prop:exp_agn}, which shows that the Gram determinant is uniquely robust to
changes in the underlying representation and preserves channel orderings under
a broad class of transformations.
These results confirm that the Gram determinant score not only performs well in
specific synthetic settings, but also provides the most stable and robust
measure of reliability across heterogeneous tasks and observation models.

\end{document}